\documentclass{article}

\PassOptionsToPackage{numbers,sort&compress}{natbib}

\usepackage[preprint]{neurips_2026}

\usepackage[utf8]{inputenc} 
\usepackage[T1]{fontenc}    
\usepackage{hyperref}       
\usepackage{url}            
\usepackage{booktabs}       
\usepackage{amsfonts}       
\usepackage{nicefrac}       
\usepackage{microtype}      
\usepackage{xcolor}         

\usepackage{amsmath,amsfonts,bm}

\def\eqref#1{equation~\ref{#1}}

\def\1{\bm{1}}

\def\mG{{\bm{G}}}

\def\mI{{\bm{I}}}

\def\mM{{\bm{M}}}

\def\mP{{\bm{P}}}
\def\mQ{{\bm{Q}}}
\def\mR{{\bm{R}}}

\def\mU{{\bm{U}}}
\def\mV{{\bm{V}}}
\def\mW{{\bm{W}}}
\def\mX{{\bm{X}}}
\def\mY{{\bm{Y}}}

\def\mSigma{{\bm{\Sigma}}}

\DeclareMathAlphabet{\mathsfit}{\encodingdefault}{\sfdefault}{m}{sl}
\SetMathAlphabet{\mathsfit}{bold}{\encodingdefault}{\sfdefault}{bx}{n}
\newcommand{\tens}[1]{\bm{\mathsfit{#1}}}

\def\tW{{\tens{W}}}

\usepackage{microtype}
\usepackage{graphicx}
\usepackage{subcaption}
\usepackage{booktabs}
\usepackage{placeins}
\usepackage{multirow}
\usepackage{comment}
\usepackage{float}
\usepackage{listings}
\usepackage{hyperref}
\usepackage{url}
\usepackage{wrapfig}
\usepackage{enumitem}
\usepackage{xcolor}
\usepackage{pgffor}
\usepackage{etoolbox}

\usepackage{amsmath}
\usepackage{amssymb}
\usepackage{mathtools}
\usepackage{amsthm}
\usepackage{pifont}
\usepackage[capitalize,noabbrev,nameinlink]{cleveref}
\usepackage{thmtools}
\usepackage{thm-restate}
\usepackage[textsize=tiny]{todonotes}
\hypersetup{
  colorlinks=true,
  citecolor=blue,
  linkcolor=blue,
  urlcolor=blue
}

\lstdefinestyle{mypython}{
  language=Python,
  basicstyle=\ttfamily\footnotesize,
  keywordstyle=\color{blue},
  commentstyle=\color{teal},
  stringstyle=\color{red!60!black},
  showstringspaces=false,
  breaklines=true,
  frame=single,
  columns=fullflexible,
  keepspaces=true,
  tabsize=4
}

\usepackage{algorithm}
\usepackage{algpseudocode}

\theoremstyle{plain}

\theoremstyle{definition}

\theoremstyle{remark}

\newif\ifdeanon
\deanontrue

\ifdeanon
  \newcommand{\githuburl}{https://github.com/davidgonmar/SLORR}
  \newcommand{\acknowledgements}{
    \begin{ack}
We would like to thank Keyu Wang, Philipp Nazari, and Neeratyoy Mallik for helpful discussions and feedback on this work.
    \end{ack}
  }
\else
  \newcommand{\githuburl}{https://anonymous.4open.science/r/SLORR-E4F8}
  \newcommand{\acknowledgements}{}
\fi

\newcommand{\good}[1]{\textcolor{green!45!black}{#1}}
\newcommand{\bad}[1]{\textcolor{red!80!black}{#1}}
\newcommand{\medium}[1]{\textcolor{orange!85!yellow!70!black}{#1}}

\title{SLORR: Simple and Efficient In-Training Low-Rank Regularization}

\newcommand{\affilmark}[1]{\textsuperscript{\normalfont #1}}

\author{%
  David Gonz\'alez-Mart\'inez\affilmark{1,2,3,4}
  \quad
  Shiwei Liu\affilmark{1,3,4}
  \\
  \affilmark{1}Max Planck Institute for Intelligent Systems
  \quad
  \affilmark{2}University of T\"ubingen
  \\
  \affilmark{3}ELLIS Institute T\"ubingen
  \quad
  \affilmark{4}T\"ubingen AI Center
  \\
  \texttt{david.martinez@tuebingen.mpg.de}
  \quad
  \texttt{sliu@tue.ellis.eu}
}

\ExplSyntaxOn
\cs_if_exist:NF \expandableinput
  {
    \cs_new:Npn \expandableinput #1
      { \use:c { @@input } { \file_full_name:n {#1} } }
  }
\ExplSyntaxOff
\begin{document}

\maketitle

\begin{abstract}
Low-rank factorization is widely used to compress neural networks, but modern models are often not naturally amenable to aggressive factorization without significant accuracy loss. Existing training-time low-rank regularizers can improve compressibility, but they often require SVDs of large weight matrices, modify the model architecture (introducing additional trainable parameters), or rely on stateful cached quantities. To address these limitations, we introduce SLORR, a simple, stateless, and architecture-preserving framework for in-training low-rank regularization, instantiated with two main variants based on the Hoyer sparsity metric and the nuclear norm. SLORR directly regularizes the original weight matrices using GPU-friendly approximations for the forward and backward passes of the regularizers, for which we provide approximation guarantees. We first evaluate SLORR on ImageNet-1K across short-horizon continued training of ResNet-50, ViT-B/16, and ViT-L/16, and pretraining of ResNet-18, where SLORR induces compressibility while introducing less than 8\% training overhead. We further evaluate SLORR-Hoyer in LLM pretraining at 135M and 560M scales: SLORR-trained compressed models preserve performance substantially better than unregularized models while adding less than 1\% average training overhead.
\end{abstract}

\section{Introduction}\label{sec:intro}

Low-rank factorization is a widely used approach for reducing the computational and memory costs of neural networks: a dense weight matrix is replaced by a product of smaller factors, yielding cheaper inference when the retained rank is sufficiently small \citep{ben-noach-goldberg-2020-compressing, 10.5555/2968826.2968968, Jaderberg2014SpeedingUC, NEURIPS2021_f56de5ef}. However, compression quality depends on the spectral properties of weight matrices: discarding meaningful spectral directions can degrade model performance. Thus, while moderate post-training compression is generally possible, modern models are often not naturally amenable to aggressive factorization without significant accuracy loss.

This motivates training-time methods that encourage low-rank structure before compression. A natural formulation is to encourage sparsity in the singular values of weight matrices, concentrating spectra into fewer dominant directions. Although direct spectral regularization can be effective \citep{10.5555/3294771.3294853,DBLP:journals/corr/abs-2004-14566}, large singular value decompositions (SVDs) are needed at every training iteration, which is prohibitively expensive for modern training pipelines.

Several ways of avoiding SVDs have been proposed \citep{DBLP:journals/corr/abs-2004-09031,alkhouri2024structurepreserving,ghosh2025qr}; however, these generally introduce new limitations that hinder practicality. For example, some works replace layers with factorized parameterizations, changing the architecture \citep{alkhouri2024structurepreserving,DBLP:journals/corr/abs-2004-09031}, but this generally increases the number of trainable parameters and alters optimization dynamics. A further class of methods maintains cached spectral quantities that are only periodically updated \citep{ghosh2025qr}, reducing but not eliminating SVD cost, and introducing statefulness and additional hyperparameters.

Our proposed method, SLORR, addresses these limitations. In particular, SLORR operates directly on the original weight matrices without altering the architecture, is SVD-free and efficient, and does not maintain any cached quantities, providing a versatile framework. The key idea is simple: we approximate the spectral quantities needed for the regularizer forward and backward passes using GPU-efficient polar factor approximations. We explore two main variants: SLORR-Hoyer and SLORR-Nuc, which respectively correspond to the squared Hoyer sparsity metric \citep{10.5555/1005332.1044709} and the nuclear norm, two common low-rank penalties. For both variants, the resulting approximations admit guarantees for the regularizer values and their gradients (\cref{prop:SLORR-approx-guarantees}). A summary of directly relevant related works appears in \cref{tab:method_comparison}, with a more extensive discussion in \cref{sec:bg_related_work}.

\begin{table}[t]
\centering
\caption{\textbf{Comparison of different low-rank-inducing regularizers}. Detailed discussion in \cref{sec:bg_related_work}.}
\label{tab:method_comparison}
\setlength{\tabcolsep}{2pt}
\renewcommand{\arraystretch}{0.7}
\begin{tabular}{@{}lccccc@{}}
\toprule
Method & Changes arch. & SVD & Efficiency \& scalability & Stateful & Prior target rank \\
\midrule
SVD-based \citep{10.5555/3294771.3294853} & \good{No} & \bad{Yes} & \bad{Poor} & \good{No} & \good{No} \\
Factorize + reg. \citep{DBLP:journals/corr/abs-2004-09031} & \bad{Yes} & \good{No} & \good{Good} & \good{No} & \good{No} \\
LoRITa \citep{alkhouri2024structurepreserving} & \bad{Yes} & \good{No} & \good{Good} & \good{No} & \good{No} \\
Q3R \citep{ghosh2025qr} & \good{No} & \medium{Periodic} & \medium{Moderate} & \bad{Yes} & \bad{Yes} \\
\textbf{SLORR (ours)} & \good{No} & \good{No} & \good{Good} & \good{No} & \good{No} \\
\bottomrule
\end{tabular}
\end{table}

Empirically, we evaluate SLORR across vision and language-modeling settings. On ImageNet-1K \citep{imagenet}, we study short-horizon continued training of ResNet-50 \citep{resnet}, ViT-B/16, and ViT-L/16 \citep{dosovitskiy2021an}, as well as ResNet-18 pretraining \citep{resnet}. Across these settings, SLORR improves post-training compressibility over unregularized training while adding less than 8\% training overhead. We further evaluate SLORR-Hoyer on LLM pretraining at the 135M and 560M scales, where SLORR-trained models preserve performance substantially better after compression than unregularized models, with less than 1\% average training overhead.

Overall, our contributions are as follows. First, we introduce an SVD-free framework for low-rank-inducing regularization using GPU-friendly approximations. Second, we show how this framework can be used to regularize the spectra of the original weight matrices without architectural changes, cached spectral quantities, or a pre-specified target rank. Third, we demonstrate across vision and language-modeling experiments that SLORR improves post-training low-rank compressibility with small training overhead.

\section{Background and Related Work}\label{sec:bg_related_work}
Various factorization approaches have been extensively studied over time \citep{10.5555/2968826.2968968,Jaderberg2014SpeedingUC,ben-noach-goldberg-2020-compressing}. In this work, we focus on factorization schemes that separate a single matrix multiplication into two sequential ones:
\[
f(\mX;\mW) = \mX \mW \approx \hat{f}(\mX; \mW_0, \mW_1) = (\mX \mW_0) \mW_1,
\]
where $\mX \in \mathbb{R}^{B \times D_i}$ is the input, $\mW \in \mathbb{R}^{D_i \times D_o}$ the weight of the original layer $f$, and $\mW_0 \in \mathbb{R}^{D_i \times P}$ and $\mW_1 \in \mathbb{R}^{P \times D_o}$ are the weights of the new approximate layer $\hat{f}$. With sufficiently small $P$, we can achieve inference speedups and memory reduction. In particular, we focus on those that employ the SVD of weight matrices \citep{jaiswal2025from,ben-noach-goldberg-2020-compressing}, which is optimal in weight space according to the Eckart--Young--Mirsky theorem \citep{golub2013matrix}. In particular, the factors can be obtained from the SVD of $\mW$:
\[
\mW_0 = \mU_{:,:P} \mSigma_{:P,:P}^{1/2},
\qquad
\mW_1 = \mSigma_{:P,:P}^{1/2} \mV_{:,:P}^{\top}.
\]

In convolutional neural networks (CNNs) \citep{LeCun1998GradientbasedLA}, SVD-based factorization can be performed by reshaping parameter tensors into matrices and subsequently decomposing them (normally called channel-wise decomposition) \citep{DBLP:journals/corr/abs-2004-09031, alkhouri2024structurepreserving}; we follow this approach throughout this work. Other recent works, such as activation-aware methods \citep{qinsi2025dobisvd, wang2025svdllm,wang2025svdllmv2optimizingsingular,NEURIPS2021_f56de5ef,yuan2025asvdactivationawaresingularvalue}, introduce different ways of computing factors that can lead to better overall accuracy.

\subsection{Inducing Low-Rank Parameters}\label{sec:inducing_lowrank}
Since compressing layers by SVD truncation essentially amounts to truncating their singular values, the accuracy degradation incurred is directly related to their spectral properties: a desirable property is that the layer weight is (approximately) low-rank. Hence, an intuitive approach to obtaining more compressible layers is to regularize the model to gradually decrease the (approximate) rank of its weight matrices during training. One can roughly think of the rank minimization objective as a sparsity objective on the singular values of the parameters, i.e.,
\[
\mathcal{L}= \mathcal{L}_{\mathrm{task}} + \lambda \sum_i \operatorname{rank}(\mW^i) = \mathcal{L}_{\mathrm{task}} + \lambda \sum_i \|\bm{\sigma}(\mW^i)\|_0,
\]
where $\bm{\sigma}$ denotes the singular values of $\mW^i$. This objective is considered to be intractable \citep{1384521}. Prior work has addressed this by using proxy objectives, such as minimizing the nuclear norm \citep{10.5555/3294771.3294853, DBLP:journals/corr/abs-2004-14566}; however, this generally requires computing the SVDs of each weight at each training step, which is computationally prohibitive for larger models. Several solutions have been proposed.

Some approaches rely on modifying the architecture. For example, \citet{DBLP:journals/corr/abs-2004-09031} obtain the SVD $\mU, \mV, \bm{\sigma}$ of each layer and substitute each layer with a factorized form (retaining all singular values). The authors then incentivize $\bm{\sigma}$ to be sparse through the Hoyer sparsity metric \citep{10.5555/1005332.1044709} (they also explore the $\ell_1$ norm, akin to nuclear-norm regularization). An additional regularizer is required to keep $\mU$ and $\mV$ orthogonal and maintain general SVD semantics, making it less practical. A recent approach, named LoRITa \citep{alkhouri2024structurepreserving}, reparametrizes each layer weight $\mW$ with an ordered product $\prod_i^N\mW_i = \mW$ and penalizes the Frobenius norm of each $\mW_i$, which can be shown to work as a proxy for the nuclear norm minimization objective of the overall $\prod_i^N\mW_i$. A similar approach was proposed in \citet{kliegl2018tracenormregularizationfaster} for depth-2 decompositions. These approaches generally increase the number of trainable parameters and can alter optimization dynamics.

Another line of work reuses computed factors instead of recomputing them at every iteration. For example, the recent Q3R \citep{ghosh2025qr} uses a surrogate objective that can be approximated with periodic SVDs\footnote{Their main experiments are run with a refresh period of 5 iterations, which is expensive. In our main Q3R experiments, we also use 5; however, we found that using a larger period (which is less expensive) can also work well (\cref{sec:app_q3r_refresh_period_ablation}). The main concern is scalability. As one uses larger models, maintaining a low cost with Q3R requires enlarging the refresh period. It is expected that, with a sufficiently large refresh period, its effectiveness will be reduced.}. However, it is stateful (as it needs to store the temporary factors until they are recomputed). It also requires specifying a target rank in advance, making hyperparameter tuning challenging.

Other approaches for approximate computations exist: \citet{ijcai2025p0625} propose randomized differentiable surrogates for generalized low-rank regularization, applying them to tasks including matrix completion, video foreground/background separation, and denoising. While the method is SVD-free in theory, it is still impractical, requiring backpropagation through long chains of matrix multiplications. Moreover, its practical implementation\footnote{\url{https://github.com/naiqili/EDLRR}} still requires a (smaller) SVD and matrix inverse. This makes it infeasible for weight regularization, even at moderate scales. \citet{qin2025lowrankprehabpreparingneural} propose a method to make models amenable to activation-aware methods (in particular, SVD-LLM \citep{wang2025svdllm}), but it is focused on very short-horizon post-training, while we focus on longer training periods (from continued training to pretraining).

Relatedly, low-rank induction can also happen (either explicitly or implicitly) at the optimizer level. For example, \citet{zimmer2024compressionawaretrainingneuralnetworks} use a Frank--Wolfe-based optimizer to promote low rank in weight matrices using partial SVDs. \citet{dolatabadi2026numuonnuclearnormconstrainedmuoncompressible} find that, surprisingly, Muon-trained LLMs are more compressible and that truncated Muon updates can further improve compressibility. Here, we restrict our focus to explicit regularization techniques, typically formulated either as additional loss terms or as decoupled penalties applied outside the primary gradient update, analogous to decoupled weight decay \citep{loshchilov2018decoupled}. In \cref{tab:method_comparison}, we compare relevant direct low-rank regularization methods. Our proposed approach is simpler: it operates on the original model, is SVD-free and efficient in practice, and is stateless.

Lastly, a related area, but fundamentally different in nature, is that of low-rank training, where the rank is generally pre-defined before training, and a low-rank architecture is directly trained. Different variants, such as full-rank warmup or combinations with sparse training, and aspects like initialization, have been investigated over time \citep{lialin2023relorahighranktraininglowrank,khodak2021initialization,kamalakara2022exploringlowranktraining,wei2024building,han2024sltrain,li2025lostlowranksparsepretraining}.

\section{Our Framework}\label{sec:our_framework}

For a weight matrix $\mW \in \mathbb{R}^{I \times O}$, we use $\mU \in \mathbb{R}^{I \times R}, \mSigma \in \mathbb{R}^{R \times R}, \mV \in \mathbb{R}^{O \times R}$ to denote its thin SVD factors ($\mW = \mU \mSigma \mV^{\top}$), where $R = \operatorname{rank}(\mW) \leq \min (O,I)$. Additionally, we use $\bm{\sigma}$ to denote the vector of singular values ($\bm{\sigma} = \operatorname{diag}(\mSigma)$). Since the SVD is not unique, when mentioning it or any factor, we implicitly refer to an arbitrary valid choice.

We view the problem from the perspective of inducing sparsity in singular values. We adopt (1) the $\ell_1$ norm of singular values, i.e., the nuclear norm, and (2) the squared $\ell_1/\ell_2$-ratio (the nonsquared version is also known as the unnormalized Hoyer metric \citep{10.5555/1005332.1044709}) of singular values, which is effective as a sparsity-inducing regularizer \citep{Yang2020DeepHoyer} (the nonsquared version has also been shown to work well in factorization-based low-rank regularization \citep{DBLP:journals/corr/abs-2004-09031}). We have:
\[
\mathcal{L}_{\mathrm{hoyer}}
= \sum_i \|\bm{\sigma}^i\|^2_1 / \|\bm{\sigma}^i\|^2_2
= \sum_i \|\mW^i\|^2_* / \|\mW^i\|^2_F, \quad \mathcal{L}_{\mathrm{nuc}}
= \sum_i \|\mW^i\|_*,
\]
where $i$ denotes the layer index. From now on, we will omit layer indices for brevity.

We optimize these regularizers together with the task loss. Strictly speaking, the regularizers are not differentiable everywhere. We use the natural backpropagation rule, which corresponds to the minimum-Frobenius-norm element of the Clarke generalized gradient (see \cref{sec:app_backprop_computations} for the derivations). We use gradient notation for simplicity. We have that:
\begin{equation}\label{eq:SLORR_grads}
\nabla{\mathcal{L}_{\mathrm{hoyer}}}({\mW}) = 2\frac{\|\mW\|_*}{\|\mW\|_F}\left(\frac{1}{\|\mW\|_F}\mU\mV^{\top} - \frac{\|\mW\|_*}{\|\mW\|_F^3}\mW\right), \quad \nabla{\mathcal{L}_{\mathrm{nuc}}}({\mW}) = \mU \mV^\top.
\end{equation}

\begin{wrapfigure}{r}{0.42\textwidth}
\vspace{-1.5\baselineskip}
\begin{minipage}{\linewidth}
\begin{lstlisting}[
  style=mypython,
  basicstyle=\ttfamily\tiny,
  keywordstyle={\color[HTML]{003B73}\bfseries},
  commentstyle={\color[HTML]{5A5A5A}\itshape},
  stringstyle={\color[HTML]{7A3E00}},
  identifierstyle={\color[HTML]{111111}},
  emph={LowRankRegularizer,Function},
  emphstyle={\color[HTML]{267F99}},
  emph={[2]forward,backward,norm,polar_express,sum,save_for_backward},
  emphstyle={[2]\color[HTML]{795E26}},
  emph={[3]torch},
  emphstyle={[3]\color[HTML]{005A5A}},
  caption={Example PyTorch implementation of SLORR-Hoyer.},
  label={lst:lowrank_reg}
]
class LowRankRegularizer(torch.autograd.Function):
    @staticmethod
    def forward(ctx, W, eps, steps=6):
        frob = torch.linalg.norm(W, ord="fro") + eps
        UV = polar_express(W, steps=steps)
        nuc = torch.sum(W * UV)
        reg = nuc / frob
        ctx.save_for_backward(W, UV, frob, nuc, reg)
        return reg ** 2

    @staticmethod
    def backward(ctx, grad_output):
        W, UV, frob, nuc, reg = ctx.saved_tensors
        G = (UV / frob) - (nuc / (frob**3)) * W
        return grad_output * 2 * reg * G, None, None

\end{lstlisting}
\end{minipage}
\vspace{-1.5\baselineskip}
\end{wrapfigure}

Both gradients have terms that are generally computed through SVDs, which is prohibitively expensive to do at every iteration. Fortunately, one can (through iterative methods) approximate every term needed in \cref{eq:SLORR_grads}. The resulting approximations have provable convergence guarantees (\cref{prop:SLORR-approx-guarantees}).

\paragraph{Computing the polar factor. } Notice that we do not need to compute the factors of $\mU \mV^{\top}$ separately; we only need their product\footnote{Note that we are using the left and right singular vectors corresponding to the nonzero singular values of the matrix. This quantity is unique for a given matrix \citep{doi:10.1137/S0895479801394623}.}, which is generally known as the generalized polar factor (we use ``polar factor'' for short) of $\mW$. It can be computed efficiently on GPUs using iterative methods (e.g., \citep{HIGHAM19943}; a similar approach is used in the Muon optimizer \citep{jordan2024muon} to obtain polar factors of gradients). In particular, we use the recent Polar Express \citep{amsel2026the}, which approximates the polar factor iteratively, to efficiently compute an approximation of $\mU \mV^{\top}$.

Polar Express depends on several hyperparameters, including the number of iterations $T$, most of which we set to their recommended values. It first normalizes the matrix so that its largest singular value is at most 1. This is achieved by dividing by the Frobenius norm, i.e., it uses $\mW / (\|\mW\|_F + \varepsilon)$\footnote{In \cref{lst:lowrank_reg}, normalization is performed inside the Polar Express function.}. Moreover, it also needs a lower bound estimate of the smallest nonzero singular value of the normalized matrix, $\ell$, which we set to $10^{-3}$ following their recommendations. We provide a more detailed discussion of our usage of the Polar Express in \cref{sec:app_polar_express_details}.

\paragraph{Computing the nuclear and Frobenius norms. } After we have (an approximation of) $\mU \mV^{\top}$, we can compute an estimate of the nuclear norm with elementwise multiplications and a summation. In particular, it holds that (for any exact $\mU \mV^{\top}$)
\begin{align*}
\operatorname{sum}(\mW \odot \mU \mV^\top) = \operatorname{tr}(\mW^{\top}\mU\mV^{\top}) = \operatorname{tr}(\mV\mSigma\mU^{\top}\mU\mV^{\top}) = \operatorname{tr}(\mSigma) = \|\mW\|_*,
\end{align*}
where $\odot$ denotes elementwise multiplication, and the intermediate steps follow from the cyclic property of the trace operator. The Frobenius norm is simply $\|\mW\|_F = \sqrt{\sum_{i,j} \mW^2_{i,j}}$.

\paragraph{Combining everything.}
We now have all the elements needed to compute approximate values and approximate backpropagation rules for both variants, which can be implemented, for example, as a custom PyTorch function. Forward and backward computations of SLORR-Hoyer are summarized in \cref{lst:lowrank_reg}\footnote{As shown in \cref{lst:lowrank_reg}, in practice, we stabilize the forward and backward expressions separately (another option is to use a stabilized forward expression and derive the exact backward expression from it) by substituting $\|\mW\|_F \mapsto \|\mW\|_F + \varepsilon$. This is done to avoid a possible division by zero.}; SLORR-Nuc admits a similar implementation. Note that we do not backpropagate through the Polar Express iterations. Instead, the approximate polar factor is used in place of the $\mU\mV^\top$ term that appears in the regularizer values and in the gradient rules in \cref{eq:SLORR_grads}. During training, we compute the task loss and regularization loss, weighted by the regularization strength $\lambda$, and use gradient-based optimization as usual\footnote{The autograd-based implementation is generally easy to integrate into existing codebases. However, in practice, for PyTorch-based implementations, one can save memory by following the approach in \cref{sec:app_memory_efficient_approach}. This approach is used for our vision experiments.}. In the case of Adam \citep{kingma2017adammethodstochasticoptimization}, one can also use a decoupled low-rank regularization term, similar to Q3R \citep{ghosh2025qr}; see \cref{alg:decoupled-SLORR}.

\paragraph{Computational cost. } In terms of matrix multiplication FLOPs (the main cost), each iteration of Polar Express (with our exact settings) with a matrix of shape $(M, N)$ uses $4P^{2}Q+2P^{3}$ FLOPs, where $P=\min\{M,N\}$ and $Q=\max\{M,N\}$, and we generally use 6 iterations. For convolutional layers, we reshape kernels into matrices (details in \cref{sec:app_interpreting_conv_layers}); in which case we get $M = C_o$ and $N = C_{i} K_w K_h$, denoting the output and input channels, as well as the width and height of the convolutional kernel, respectively. We discuss practical SLORR overhead measures in \cref{sec:overhead_exps_main} and \cref{sec:llm_experiments_main}, where we find that it generally remains below or around 8\% in our ImageNet-scale settings and 1\% in LLM pretraining at the 135M and 560M scales, whereas SVD remains prohibitively expensive.

\paragraph{Approximation guarantees.} We remark that the computed forward and backward passes of both variants are approximate in practice. Polar Express has worst-case approximation guarantees, which can be used to provide guarantees on the values and gradients of both SLORR variants (\cref{prop:SLORR-approx-guarantees}). We provide further discussion of this result and its proof in \cref{sec:app_approx_guarantees}.

\begin{restatable}[Approximation guarantees]{proposition}{SLORRApproxGuarantees}\label{prop:SLORR-approx-guarantees}
Let $\mW \neq \bm 0$ have thin SVD $\mW = \mU \mSigma \mV^\top$, and assume the nonzero singular values of a normalized version of $\mW$ lie in $[\ell,1]$. Let $\widehat{\mP}$ denote the Polar Express approximation of the polar factor $\mU \mV^\top$ after $T$ steps, and define $\delta = \left|1-\ell^2\right|^{(q+1)^T}$, where $d=2q+1$ is the Polar Express polynomial degree. Additionally, let $\hat n := \operatorname{tr}(\mW^\top \widehat{\mP})$ denote the approximate nuclear norm estimate, and let $n := \|\mW\|_*$ and $f := \|\mW\|_F$. Moreover, let $\mR$ denote the exact SLORR-Hoyer gradient and $\widehat{\mR}$ its approximation. Then
\[
\begin{aligned}
\frac{|\hat n - n|}{n} &\le \delta,\quad
\|\widehat{\mP} - \mU \mV^\top\|_2 \le \delta
&&\text{(SLORR-Nuc)},\\
\frac{|\hat n^2/f^2 - n^2/f^2|}{n^2/f^2} &\le \delta(2+\delta),\quad
\|\mR - \widehat{\mR}\|_2 \le 2\delta(2+\delta)
\left(
\frac{n}{f^2}
+
\frac{n^2}{f^4}\|\mW\|_2
\right)
&&\text{(SLORR-Hoyer)}.
\end{aligned}
\]
In particular, errors vanish as $T \to \infty$.
\end{restatable}

\paragraph{Practical remarks.}
The singular value range in \cref{prop:SLORR-approx-guarantees} is a theoretical assumption on the normalized matrix, inherited from the Polar Express analysis in \citet{amsel2026the}. In practice, Polar Express treats $\ell$ as an approximate lower bound parameter; the authors report that inaccurate guesses are typically not severe. Therefore, the displayed worst case bound should be interpreted as applying to normalized nonzero singular values lying in $[\ell,1]$, while our choice $\ell=10^{-3}$ follows their recommended value. We also emphasize that this is a worst case bound in an idealized setting. In practice, a smaller number of iterations works well even when the bound is not tight, and floating point error and implementation details may introduce additional approximation error. We follow the practical implementation choices of \citet{amsel2026the}, including their recommended parameter settings and numerical stabilizations. Further discussion is provided in \cref{sec:app_polar_express_details} and \cref{sec:app_approx_guarantees}.

\subsection{Understanding SLORR-Hoyer}\label{sec:understanding_slorr_hoyer_main}

Our analysis is performed in an exact regularizer-only gradient descent setting.  In this section, $\eta>0$ denotes the step size of the regularizer-only update, with the training-time regularization strength absorbed into it. Additional comments and proofs appear in \cref{sec:app_slorr_hoyer_theoretical_analysis}. We let $\mW = \mU\operatorname{diag}{(\bm \sigma)}\mV^{\top}$ be a weight matrix and its thin SVD. Additionally, let $\mW^+$ be the weight matrix after a regularizer-only gradient step from $\mW$, and define $\bm{\sigma}^+$, $\mU^+$, and $\mV^+$ as its thin SVD. We assume ordered singular values in decreasing order ($\bm{\sigma}_1\ge\cdots\ge\bm{\sigma}_r>0$) and $\mW \neq \mathbf{0}$. We first remark that SLORR-Hoyer is scale-invariant (in the same way that the original version for vector sparsity is \citep{10.5555/1005332.1044709,Yang2020DeepHoyer}).\footnote{This is easy to see: $\| \alpha \mW\|_*^2 / \| \alpha \mW\|_F^2
= \alpha^2 \| \mW\|_*^2 / \alpha^2 \|\mW\|_F^2
= \| \mW\|_*^2 / \|\mW\|_F^2$.} Intuitively, this means it aims to redistribute the spectral energy rather than grow or shrink it. Note, however, that in a practical discretized optimization setting, the overall norm of the weight tends to increase (note that the gradient is orthogonal to the weight matrix):
\[
\begin{aligned}
\|\mW - \eta\nabla \mathcal{L}_{\mathrm{hoyer}}(\mW)\|_F^2
&=
\|\mW\|_F^2
-2\eta \left\langle \mW, \nabla \mathcal{L}_{\mathrm{hoyer}}(\mW) \right\rangle_F
+\eta^2 \|\nabla \mathcal{L}_{\mathrm{hoyer}}(\mW)\|_F^2
\geq
\|\mW\|_F^2 .
\end{aligned}
\]

In preliminary experiments, we tried a norm-correction step that did not seem to help in practice.

\begin{restatable}[SLORR-Hoyer shrinkage and growth of singular values]{proposition}{noGradBehaviorProp}\label{prop:shrinkandgrouth}
For a sufficiently small regularizer-only gradient step, it holds that
\[
\bm{\sigma}^+_i \diamond\bm{\sigma}_i \iff \bm{\sigma}_i \diamond\frac{\|\mW\|_F^2}{\|\mW\|_*}, \quad \text{for $\diamond \in \{<, >, =\}$.}
\]
\end{restatable}

\Cref{prop:shrinkandgrouth} suggests that SLORR-Hoyer enlarges ``large'' singular values and shrinks ``small'' ones, where ``large'' and ``small'' are determined by the dynamic quantity $\|\mW\|_F^2/\|\mW\|_*$, which varies as the weights evolve. Intuitively, this concentrates the spectrum onto the already large singular values. In turn, under certain assumptions, it induces low-rank compressibility (\cref{prop:energy_redistr}).

\begin{restatable}[Sufficiently small steps of SLORR-Hoyer induce compressibility]{proposition}{energyRedistributionProp}\label{prop:energy_redistr}
Let $\tau=\|\mW\|_F^2/\|\mW\|_*$. For a sufficiently small regularizer-only gradient step, if $\bm{\sigma}_k>\tau$, then
\[
\frac{\|\mU^+_{:,:k}\operatorname{diag}(\bm{\sigma}_{:k}^+)\mV_{:,:k}^{+\top}-\mW^+\|_F^2}{\|\mW^+\|_F^2}
<
\frac{\|\mU_{:,:k}\operatorname{diag}(\bm{\sigma}_{:k})\mV_{:,:k}^\top-\mW\|_F^2}{\|\mW\|_F^2}.
\]
\end{restatable}

\section{Experiments}\label{sec:experiments_main}

Our goal is to evaluate the effectiveness of SLORR in inducing compressibility during training. Our experiments cover two main domains: image classification on ImageNet-1K \citep{imagenet}, including both continued training from a checkpoint and pretraining, and LLM pretraining. Our general pipeline consists of first training regularized and unregularized models, and then compressing them and assessing their retained performance. Unless stated otherwise, regularization and compression are applied to all linear and convolutional layers, except for the first and last layers; for transformers, this corresponds to the layers inside the transformer blocks. In our experiments, SLORR-Hoyer-D denotes the decoupled variant, following \cref{sec:app_decoupled_impl_details}. 

Our code is available at \url{\githuburl}.

\subsection{Image Classification Experiments}\label{sec:img_classification_exps}
We first evaluate SLORR on ImageNet-1K classification. We cover two settings: pretraining and continued training, across four architectures. We compare against recent baselines, which have largely focused on the image classification domain. In total, we perform more than 150 runs across settings.

\paragraph{Baselines and comparisons. } We compare SLORR against Q3R \citep{ghosh2025qr} and LoRITa \citep{alkhouri2024structurepreserving}. These baselines are used because: (1) they are the most recent relevant methods, to the best of our knowledge, and (2) they allow us to compare different objectives, whereas factorization-based approaches \citep{DBLP:journals/corr/abs-2004-09031} generally have similar objectives to ours but optimize them after altering the architecture.

For Q3R, following their experiments, the regularization term is added to the gradients in a decoupled form. We sweep over both the regularization strength and their so-called ``target rank'', fixing its refresh interval to 5 (following their main experiments). We note, however, that in an ablation, we found that higher (and cheaper) refresh intervals can also work well in the continued training setting (see \cref{tab:efficiency_measures_per_method} for overhead measurements and \cref{sec:app_q3r_refresh_period_ablation} for the Q3R ablation). While the authors of Q3R recommend smaller regularization strengths for their settings ($\lambda \in [0.001, 0.01]$), we found in our initial experiments that higher values provided substantially better performance. For LoRITa, we follow the approach described in their paper; however, we use a more efficient and practical implementation (their original implementation significantly increases runtime and memory costs). We also adapt it to allow initialization from a trained checkpoint, which is necessary for our continued training setting. We fix $N = 2$ and vary the regularization strength. For SLORR, we use 6 Polar Express steps (following their recommended setup, which we found to work well; an ablation is presented in \cref{sec:app_polar_express_details}) and sweep over the regularization strength $\lambda$. We also include an unregularized baseline following the same training recipe. Details are given in \cref{sec:app_vision_experiment_details}. While LoRITa (fixing $N = 2$) and SLORR have one hyperparameter, Q3R has two that play a strong role together, which makes tuning significantly more difficult. Details regarding selection, which mainly followed a best-effort manual approach, as well as the final configurations, can be found in \cref{sec:app_vision_experiment_details} and \cref{sec:app_additional_vision_results}, respectively.

\begin{figure*}[t]
  \centering
  \includegraphics[width=\textwidth]{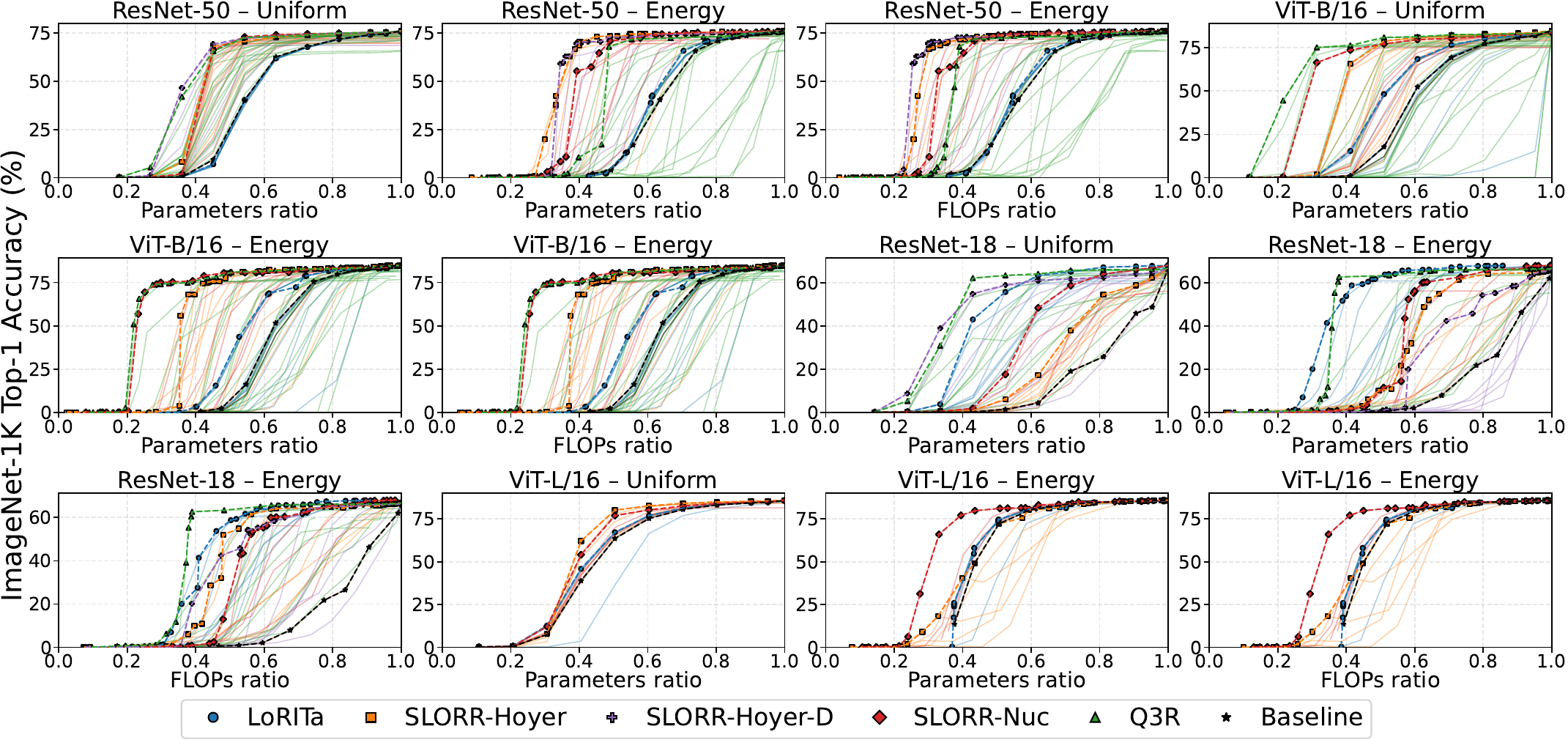}
  \caption{\textbf{Accuracy--compression curves of regularized ResNet-50, ViT-B/16, ViT-L/16, and ResNet-18 models on ImageNet-1K.} The y-axis reports ImageNet-1K validation top-1 accuracy of the compressed model, and the x-axis reports either the retained parameter ratio or the retained FLOPs ratio, as indicated in each panel. Factorization is performed using the energy and uniform rank-selection criteria (see \cref{sec:experiments_main} for details). Faint curves indicate different hyperparameter selections, while highlighted curves indicate Pareto frontiers. For uniform compression, the parameter and FLOPs ratios are equal up to a small shift, so separate FLOPs ratio panels are omitted.}
  \label{fig:main_compression_results}
\end{figure*}

\paragraph{Compression.} After training, we perform compression using two different approaches. The first is uniform compression, in which every layer retains (approximately) a uniform ratio of parameters and FLOPs across layers. Separately, we perform energy-based truncation, in which a retention threshold $\rho$ is used to compress each layer to the minimum rank such that the weight relative squared Frobenius approximation error is at most $1 - \rho$. Details are given in \cref{sec:app_vision_experiment_details}.

\paragraph{ImageNet-1K continued training.}
We perform continued training on ViT-B/16 \citep{dosovitskiy2021an} and ResNet-50 \citep{resnet} using pretrained checkpoints on the ImageNet-1K dataset \citep{imagenet}. We train for 30 epochs and compare against Q3R and LoRITa. We use the AdamW optimizer \citep{loshchilov2018decoupled}, a batch size of 1024, cosine learning rate decay, and linear warmup. For ViT-B/16, we use base learning rates of $10^{-5}$ and $5\times 10^{-5}$; for ResNet-50, we use a base learning rate of $5\times 10^{-5}$. We also use the standard data augmentation protocol for each model. The exact training details and motivation for these hyperparameters are provided in \cref{sec:app_vision_experiment_details}.

The base learning rates are chosen to be sufficiently small so that an unregularized baseline maintains its pretrained performance. However, we found that Q3R did not perform well on ViT-B/16 with the smaller learning rate, so we also conducted experiments using the larger learning rate. We additionally perform continued training on ViT-L/16 \citep{dosovitskiy2021an}, comparing SLORR with LoRITa and an unregularized baseline. Unless otherwise noted, we use the same settings as in the ViT-B/16 experiments, with a learning rate of $10^{-5}$. We use a batch size of 896 because LoRITa caused out-of-memory errors at a batch size of 1024. To provide a matched comparison, we also use a batch size of 896 for SLORR.

For SLORR, we used both the SLORR-Nuc and SLORR-Hoyer variants. For ResNets, we also include the SLORR-Hoyer-D variant, which applies regularization in a decoupled way (see \cref{sec:app_decoupled_impl_details}). For ViT-B/16, we tried it preliminarily, but it did not seem to help, so we did not continue exploring it. We did not explore the decoupled variant for SLORR-Nuc.

\paragraph{ImageNet-1K pretraining. } Additionally, we pretrain ResNet-18 for 110 epochs on ImageNet-1K (the rest of the settings are similar to ResNet-50, but we use a base learning rate of $10^{-3}$). We compare against Q3R and LoRITa and conduct experiments with SLORR-Nuc, SLORR-Hoyer and SLORR-Hoyer-D.

\begin{wrapfigure}{r}{0.52\textwidth}
    \centering
    \vspace{-3.5em}

    \begin{minipage}{\linewidth}
        \centering
        \includegraphics[width=\linewidth]{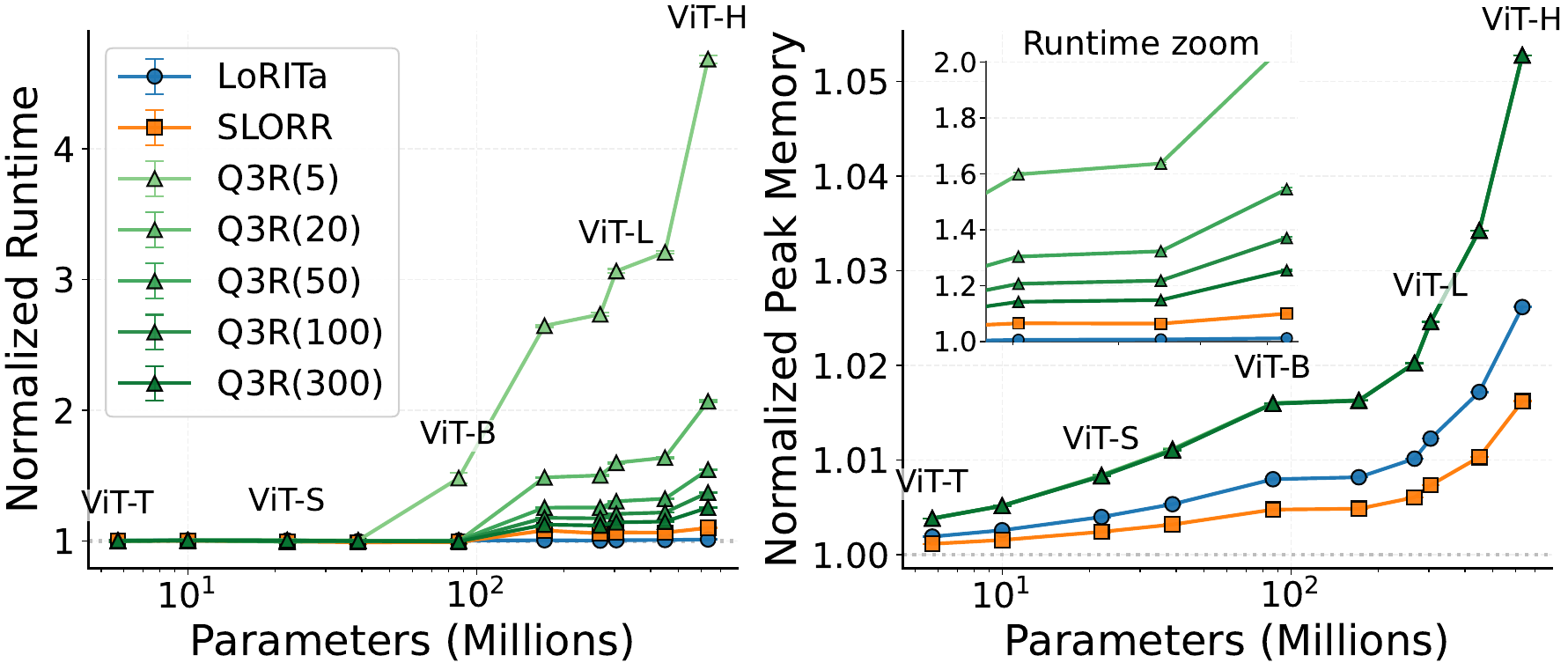}
\captionof{figure}{\textbf{Time and memory overhead of different regularizers across ViT model scales.} Runtime (left) and peak memory (right) are normalized to unregularized training. For Q3R, the number in parentheses denotes the SVD refresh interval. The inset shows a zoomed view of the normalized-runtime curves near ViT-H scale.}
        \label{fig:vit_gradual_overhead}
    \end{minipage}

    \vspace{0.5em}

    \begin{minipage}{\linewidth}
        \centering
        \small
        \renewcommand{\arraystretch}{1.0}
            \setlength{\aboverulesep}{0.2ex}
    \setlength{\belowrulesep}{0.2ex}
    
    \begin{tabular}{llrr}
    \toprule
    Model & Method & Time & Mem. (GB) \\
    \midrule
    \multirow{5}{*}{ViT-B/16} & SLORR & $\times1.037$ & $+0.34$ \\
     & LoRITa & $\times1.007$ & $+0.57$ \\
     & Q3R(5) & $\times1.656$ & $+1.14$ \\
     & Q3R(50) & $\times1.106$ & $+1.14$ \\
     & Q3R(300) & $\times1.055$ & $+1.14$ \\
    \midrule
    \multirow{5}{*}{ResNet-50} & SLORR & $\times1.048$ & $+0.09$ \\
     & LoRITa & $\times1.005$ & $+0.13$ \\
     & Q3R(5) & $\times1.222$ & $+0.31$ \\
     & Q3R(50) & $\times1.062$ & $+0.31$ \\
     & Q3R(300) & $\times1.049$ & $+0.31$ \\
    \midrule
    \multirow{5}{*}{ViT-L/16} & SLORR & $\times1.055$ & $+1.21$ \\
     & LoRITa & $\times1.010$ & $+2.02$ \\
     & Q3R(5) & $\times3.059$ & $+4.05$ \\
     & Q3R(50) & $\times1.301$ & $+4.04$ \\
     & Q3R(300) & $\times1.139$ & $+4.04$ \\
    \bottomrule
    \end{tabular}
        \captionof{table}{\textbf{Training overhead in our runs.} We report normalized time vs. unregularized training and increase in peak GPU memory. For Q3R, the number in parentheses denotes the SVD refresh interval. All methods achieve near-zero overhead on ResNet-18.}
            \vspace{-3.5em}
        \label{tab:main_training_overhead}
    \end{minipage}
    
\end{wrapfigure}

\paragraph{Compression results. } We report results for the different settings in \cref{fig:main_compression_results}. The accuracy--compression Pareto frontier is highlighted, with fainter lines showing trajectories for each hyperparameter configuration. In general, for all methods, there appears to be an accuracy/compressibility tradeoff. Under our settings, we find that at least one variant of SLORR generally performs close to Q3R, sometimes surpassing it. LoRITa is relatively weak in most of our continued training settings, though it is more effective in the ResNet-18 pretraining setting. Note that, in some circumstances, some runs of different methods underperform the baseline in terms of compressibility (in particular, Q3R seems not to perform well on the ViT-B/16 runs with smaller learning rates; see \cref{tab:vitb16_imagenet_acc} in \cref{sec:app_additional_vision_results}). Interestingly, each exact setting appears to favor different methods, including different SLORR variants, suggesting that there might not be a single one-size-fits-all objective.

\subsubsection{Training Overhead for Classification Models}\label{sec:overhead_exps_main}

\paragraph{Overhead in our training runs. } We measure the overhead in a controlled environment (full details and measurements can be found in \cref{sec:app_overhead_details}). In \cref{tab:main_training_overhead}, we report the overhead of different methods\footnote{We remark that, in general, the training overhead is related to different factors. For example, a slow data-loading pipeline might mask a large overhead. Moreover, smaller batch sizes will generally suffer from higher relative overhead, as the regularization time is fixed. Similarly, implementation specifics can slightly alter the exact results.}. We find that SLORR consistently maintains a small training overhead of less than 8\% on average\footnote{Our implementation of the decoupled variant has a slightly larger overhead. For the default variant, overhead is less than 6\% on average}. Furthermore, in these experiments, SLORR consistently achieves the smallest memory overhead. Although we use a refresh period of 5 in our Q3R experiments, which is expensive, we also note that larger periods, significantly cheaper than what \citet{ghosh2025qr} used in their main experiments, appear to work well in the settings we tried (as discussed earlier); these timings are included in the table for transparency.

\paragraph{Scaling up: overhead across different ViT scales. }
We evaluate how regularization overhead scales on modern hardware by measuring the relative time and memory costs on B200 GPUs from ViT-T through ViT-H scale (\cref{fig:vit_gradual_overhead}). This reveals a key scaling tradeoff: while Q3R with large periods can be efficient, larger refresh periods are needed to maintain efficiency as model size increases. While large refresh periods work in our settings, we expect sufficiently infrequent refreshes to eventually reduce regularization quality. Although we do not reach that ``saturation point'' in our scales, this is an important aspect to take into account in terms of scalability.\footnote{While Q3R relies on exact SVDs, it might be an interesting avenue for future work to try approximate SVDs, which might improve scalability.}

\subsection{LLM Pretraining}\label{sec:llm_experiments_main}

We now study regularized LLM pretraining using SLORR-Hoyer. In particular, we train standard Llama-like \citep{touvron2023llamaopenefficientfoundation} transformer models on FineWeb-Edu \citep{penedo2024finewebdatasetsdecantingweb}, using the AdamW optimizer and distributed data parallelism (DDP), following standard practice. Exact training details are provided in \cref{sec:app_llm_training_details}. Our main models are trained on approximately 20 training tokens per parameter, following the compute-optimal scaling rule of \citet{hoffmann2022trainingcomputeoptimallargelanguage}. We train 135M and 560M models using both unregularized standard pretraining and SLORR-Hoyer for different regularization strengths $\lambda$. We find that, across both scales, SLORR incurs less than 1\% training overhead on average; see \cref{sec:app_llm_training_overhead} for concrete details.

After training, we compress models using both plain SVD and the recent SVD-LLM \citep{wang2025svdllm}, which uses activation statistics to compute a whitening transformation before compressing each layer, generally enabling stronger compression\footnote{We use a slightly modified version. Details are provided in \cref{sec:app_llm_training_details}.}. Plain SVD is known to be impractical for LLMs, as degradation is significant even at very small compression ratios \citep{wang2025svdllm}. After compression, we evaluate the models. For both scales, we report validation perplexity on a held-out FineWeb-Edu set of approximately 2M tokens. For the 560M model, we additionally report zero-shot downstream task accuracy using the \texttt{lm-evaluation-harness} library \citep{eval-harness}. Downstream tasks include ARC-Easy and ARC-Challenge \citep{clark2018thinksolvedquestionanswering}, HellaSwag \citep{zellers-etal-2019-hellaswag}, LAMBADA \citep{paperno-etal-2016-lambada}, OpenBookQA \citep{mihaylov-etal-2018-suit}, and PIQA \citep{bisk2019piqareasoningphysicalcommonsense}.

\begin{figure}[t]
    \centering

    \begin{subfigure}[t]{\linewidth}
        \centering
        \includegraphics[width=\linewidth]{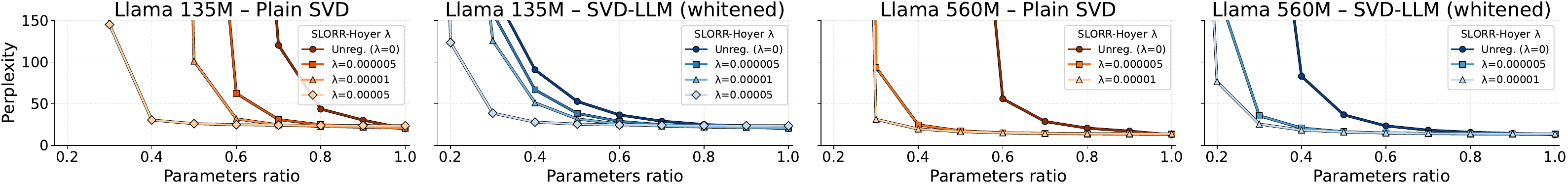}
    \end{subfigure}

    \vspace{0.1em}

    \begin{subfigure}[t]{\linewidth}
        \centering
        \includegraphics[width=\linewidth]{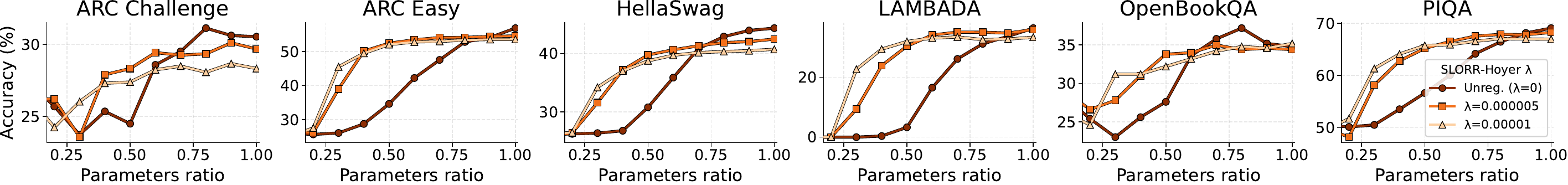}
    \end{subfigure}

    \vspace{0.1em}

    \begin{subfigure}[t]{\linewidth}
        \centering
        \includegraphics[width=\linewidth]{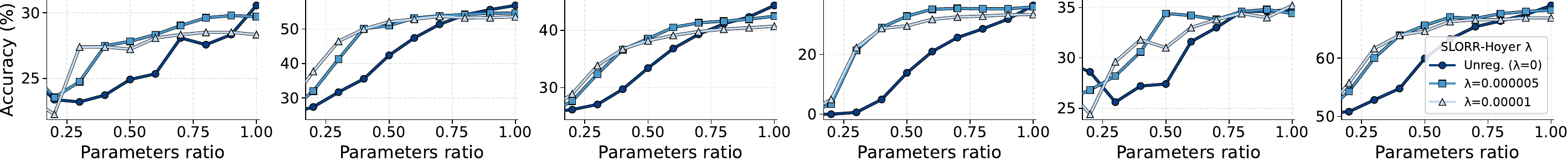}
    \end{subfigure}

 \caption{\textbf{LLM compression results with SLORR-Hoyer regularization.} Colors indicate the SLORR-Hoyer regularization strength $\lambda$ used during pretraining. $\lambda=0$ corresponds to the unregularized baseline. \textbf{Top row:} FineWeb-Edu validation perplexity for Llama 135M and Llama 560M after compression with either plain SVD or SVD-LLM whitening, as indicated in each panel title. Perplexity values are clipped at 150 for visibility; full tabular results are reported in \cref{sec:app_complete_llm_results}. \textbf{Middle row:} Downstream zero-shot accuracy for Llama 560M compressed with plain SVD. \textbf{Bottom row:} Downstream zero-shot accuracy for Llama 
 560M compressed with SVD-LLM whitening.}
    \label{fig:compressibility_llm_combined}
\end{figure}

\paragraph{Compression results. } Results for the 135M and 560M models are presented in \cref{fig:compressibility_llm_combined}. The reported parameter ratios are computed over the transformer blocks. In particular, embeddings and the final layer are not compressed, following common practice in this area.  We find that SLORR-Hoyer improves compressibility under both plain SVD and SVD-LLM compression. Plain SVD is particularly brittle for LLMs, as perplexity degrades sharply even when only a small fraction of block parameters is removed. Although SLORR-trained models have slightly higher uncompressed perplexity, they outperform unregularized models after mild compression, e.g., 10\% parameter removal, corresponding to a retained parameter ratio of 0.9. The gap widens significantly at lower parameter count ratios.

For the 560M model, we observe similar trends in downstream task performance. While the unregularized model quickly suffers from significant degradation, regularized models maintain substantially higher downstream performance in general, even at higher compression ratios, at the cost of a slight hit in uncompressed performance on some tasks.

\paragraph{Longer training. } As noted above, our main models are trained to the compute-optimal token budget \citep{hoffmann2022trainingcomputeoptimallargelanguage}. We also train 135M models for longer, specifically for 4 and 8 times the compute-optimal token budget, to explore whether the gains transfer to longer training. We find that, as models are trained for longer, they become less compressible in our setting, both under unregularized and regularized training. However, regularized models remain more compressible. These results are shown in \cref{sec:llm_overtraining}, with tabular values in \cref{sec:app_complete_llm_results}.

\paragraph{Visualizing the effect of SLORR on weight spectra. } SLORR promotes low rank with the objective of making models amenable to factorization. Naturally, this is directly reflected in the visual properties of weight spectra, where stronger regularization induces faster singular value decay. We include plots for all Llama 135M layers in the appendix (\cref{fig:llama135m_singular_values}), where the spectral concentration effect is apparent. Perhaps surprisingly, there are visible spectral differences in the embedding matrix, but not in the final linear layer. Since the embedding layer was not regularized, this indicates that modifying the spectra of regularized layers can affect other parts of the model.

\section{Discussion}\label{sec:discussion}

\paragraph{Limitations. } Weight low-rank regularization may not be effective in every setting. In exploratory experiments, we found that fine-tuning LLMs with low-rank regularization while preserving general capabilities was challenging, suggesting that the data and training setup are important, as one might expect. Additionally, there is generally a tradeoff between final accuracy and compressibility, and excessive regularization can cause training instabilities, particularly for the Hoyer variant. Moreover, although our experiments are significantly more extensive than those of prior work (more than 150 runs on ImageNet-1K), the performance of regularizers can vary depending on the training setting and hyperparameters. Similarly, we did not explore every SLORR variant in every setting. Therefore, although we provide a broad best-effort comparison, our results should not be interpreted as a definitive ranking of all methods. Finally, our theory, while informative, relies on specific assumptions that may not hold in every practical case.

\paragraph{Future work.} One promising direction is to combine regularization with complementary techniques, such as periodic truncation, similar to \citet{DBLP:journals/corr/abs-2004-14566}, or multiple rounds of retraining and compression. Regularization strength schedules, as in \citet{dolatabadi2026numuonnuclearnormconstrainedmuoncompressible}, would likely also yield better results. Moreover, the need for orthogonalization in optimization has sparked research into more effective polar-function approximations, e.g., \citep{GramNewtonSchulz}, which could be applied to SLORR. Finally, we believe our work could motivate broader research beyond the scope of this paper. Although the main application of this work is to directly regularize weights, the versatility of SLORR may open new avenues in settings where one cannot rely on architectural modifications or the validity of cached quantities. For example, low-rank regularization of activations via the nuclear norm has been shown to enhance generalization \citep{shi2023domain}; while that work uses SVDs, an efficient version could potentially be computed using our approach. Similarly, although nontrivial, it would be interesting to extend low-rank regularization to encourage activations to lie in a subspace, enabling better use of activation-aware factorization methods. More broadly, an accessible way to regularize rank may motivate work on broader applications of rank regularization beyond those studied here, including settings where encouraging a higher rank may be beneficial.

\paragraph{Conclusion.} We have introduced SLORR, a framework for directly regularizing weight matrices to be low rank. The idea is simple yet highly versatile: it does not require SVDs, remains efficient, operates directly on the matrix being regularized, and is stateless. Our experiments on ImageNet-1K image classification and LLM pretraining show that this regularizer is effective at producing more compressible models while maintaining training efficiency. More broadly, although we focused on weight compression, we are excited about the avenues that this versatility may open in other use cases.

\acknowledgements

\bibliography{main}
\bibliographystyle{bibstyle}

\newpage
\appendix

\section{Backpropagating Through the Regularizers}
\label{sec:app_backprop_computations}

As noted before, the SLORR regularizers are not differentiable everywhere. In practice, one can informally use ``natural'' backpropagation rules. We note that, for SLORR-Hoyer, the computed rule simply corresponds to backpropagating through the whole expression (and, when backpropagating through the nuclear norm, using the generalized polar factor). These correspond to minimum-norm choices in the generalized gradient of the corresponding functions, in the sense of Clarke \citep[Chapter 2]{clarkecalc}, which we denote by
\[
\partial f(\mW).
\]

Throughout this section, when $\mU$ and $\mV$ appear, they refer to the thin SVD of $\mW$, as in the rest of the paper. We also assume $\mW \neq \bm0$, so that $\|\mW\|_F > 0$.

First, note that both $\|\mW\|_*$ and $\|\mW\|_F$ are locally Lipschitz. Moreover, $\|\mW\|_F^2$ is continuously differentiable, with
\[
\nabla \|\mW\|_F^2 = 2\mW.
\]
By \citep[Proposition 2.2.4]{clarkecalc}, this implies
\[
\partial \|\mW\|_F^2 = \{2\mW\}.
\]

We use Clarke's notion of regularity \citep[Definition 2.3.4]{clarkecalc}. Continuously differentiable functions and convex functions are regular \citep[Proposition 2.3.6]{clarkecalc}. Therefore, $\|\mW\|_F^2$, $\|\mW\|_*$, and $\|\mW\|_*^2$ are regular.

The nuclear norm is convex. Hence, by \citep[Proposition 2.2.7]{clarkecalc}, its Clarke generalized gradient coincides with the subdifferential from convex analysis.

\paragraph{Generalized gradient of SLORR-Nuc.}

The SLORR-Nuc regularizer is the nuclear norm,
\[
\mathcal{L}_{\operatorname{nuc}}(\mW)=\|\mW\|_*.
\]
Let
\[
\mW = \mU \mSigma \mV^\top
\]
be the thin SVD of $\mW$. The generalized gradient is
\[
\partial \|\mW\|_*
=
\left\{
\mU \mV^\top + \mQ
:
\mU^\top \mQ = \bm 0,
\mQ \mV = \bm 0,
\|\mQ\|_2 \leq 1
\right\}.
\]
Since $\mQ$ is Frobenius-orthogonal to $\mU \mV^\top$, the minimum-norm element is obtained by taking $\mQ=\textbf{0}$. Thus, the backpropagation choice is
\[
\nabla \mathcal{L}_{\operatorname{nuc}}(\mW)
=
\mU \mV^\top
\in
\partial \mathcal{L}_{\operatorname{nuc}}(\mW).
\]

\paragraph{Generalized gradient of SLORR-Hoyer.}

The Hoyer-type SLORR regularizer is not convex. It can be written as
\[
\mathcal{L}_{\operatorname{hoyer}}(\mW)
=
\frac{a(\mW)}{b(\mW)},
\]
where
\[
a(\mW)=\|\mW\|_*^2,
\qquad
b(\mW)=\|\mW\|_F^2.
\]
Since $\mW \neq \bm 0$, we have $a(\mW)>0$ and $b(\mW)>0$.

Moreover, since $\|\mW\|_*^2$ is convex (and locally Lipschitz), its generalized gradient corresponds to the convex subdifferential \citep[Proposition 2.2.7]{clarkecalc}:
\[
\partial a(\mW)
=
2\|\mW\|_* \partial \|\mW\|_*.
\]
Also,
\[
\nabla b(\mW)=2\mW.
\]

By the quotient rule for generalized gradients \citep[Proposition 2.3.14]{clarkecalc},
\[
\partial\left[\frac{a(\mW)}{b(\mW)}\right]
=
\frac{
b(\mW)\partial a(\mW)
-
a(\mW)\nabla b(\mW)
}{
b(\mW)^2
}.
\]
Note that equality holds because $a$ and $-b$ are regular, and because $a(\mW)>0$ and $b(\mW)>0$.

Substituting, we obtain
\[
\partial \mathcal{L}_{\operatorname{hoyer}}(\mW)
=
\left\{
2\frac{\|\mW\|_*}{\|\mW\|_F^2}
\left(\mU \mV^\top + \mQ\right)
-
2\frac{\|\mW\|_*^2}{\|\mW\|_F^4}\mW
:
\mU^\top \mQ = \bm 0,\,
\mQ \mV = \bm 0,\,
\|\mQ\|_2 \leq 1
\right\}.
\]
The $\mQ$ term is again orthogonal to the remaining terms. Hence, the minimum-norm element is again obtained by taking $\mQ=\bm 0$. Thus, the backpropagation choice is
\[
\nabla \mathcal{L}_{\operatorname{hoyer}}(\mW)
=
2\frac{\|\mW\|_*}{\|\mW\|_F^2}\mU \mV^\top
-
2\frac{\|\mW\|_*^2}{\|\mW\|_F^4}\mW
\in
\partial \mathcal{L}_{\operatorname{hoyer}}(\mW).
\]
Equivalently,
\[
\nabla \mathcal{L}_{\operatorname{hoyer}}(\mW)
=
2\frac{\|\mW\|_*}{\|\mW\|_F}
\left(
\frac{1}{\|\mW\|_F}\mU \mV^\top
-
\frac{\|\mW\|_*}{\|\mW\|_F^3}\mW
\right).
\]

We then use these expressions when backpropagating through the regularizers. We remark that, intuitively, these correspond to backpropagating individually through the composition of functions in PyTorch, and setting the gradient of the nuclear norm to the generalized polar factor. We use this specific choice because it is the one that is generally (approximately) computed by polynomial approximations such as Polar Express.

\section{Analysis of SLORR-Hoyer}\label{sec:app_slorr_hoyer_theoretical_analysis}

Here, we provide proofs and additional discussion of the results shown in \cref{sec:understanding_slorr_hoyer_main}. We restate the results for the convenience of the reader.

We remark, following \cref{sec:understanding_slorr_hoyer_main}, that our analysis is performed in a regularizer-only gradient descent setting. $\eta>0$ denotes the step size of the regularizer-only update, with the training-time regularization strength absorbed into it. We let $\mW = \mU\operatorname{diag}{(\bm \sigma)}\mV^{\top}$ be a weight matrix and its thin SVD, with singular values $\bm{\sigma}$. Additionally, let $\mW^+$ be the weight matrix after a gradient step from $\mW$, and define $\bm{\sigma}^+$, $\mU^+$, and $\mV^+$ as its thin SVD. We assume ordered singular values in decreasing order ($\bm{\sigma}_1\ge\cdots\ge\bm{\sigma}_r>0$) and $\mW \neq \mathbf{0}$. 

\noGradBehaviorProp*
\begin{proof}
Define $C = 2\|\mW\|_*/\|\mW\|_F$ for brevity. First, write our gradient update as
\begin{align*}
\mW^+
&= \mW
- \eta \nabla\mathcal{L}_{\mathrm{hoyer}}(\mW)\\ &= 
\mW
- \eta\left[
\frac{C}{\|\mW\|_F}\mU\mV^{\top} - \frac{C\|\mW\|_*}{\|\mW\|_F^3}\mW
\right] \\ &= 
\mU\mSigma\mV^{\top}
- \eta \frac{C}{\|\mW\|_F}\mU\mV^{\top} + \eta \frac{C\|\mW\|_*}{\|\mW\|_F^3}\mU\mSigma\mV^{\top}
\\ &=  
\mU \left[ \mSigma - \eta \frac{C}{\|\mW\|_F} \mI + \eta \frac{C\|\mW\|_*}{\|\mW\|_F^3}\mSigma \right] \mV^{\top}.
\end{align*}

Since the matrix in the ``sandwich'' is diagonal (as it is a sum of two diagonal matrices), we can conclude that, for small enough updates such that no inner term becomes negative,
\[
\bm{\sigma}^+_i = \bm{\sigma}_i - \eta \frac{C}{\|\mW\|_F} + \eta \frac{C\|\mW\|_*}{\|\mW\|_F^3}\bm{\sigma}_i.
\]
Moreover, note that the order of singular values is preserved.

Now, 
\begin{align*}
\bm{\sigma}^+_i \diamond\bm{\sigma}_i
&\iff \bm{\sigma}_i - \eta \frac{C}{\|\mW\|_F}
        + \eta \frac{C\|\mW\|_*}{\|\mW\|_F^3}\bm{\sigma}_i \diamond\bm{\sigma}_i \\
&\iff - \eta \frac{C}{\|\mW\|_F}
      + \eta \frac{C\|\mW\|_*}{\|\mW\|_F^3}\bm{\sigma}_i \diamond0 \\
&\iff - \frac{C}{\|\mW\|_F}
      + \frac{C\|\mW\|_*}{\|\mW\|_F^3}\bm{\sigma}_i \diamond0
      \quad (\eta > 0) \\
&\iff - \frac{1}{\|\mW\|_F}
      + \frac{\|\mW\|_*}{\|\mW\|_F^3}\bm{\sigma}_i \diamond0
      \quad (C>0) \\
&\iff -\|\mW\|_F^2 + \|\mW\|_* \bm{\sigma}_i \diamond0 \\
&\iff \|\mW\|_* \bm{\sigma}_i \diamond\|\mW\|_F^2 \\
&\iff \bm{\sigma}_i \diamond\frac{\|\mW\|_F^2}{\|\mW\|_*},
\end{align*}
for $\diamond \in \{<, >, =\}$.

This concludes the proof.

\end{proof}

\energyRedistributionProp*
\begin{proof}
For brevity, define

\[
\varepsilon
=
\eta\frac{2\|\mW\|_*^2}{\|\mW\|_F^4}.
\]

Note that (see the proof of \cref{prop:shrinkandgrouth})
\[
\bm\sigma_i^+
=
\bm\sigma_i
-\eta\frac{2\|\mW\|_*}{\|\mW\|_F^2}
+
\eta\frac{2\|\mW\|_*^2}{\|\mW\|_F^4}\bm\sigma_i.
\]

Then

\[
\eta\frac{2\|\mW\|_*}{\|\mW\|_F^2}
=
\eta\frac{2\|\mW\|_*^2}{\|\mW\|_F^4}
\cdot
\frac{\|\mW\|_F^2}{\|\mW\|_*}
=
\varepsilon\tau.
\]

Therefore,

\[
\begin{aligned}
\bm\sigma_i^+
&=
\bm\sigma_i
-\eta\frac{2\|\mW\|_*}{\|\mW\|_F^2}
+
\eta\frac{2\|\mW\|_*^2}{\|\mW\|_F^4}\bm\sigma_i \\
&=
\bm\sigma_i-\varepsilon\tau+\varepsilon\bm\sigma_i \\
&=
\bm\sigma_i+\varepsilon(\bm\sigma_i-\tau).
\end{aligned}
\]

For convenience, now define

\[
\begin{aligned}
A^+
&=
\sum_{i\le k}
\left(
\bm\sigma_i+\varepsilon(\bm\sigma_i-\tau)
\right)^2
=
\sum_{i\le k}
\left[
\bm\sigma_i^2
+
2\varepsilon\bm\sigma_i(\bm\sigma_i-\tau)
+
\varepsilon^2(\bm\sigma_i-\tau)^2
\right]
\\
&=
\underbrace{\sum_{i\le k}\bm\sigma_i^2}_{A}
+
\sum_{i\le k}
\left[
2\varepsilon\bm\sigma_i(\bm\sigma_i-\tau)
+
\varepsilon^2(\bm\sigma_i-\tau)^2
\right]
=
A
+
\sum_{i\le k}
\left[
2\varepsilon\bm\sigma_i(\bm\sigma_i-\tau)
+
\varepsilon^2(\bm\sigma_i-\tau)^2
\right],
\end{aligned}
\]

\[
\begin{aligned}
B^+
&=
\sum_{i>k}
\left(
\bm\sigma_i+\varepsilon(\bm\sigma_i-\tau)
\right)^2
=
\sum_{i>k}
\left[
\bm\sigma_i^2
+
2\varepsilon\bm\sigma_i(\bm\sigma_i-\tau)
+
\varepsilon^2(\bm\sigma_i-\tau)^2
\right]
\\
&=
\underbrace{\sum_{i>k}\bm\sigma_i^2}_{B}
+
\sum_{i>k}
\left[
2\varepsilon\bm\sigma_i(\bm\sigma_i-\tau)
+
\varepsilon^2(\bm\sigma_i-\tau)^2
\right]
=
B
+
\sum_{i>k}
\left[
2\varepsilon\bm\sigma_i(\bm\sigma_i-\tau)
+
\varepsilon^2(\bm\sigma_i-\tau)^2
\right].
\end{aligned}
\]

\paragraph{The $A$ term. } Note that for the $A$ term,
\[
(\bm\sigma_i-\tau) > 0,
\]
as required in the statement of this result. Hence, $A^+ - A = \varepsilon a_1 + 
\varepsilon^2 a_2$ (as every other term is nonnegative), where $a_1, a_2 > 0$.

\paragraph{The $A + B$ term. } We will show that, for sufficiently small $\varepsilon$, the $A + B$ increases less than $A$. First note
\[
\sum_i 2\varepsilon \bm\sigma_i(\bm\sigma_i-\tau)
=
2\varepsilon\left(
\sum_i \bm\sigma_i^2
-
\tau\sum_i \bm\sigma_i
\right)
=
2\varepsilon\left(
\sum_i \bm\sigma_i^2
-
\frac{\|\mW\|_F^2}{\|\mW\|_*}\sum_i \bm\sigma_i
\right)
=
0.
\]

Now,
\begin{align*}
A^+ + B^+ - (A+B)
&=
\sum_i
\left[
2\varepsilon\bm\sigma_i(\bm\sigma_i-\tau)
+
\varepsilon^2(\bm\sigma_i-\tau)^2
\right] =
\sum_i
\left[
\varepsilon^2(\bm\sigma_i-\tau)^2
\right] = b\varepsilon^2,
\end{align*}
with  $b > 0$.

As the $A$ term has a strictly positive linear term in $\varepsilon$, while the
total increment is only quadratic, taking $\varepsilon$ small enough gives
\[
A^+ - A > (A^+ + B^+) - (A+B).
\]
Also note that $A^+-A>0$. Hence,
\[
\frac{A^+}{A^+ + B^+} > \frac{A}{A+B}.
\]

\paragraph{Final step. } Now, separate the sums
\[
\begin{aligned}
\frac{\sum_{i\le k} (\bm{\sigma}_i^+)^2}{\sum_i (\bm{\sigma}_i^+)^2}
&=
\frac{\sum_{i\le k} (\bm{\sigma}_i^+)^2}
{\sum_{i\le k} (\bm{\sigma}_i^+)^2 + \sum_{i > k} (\bm{\sigma}_i^+)^2}
=
\frac{A^+}{A^+ + B^+}, \\[0.5em]
\frac{\sum_{i\le k} \bm{\sigma}_i^2}{\sum_i \bm{\sigma}_i^2}
&=
\frac{\sum_{i\le k} \bm{\sigma}_i^2}
{\sum_{i\le k} \bm{\sigma}_i^2 + \sum_{i > k} \bm{\sigma}_i^2}
=
\frac{A}{A + B}.
\end{aligned}
\]

Finally, it follows that
\[
1-\frac{\|\mU^+_{:,:k}\operatorname{diag}(\bm{\sigma}_{:k}^+)\mV_{:,:k}^{+\top}-\mW^+\|_F^2}{\|\mW^+\|_F^2}
=
\frac{\sum_{i\le k} (\bm{\sigma}_i^+)^2}{\sum_i (\bm{\sigma}_i^+)^2}
>
\frac{\sum_{i\le k} \bm{\sigma}_i^2}{\sum_i \bm{\sigma}_i^2}
=
1-\frac{\|\mU_{:,:k}\operatorname{diag}(\bm{\sigma}_{:k})\mV_{:,:k}^\top-\mW\|_F^2}{\|\mW\|_F^2}.
\]

The result follows from subtracting 1 and swapping sides.
\end{proof}

\section{Approximation Error Guarantees}\label{sec:app_approx_guarantees}
We restate the result for the convenience of the reader.
\SLORRApproxGuarantees*

\begin{proof}
Given $\mW$, by assumption, a normalized version $\mM$ satisfies $\bm\sigma(\mM)\subseteq[\ell,1]$ (where only the nonzero singular values are taken into account). By Theorem 3.3 of \citet{amsel2026the}, 
\[
\left\|\widehat{\mP}- \mU\mV^\top\right\|_2 \leq \left|1-\ell^2\right|^{(q+1)^T},
\]
where $T$ is the number of Polar Express steps (6 in our main experiments) and $d = 2q + 1$ is the degree of the Polar Express polynomial (5 in our experiments). We can now extend it to the approximation of the nuclear norm. In particular, fix the PE parameters and let $\delta = \left|1-\ell^2\right|^{(q+1)^T}$.

First, we have
\begin{align*}
\left|\hat n-n\right|
&=
\left|
\operatorname{tr}\!\left(\mW^\top(\widehat{\mP}-\mU\mV^\top)\right)
\right| \\
&\leq
\left\|\mW\right\|_*\,\left\|\widehat{\mP}-\mU\mV^\top\right\|_2 \\
&\leq
\delta \left\|\mW\right\|_*.
\end{align*}
Therefore,
\[
\frac{\left|\hat n-n\right|}{n}\leq \delta.
\]

Similarly, we can bound the error in the SLORR-Hoyer value. Using
\[
\left|\hat n\right|\leq n+\left|\hat n-n\right|\leq (1+\delta)n,
\]
we obtain
\begin{align*}
\left|\hat n^2-n^2\right|
&=
\left|\hat n-n\right|\,\left|\hat n+n\right| \\
&\leq
\left|\hat n-n\right|\,\left(\left|\hat n\right|+n\right) \\
&\leq
\delta n\left((1+\delta)n+n\right) \\
&=
n^2\delta(2+\delta).
\end{align*}
Therefore, we get (dividing by $n^2$)
\[
\frac{\left|\hat n^2-n^2\right|}{n^2}
\leq
\delta(2+\delta),
\]
or equivalently,
\[
\frac{\left|\hat n^2/f^2-n^2/f^2\right|}{n^2/f^2}
\leq
\delta(2+\delta).
\]

Now, we can bound the error on the SLORR-Hoyer gradient. Recall that (for the exact polar factor):
\[
\mR
=
2\frac{n}{f}
\left(
\frac{1}{f}\mU\mV^\top
-
\frac{n}{f^3}\mW
\right),
\]
and note that its Polar Express approximation is 
\[
\widehat{\mR}
=
2\frac{\hat n}{f}
\left(
\frac{1}{f}\widehat{\mP}
-
\frac{\hat n}{f^3}\mW
\right).
\]

We have
\begin{align*}
\left\|\mR - \widehat \mR\right\|_2 &= \left\| 2\frac{n}{f}\left(\frac{1}{f}\mU\mV^{\top} - \frac{n}{f^3}\mW\right) - 2\frac{\hat n}{f}\left(\frac{1}{f}\widehat{\mP} - \frac{\hat n}{f^3}\mW\right) \right\|_2 \\ &= \left\| \frac{2n}{f^2}\mU\mV^{\top} - \frac{2n^2}{f^4}\mW - \frac{2\hat n}{f^2}\widehat{\mP} + \frac{2\hat n^2}{f^4}\mW \right\|_2 \\ &\leq \left\| \frac{2n}{f^2}\mU\mV^{\top} - \frac{2 \hat n}{f^2}\mU\mV^{\top} \right\|_2 +  \left\| \frac{2 \hat n}{f^2}\mU\mV^{\top} - \frac{2\hat n}{f^2}\widehat{\mP} \right\|_2 + \left\|- \frac{2n^2}{f^4}\mW  + \frac{2\hat n^2}{f^4}\mW \right\|_2 
\\ &\leq \delta \frac{2n}{f^2} + \frac{2 \left|\hat n\right|}{f^2} \delta + \frac{2\delta(2+\delta)n^2}{f^4}\left\|\mW\right\|_2 \leq \frac{2\delta (n + \left|\hat n\right|)}{f^2}
+ \frac{2\delta(2+\delta)n^2}{f^4}\left\|\mW\right\|_2.
\end{align*}

Now, to get rid of $\hat n$, note that
\[
\left|\hat n\right| \leq n + \left|\hat n - n\right| \leq (1+\delta)n.
\]
Hence,
\[
\begin{aligned}
\left\|\mR - \widehat \mR\right\|_2
&\leq
\frac{2\delta (n + (1+\delta)n)}{f^2}
+
\frac{2\delta(2+\delta)n^2}{f^4}\left\|\mW\right\|_2 \\
&=
\frac{2\delta(2+\delta)n}{f^2}
+
\frac{2\delta(2+\delta)n^2}{f^4}\left\|\mW\right\|_2 \\
&=
2\delta(2+\delta)
\left(
\frac{n}{f^2}
+
\frac{n^2}{f^4}\left\|\mW\right\|_2
\right).
\end{aligned}
\]
In particular, as $T \to \infty$, the error tends to 0, as expected.
\end{proof}

\paragraph{Remark about zero singular values.} The Polar Express guarantee we use only concerns the nonzero singular subspace, corresponding to the generalized polar factor (i.e., using the thin SVD) used throughout this paper. Exact zero singular values are therefore excluded from the result. In exact arithmetic, the Polar Express updates do not create components in this zero singular subspace, since the scalar polynomial maps used in their iteration satisfy $p(0)=0$. Thus, zero singular values remain zero under the idealized iteration, and hence our approximation guarantee concerns only the nonzero singular values.

\paragraph{Remark about normalization and the lower bound parameter.} The singular value range $[\ell,1]$ is inherited from the analysis in \citet{amsel2026the}. For a nonzero matrix $\mW$, we apply Polar Express to the normalized matrix
\[
\widetilde{\mW}=\frac{\mW}{\|\mW\|_F+\varepsilon}.
\]
This positive rescaling leaves the polar factor unchanged and ensures that the largest singular value of $\widetilde{\mW}$ is at most one. However, it does not by itself guarantee that every nonzero singular value is at least $\ell$. The parameter $\ell$ should therefore be interpreted as the lower bound parameter used by Polar Express. If the actual smallest nonzero normalized singular value is below $\ell$, the displayed worst-case bound does not apply literally to that singular direction. We note that \citet{amsel2026the} mention that inaccurate lower bound guesses are typically not severe in practice.

\paragraph{Remark about practical considerations.} The guarantee above is a worst-case bound for the idealized Polar Express iteration and should not be read as a tight estimate of the implemented approximation error. In our experiments, we use a small number of Polar Express iterations together with the practical adaptations of \citet{amsel2026the}. These choices slightly deviate from the idealized theorem, and floating point arithmetic can introduce additional error. Empirically, this setting works well in our experiments; see \cref{sec:app_polar_express_details} for an ablation and implementation details.

\clearpage

\section{General Experimental and Implementation Details}\label{sec:app_impl_details_general}

The code used to produce our experimental results is available at \url{\githuburl}. It includes the implementation of our method and baselines, as well as the scripts, hyperparameters, and configurations used for our experiments. Our code for both the vision and LLM experiments is based on the PyTorch framework \citep{paszke2019pytorchimperativestylehighperformance}.

\subsection{On our Use of Polar Express}\label{sec:app_polar_express_details}

Polar Express has several tunable knobs. The first is the lower bound on singular values, $\ell$, along with the degree and number of iterations. Following Implementation 1 in \citet{amsel2026the}, we set $\ell = 10^{-3}$ and the degree to $d=5$. They note that choosing $\ell$ too small can delay convergence for smaller singular values, but that inaccurate guesses are typically not severe in practice. We found these values to work well in our experiments. We use \texttt{bfloat16}, as recommended.

An important hyperparameter of Polar Express is the number of iterations. The authors recommend 5 or 6 iterations. Throughout our experiments, we used 6 Polar Express iterations, which add minimal overhead. In early experiments, we found that this quantity worked well. Here, we explicitly ablate the number of iterations. Following our ViT-B/16 experiments, we fix $\lambda=0.002$ and repeat our uniform compressibility experiments using runs with $\{1, 2, 4, 6, 8, 10\}$ Polar Express iterations. Results are presented in \cref{tab:polarexpress_iters_ablation_vitb16}. In this setting, after 6 iterations, results across compression ratios change little, suggesting diminishing returns from additional Polar Express iterations in this setting.

\begin{table}[h]
\centering
\caption{\textbf{Effect of Polar Express iterations (ViT-B/16).} We repeat the main ViT-B/16 SLORR-Hoyer experiment with $\lambda=0.002$ while varying only the number of Polar Express iterations. Columns denote retained parameter ratios, and entries are ImageNet-1K top-1 accuracy.}
\label{tab:polarexpress_iters_ablation_vitb16}
\small
\setlength{\tabcolsep}{4pt}
\begin{tabular}{ccccccccccc}
\toprule
\textbf{Steps} & \textbf{21\%} & \textbf{31\%} & \textbf{41\%} & \textbf{51\%} & \textbf{60\%} & \textbf{70\%} & \textbf{80\%} & \textbf{90\%} & \multicolumn{1}{|c}{\textbf{Full}} \\
\midrule
1 & 0.1 & 0.1 & 0.6 & 14.4 & 49.0 & 67.9 & 77.1 & 81.0 & \multicolumn{1}{|c}{84.9} \\
2 & 0.1 & 0.1 & 0.2 & 0.3 & 1.6 & 32.3 & 63.7 & 72.6 & \multicolumn{1}{|c}{83.4} \\
4 & 0.1 & 0.1 & 0.1 & 0.2 & 3.8 & 18.6 & 82.6 & 83.3 & \multicolumn{1}{|c}{83.7} \\
6 & 0.1 & 0.1 & 6.6 & 66.2 & 79.5 & 81.1 & 81.5 & 81.7 & \multicolumn{1}{|c}{81.7} \\
8 & 0.1 & 0.1 & 5.2 & 67.7 & 79.5 & 80.9 & 81.3 & 81.5 & \multicolumn{1}{|c}{81.6} \\
10 & 0.1 & 0.1 & 7.0 & 68.9 & 79.4 & 80.9 & 81.2 & 81.4 & \multicolumn{1}{|c}{81.5} \\
\bottomrule
\end{tabular}
\end{table}

\subsection{Implementation Details for the Decoupled Variant}\label{sec:app_decoupled_impl_details}
\cref{alg:decoupled-SLORR} details how the gradient of the SLORR regularizer can be applied in a decoupled form. This follows the rationale of AdamW \citep{loshchilov2018decoupled}, where weight decay is applied based on the pre-update weights. SLORR-Hoyer-D follows this approach, but note that this approach can also be applied to SLORR-Nuc. Note that, in practice, the regularizer value (not only the gradient) can also be computed (at the same time as the gradient) and, for example, logged.

\begin{algorithm}[H]
\caption{Decoupled SLORR update}
\label{alg:decoupled-SLORR}
\begin{algorithmic}[1]
\Require Weight set $\mathcal{W}$, regularized weight set $\mathcal{W}_{\mathrm{reg}} \subseteq \mathcal{W}$, learning rate $\eta$, regularization strength $\lambda$
\Require Task gradients $\nabla_{\mathcal{W}} \mathcal{L}_{\mathrm{task}}$
\ForAll{$\mW \in \mathcal{W} \setminus \mathcal{W}_{\mathrm{reg}}$}
    \If{$\nabla_{\mW} \mathcal{L}_{\mathrm{task}}$ exists}
        \State Apply AdamW update to $\mW$ using $\nabla_{\mW} \mathcal{L}_{\mathrm{task}}$
    \EndIf
\EndFor
\ForAll{$\mW \in \mathcal{W}_{\mathrm{reg}}$}
    \State $\widehat{\mG}_{\mathrm{SLORR}}^{\mathrm{pre}} \gets \widehat{\mG}_{\mathrm{SLORR}}(\mW)$
    \If{$\nabla_{\mW} \mathcal{L}_{\mathrm{task}}$ exists}
        \State Apply AdamW update to $\mW$ using $\nabla_{\mW} \mathcal{L}_{\mathrm{task}}$
    \EndIf
    \State $\mW \gets \mW - \eta \lambda \widehat{\mG}_{\mathrm{SLORR}}^{\mathrm{pre}}$ \Comment{uses the pre-update SLORR gradient}
\EndFor
\end{algorithmic}
\end{algorithm}

\subsection{Implementation Details for the Memory-Efficient Variant}\label{sec:app_memory_efficient_approach}
\Cref{alg:iterative-reg} details a more memory-efficient approach for PyTorch environments. When implemented as a loss function (i.e., not in decoupled form, as detailed in \cref{sec:app_decoupled_impl_details}), using an autograd function is easy to implement. However, if one computes each layer-wise loss, sums them, and then backpropagates, PyTorch keeps intermediate values for all layers. This can consume additional memory proportional to the total size of the regularized layers. Again, this can be applied to both SLORR-Nuc and SLORR-Hoyer. Similar to \cref{sec:app_decoupled_impl_details}, the regularizer value can also be computed at the same time, if needed.

\begin{algorithm}[H]
\caption{Iterative memory-efficient SLORR implementation}
\label{alg:iterative-reg}
\begin{algorithmic}[1]
\Require Regularized weight set $\mathcal{W}_{\mathrm{reg}}$, regularization strength $\lambda$
\ForAll{$\mW \in \mathcal{W}_{\mathrm{reg}}$}
    \State Compute $\widehat{\mG}_{\mathrm{SLORR}}(\mW)$ from the current value of $\mW$
    \If{$\mW.\mathrm{grad}$ exists}
        \State $\mW.\mathrm{grad} \gets \mW.\mathrm{grad} + \lambda \widehat{\mG}_{\mathrm{SLORR}}(\mW)$
    \Else
        \State $\mW.\mathrm{grad} \gets \lambda \widehat{\mG}_{\mathrm{SLORR}}(\mW)$
    \EndIf
\EndFor
\end{algorithmic}
\end{algorithm}

\clearpage
\section{Vision Experimental and Implementation Details}\label{sec:app_vision_experiment_details}

For the vision experiments, we use standard ImageNet training pipelines built on the model implementations provided by \texttt{timm} \citep{Wightman_PyTorch_Image_Models}.

\subsection{Interpreting Convolutional Layers}\label{sec:app_interpreting_conv_layers}

For factorization and regularization, we interpret each convolutional kernel as a matrix by reshaping
\[
\tW \in \mathbb{R}^{C_o \times C_i \times H_k \times W_k}
\]
as $\operatorname{reshape}(\tW; (C_o,-1))$. A low-rank factorization under this interpretation corresponds to replacing the original convolution with two sequential convolutions. This is the standard channel-wise decomposition used in prior low-rank factorization work, including LoRITa.

\subsection{Implementation Details for Baselines}\label{sec:app_baseline_impl_details}

\subsubsection{LoRITa}

We implement LoRITa largely as described in \citet{alkhouri2024structurepreserving}, with a few adaptations. We fix $N=2$, where $N$ denotes the number of factorized layers used to replace each original layer. Since we initialize from pretrained checkpoints, we choose the factors so that the initial network is functionally equivalent to the original model. Specifically, for a weight matrix $\mW$, we compute its thin SVD and initialize
\[
\mW_0 = \mU \mSigma^{1/2}, \qquad
\mW_1 = \mSigma^{1/2} \mV^{\top}.
\]
This corresponds to the ``spectral initialization'' of \citet{khodak2021initialization} in the case where all singular values are preserved.

For ResNet-18 pretraining, we first initialize each corresponding dense layer using the same random initializer as in the dense baseline. We then compute its factorization and initialize the LoRITa factors using the spectral initialization above. This matches the initial function of the dense baseline while using LoRITa's factorized parameterization, allowing us to control for initialization effects.

Following \citet{ghosh2025qr} and \citet{alkhouri2024structurepreserving}, we implement the LoRITa regularizer through the optimizer's weight decay term. The value reported for LoRITa in the result tables is this coefficient. As we use AdamW, it is applied in a decoupled way. Parameters that are not replaced by LoRITa factors use the default weight decay for the corresponding experiment, while normalization parameters and biases are not decayed. This keeps the LoRITa penalty separate from the ordinary optimizer weight decay used for the rest of the model.

For convolutional layers, we follow the construction described in \citet{alkhouri2024structurepreserving}, which coincides with the general interpretation of convolutional layers discussed in \cref{sec:app_interpreting_conv_layers}.

We make one implementation modification to improve memory efficiency, which allows us to run LoRITa on larger models. In the original sequential implementation\footnote{\url{https://github.com/XitongSystem/LoRITa}}, LoRITa replaces a layer by a sequence of factorized layers. With $N=2$, this introduces an additional intermediate activation and can substantially increase memory use at realistic batch sizes. Even when memory is sufficient, the sequential implementation can remain computationally expensive; see the LoRITa authors' comments at \url{https://openreview.net/forum?id=1KCrVMJoJ9}.

We describe the modification for a linear layer without bias; the convolutional case follows the same matrix interpretation as in \cref{sec:app_interpreting_conv_layers}. Let the original layer be parameterized by $\mW\in\mathbb{R}^{I\times O}$, and let $R=\min\{I,O\}$. LoRITa replaces this layer with two trainable matrices,
\[
\mW_0 \in \mathbb{R}^{I \times R}, \qquad
\mW_1 \in \mathbb{R}^{R \times O}.
\]
For an input $\mX\in\mathbb{R}^{B\times I}$, with batch and token dimensions folded into a single effective batch dimension, the original computation is
\[
\mY=\mX\mW,
\]
while the sequential LoRITa implementation computes
\[
\mY=(\mX\mW_0)\mW_1.
\]
Instead, in each forward pass we first merge the factors and then apply the merged matrix:
\[
\mY=\mX(\mW_0\mW_1).
\]
This avoids storing the intermediate activation $\mX\mW_0$ and reduces LoRITa's overhead to effectively zero in our experiments. The merged computation defines the same function of $\mW_0$ and $\mW_1$ as the sequential implementation, so the forward pass and gradients are equivalent up to floating point differences.

Moreover, the original LoRITa implementation uses $R=I$. We use $R=\min\{I,O\}$, which is the natural full-rank dimension for the matrix factorization above and enables the checkpoint-preserving initialization described here. The main theoretical results supporting LoRITa still apply (see Section 4.3 of \citet{alkhouri2024structurepreserving}).

We restrict LoRITa to $N=2$. Extending the initialization to $N>2$ would require choosing how to split the singular values across more than two factors, introducing an additional experimental choice. Since the LoRITa paper reports only marginal ImageNet ResNet gains from $N=3$, including for ResNet-18, we keep $N=2$ throughout and leave larger $N$ as a limitation of our study.

\subsubsection{Q3R}

Following the authors' recommended setup, we use their AdamQ3R optimizer. AdamQ3R combines Adam-style optimization with the Q3R regularizer, applying the regularization effect in a decoupled form motivated by AdamW \citep{loshchilov2018decoupled}. We copied the optimizer and regularizer implementation from the authors' official codebase\footnote{\url{https://github.com/ThatE10/Q3R}} and left the linear-layer path unchanged, except that we added support for decoupled weight decay to match our AdamW baselines.

We use a refresh interval of 5 for all main experiments, unless noted otherwise, following the setting used by the Q3R authors in their main experiments. We also found that larger, cheaper refresh intervals can work well in some settings; see \cref{sec:app_q3r_refresh_period_ablation}.

The original implementation does not support convolutional layers. We extend it to convolutions by interpreting convolutional layers as matrices, as described in \cref{sec:app_interpreting_conv_layers}. This extension only inserts the necessary reshape operations and leaves other parts of the code untouched.

When we refer to Q3R's ``target rank'', we mean the value passed to the AdamQ3R optimizer. This value can be a ratio between 0 and 1 rather than an absolute rank.

\subsection{Training Details and Hyperparameters}\label{sec:training_details_app}

We first describe the general training setup and then list specific hyperparameters for each model.

Regularization and factorization are applied to layers in their native parameterization. In particular, the \texttt{timm} ViT implementations use fused QKV layers, and we regularize and factorize these fused matrices without splitting them. This follows LoRITa's ViT implementation. Q3R's experiments split QKV layers, but we note this is an experimental choice rather than a mechanism inherent to the method (we remark that, in any case, this choice must be consistent across all methods). Hence, for simplicity, we use the default fused QKV implementation. We did not explore split QKV layers.

We keep non-regularizer hyperparameters fixed across methods whenever possible. For LoRITa, the factor weight decay is part of the method and is varied as its regularization strength. For Q3R, our AdamQ3R implementation supports the same decoupled weight decay used by the AdamW baselines and SLORR runs. Across methods, ordinary weight decay is applied only to weight matrices; normalization parameters and biases are not decayed.

For continued training experiments, we choose learning rates so that an unregularized baseline largely preserves the accuracy of the pretrained checkpoint. On ViT-B/16, Q3R tended to perform poorly with the smaller learning rate, so we tested a larger one. LoRITa also improved slightly in that setting, so we include both learning rates. We did not explore this higher-learning-rate setting for other models or for ResNet-18 pretraining (for ResNet-18, the learning rate was again chosen based on unregularized performance).

We remark that even though our sweeps are significantly more extensive than those of prior work, differences in training configurations might benefit different techniques. Accordingly, our comparisons should be interpreted as a broad best-effort study under mostly matched training recipes rather than as an exhaustive sweep, as mentioned in \cref{sec:discussion}.

All runs use \texttt{bfloat16} automatic mixed precision (weights are stored in \texttt{float32}).

\paragraph{Checkpoints. } For runs starting from checkpoints, we use the default ones on \texttt{timm}, except for ResNet-50 for which we use the TorchVision \citep{TorchVision} checkpoint, which is the most common one.

\paragraph{Hyperparameters for ViT-B/16. } For ViT-B/16, we use learning rates of $1\times 10^{-5}$ and $5 \times 10^{-5}$, a batch size of 1024, a weight decay of 0.1, and a cosine schedule with 1 epoch of linear warmup \citep{goyal2018accuratelargeminibatchsgd,loshchilov2017sgdr} with an initial factor of $10^{-3}$, followed by decay to 0.1 times the original learning rate. We use standard ViT-B/16 augmentation through the \texttt{timm} library. We detail the augmentation parameters in \cref{tab:augs}.

\paragraph{Hyperparameters for ViT-L/16. } We use the same setup as for ViT-B/16, but with a batch size of 896. With a batch size of 1024, LoRITa produced out-of-memory errors, even with our optimizations. We therefore use a batch size of 896 for all methods in this setting to ensure comparable conditions. We only used a learning rate of $1\times 10^{-5}$.

\paragraph{Hyperparameters for ResNet-50.} For ResNet-50, we use a learning rate of $5\times 10^{-5}$. We use the same scheduler as in the ViT experiments, a batch size of 1024, and a weight decay of 0.01. We apply standard data augmentation, consisting of random resizing and cropping, followed by horizontal flipping with probability 0.5, and normalization using the standard ImageNet-1K mean and standard deviation.

\paragraph{Hyperparameters for ResNet-18 (training from scratch).} We use a learning rate of $10^{-3}$ and train for 110 epochs. The rest is the same as for ResNet-50.

All vision experiments use A100 GPUs, except ViT-L/16 ones, where we use B200 GPUs.

\paragraph{Method hyperparameter selection.}
Method-specific hyperparameters were selected through exploratory runs and manual coarse-to-fine tuning. SLORR and LoRITa each have a single main regularization coefficient, whereas Q3R has two interacting hyperparameters, excluding the refresh interval, making exhaustive grid search impractical. For Q3R, we found that values larger than the authors' recommended range of $[0.001,0.01]$ worked substantially better in our settings.

For all methods, stronger regularization generally improved compressibility but could reduce uncompressed accuracy or cause training collapse. Conversely, too little regularization had little effect on compressibility. Runs were discarded only when they collapsed or produced unusable uncompressed accuracy, not based on post-compression performance. All included configurations are listed in \cref{sec:app_additional_vision_results}.

\subsection{Factorization Details}

Our factorization implementation is adapted from the BALF \citep{gonzalezmartinez2025balfbudgetedactivationawarelowrank} repository. For every model, we factorize all layers except the first and last, typically the classifier. As mentioned in the main text, we perform experiments with both uniform truncation and energy truncation. We describe both below; our implementation is included in our code.

\paragraph{Uniform truncation.}
Given a target retention ratio $c\in(0,1]$, with $c=1$ denoting no compression, each eligible layer is truncated to approximately retain a fraction $c$ of its original parameters. Under our matrix interpretation, the same rank choice also gives approximately the same retained FLOPs fraction for both linear and convolutional layers. Thus, uniform truncation applies the same compression ratio across eligible layers, although different layers might still retain different absolute numbers of parameters and FLOPs.

\paragraph{Energy truncation.} Given a single matrix, the retained energy metric when keeping up to rank $P$ is defined as
\begin{equation*}\label{eq:sval_energy}
E(P) = \frac{\sum_{i=1}^P\bm{\sigma}_i^2}{\sum_{i=1}^R\bm{\sigma}_i^2} = 1 - \frac{\|\mW - \widehat{\mW}\|_F^2}{\|\mW\|_F^2},
\end{equation*}
where $R = \min \{M, N\}$, $M$ and $N$ are the dimensions of $\mW$, and $\widehat{\mW}$ denotes the reconstruction of the compressed matrix. Similar formulations are used in \citet{10.5555/3294771.3294853,Yang2020DeepHoyer,liebenwein2021compressing}. 

Given a fixed retained energy threshold $\rho$, the energy criterion selects a per-layer rank
\begin{equation*}
P^* = \min \{P : E(P) \geq \rho\},
\end{equation*}
i.e., the lowest rank that maintains a relative squared error of at most $1 - \rho$. In particular, for each model, we sweep
$\{0.7,\allowbreak0.75,\allowbreak0.8,\allowbreak0.85,\allowbreak0.9,\allowbreak
0.925,\allowbreak0.95,\allowbreak0.975,\allowbreak0.99,\allowbreak
0.995,\allowbreak0.999,\allowbreak0.9995,\allowbreak0.9999,\allowbreak
0.99995,\allowbreak0.99996,\allowbreak0.99997,\allowbreak0.99998,\allowbreak
0.99999\}$.
    
\begin{table}[t]
\caption{Training augmentation for the ViT-B runs.}
\label{tab:augs}
\centering
\small
\setlength{\tabcolsep}{8pt}
\renewcommand{\arraystretch}{1.15}
\begin{tabular}{lll}
\toprule
\textbf{Component} & \textbf{Parameter} & \textbf{Value} \\
\midrule
\multirow{10}{*}{\texttt{create\_transform}} 
 & \texttt{input\_size}      & \texttt{(3, 224, 224)} \\
 & \texttt{is\_training}     & \texttt{True} \\
 & \texttt{color\_jitter}    & \texttt{0.3} \\
 & \texttt{auto\_augment}    & \texttt{"rand-m9-mstd0.5-inc1"} \\
 & \texttt{interpolation}    & \texttt{"bicubic"} \\
 & \texttt{re\_prob}         & \texttt{0.25} \\
 & \texttt{re\_mode}         & \texttt{"pixel"} \\
 & \texttt{re\_count}        & \texttt{1} \\
 & \texttt{mean}             & \texttt{(0.5, 0.5, 0.5)} \\
 & \texttt{std}              & \texttt{(0.5, 0.5, 0.5)} \\
\midrule
\multirow{7}{*}{\texttt{Mixup}}
 & \texttt{mixup\_alpha}     & \texttt{0.8} \\
 & \texttt{cutmix\_alpha}    & \texttt{1.0} \\
 & \texttt{prob}             & \texttt{1.0} \\
 & \texttt{switch\_prob}     & \texttt{0.5} \\
 & \texttt{mode}             & \texttt{"batch"} \\
 & \texttt{label\_smoothing} & \texttt{0.1} \\
 & \texttt{num\_classes}     & \texttt{1000} \\
\bottomrule
\end{tabular}
\end{table}

\section{Q3R Refresh Period Ablation}\label{sec:app_q3r_refresh_period_ablation}
Throughout our experiments, we use a refresh period of 5 for Q3R, following their main experiments. Sweeping the period in addition to the target rank and regularization strength would have been impractical, so we use their recommended configuration.

However, for completeness, we also swept the refresh interval over a small number of runs to probe whether one could use a larger period, which is more efficient, and what the corresponding accuracy tradeoff would be. We limit ourselves to ViT-B/16.

Results are presented in \cref{tab:q3r_steps_results}. Q3R appears to be surprisingly robust in the cases we tested. Note, however, that this might differ in other models, scales, or training horizons.

\begin{table}[h]
\caption{\textbf{Q3R refresh interval ablation on ViT-B/16.} ImageNet-1K top-1 accuracy under uniform retained parameter ratios and different configurations.}
\small
\label{tab:q3r_steps_results}
\centering
\setlength{\tabcolsep}{5.5pt}
\begin{tabular}{@{}lrrrrrrrrrr@{}}
\toprule
\textbf{Setting} & \textbf{Interval} & \textbf{21\%} & \textbf{31\%} & \textbf{41\%} & \textbf{51\%} & \textbf{60\%} & \textbf{70\%} & \textbf{80\%} & \textbf{90\%} & \multicolumn{1}{|c}{\textbf{Full}} \\
\midrule
\multirow{5}{*}{\begin{tabular}[c]{@{}c@{}}$\lambda=8.0$ \\ target rank 0.05 \\ learning rate $5\times10^{-5}$\end{tabular}} & 5 & 29.6 & 75.2 & 76.0 & 76.1 & 76.2 & 76.3 & 76.3 & 76.3 & \multicolumn{1}{|c}{76.3} \\
 & 20 & 34.8 & 75.2 & 75.8 & 76.1 & 76.3 & 76.4 & 76.4 & 76.4 & \multicolumn{1}{|c}{76.4} \\
 & 50 & 33.3 & 75.1 & 75.8 & 76.1 & 76.3 & 76.3 & 76.3 & 76.4 & \multicolumn{1}{|c}{76.4} \\
 & 100 & 32.9 & 75.0 & 75.8 & 76.0 & 76.2 & 76.2 & 76.3 & 76.3 & \multicolumn{1}{|c}{76.3} \\
 & 300 & 44.4 & 74.8 & 75.5 & 75.8 & 75.9 & 76.0 & 76.0 & 76.0 & \multicolumn{1}{|c}{76.0} \\
\midrule
\multirow{5}{*}{\begin{tabular}[c]{@{}c@{}}$\lambda=3.0$ \\ target rank 0.15 \\ learning rate $5\times10^{-5}$\end{tabular}} & 5 & 0.1 & 0.2 & 21.2 & 71.9 & 80.7 & 82.5 & 82.7 & 82.8 & \multicolumn{1}{|c}{82.8} \\
 & 20 & 0.1 & 0.2 & 19.3 & 72.0 & 80.7 & 82.5 & 82.7 & 82.8 & \multicolumn{1}{|c}{82.8} \\
 & 50 & 0.1 & 0.1 & 20.1 & 72.2 & 80.8 & 82.5 & 82.8 & 82.8 & \multicolumn{1}{|c}{82.9} \\
 & 100 & 0.1 & 0.2 & 18.7 & 72.0 & 80.8 & 82.6 & 82.8 & 82.8 & \multicolumn{1}{|c}{82.9} \\
 & 300 & 0.1 & 0.2 & 22.1 & 72.2 & 80.6 & 82.6 & 82.7 & 82.8 & \multicolumn{1}{|c}{82.9} \\
\midrule
\multirow{5}{*}{\begin{tabular}[c]{@{}c@{}}$\lambda=5.0$ \\ target rank 0.10 \\ learning rate $1\times10^{-5}$\end{tabular}} & 5 & 0.1 & 0.1 & 0.1 & 9.4 & 53.5 & 73.0 & 77.3 & 78.7 & \multicolumn{1}{|c}{81.0} \\
 & 20 & 0.1 & 0.1 & 0.1 & 10.0 & 53.5 & 73.0 & 77.3 & 78.7 & \multicolumn{1}{|c}{81.1} \\
 & 50 & 0.1 & 0.1 & 0.1 & 8.7 & 54.1 & 72.8 & 77.2 & 78.7 & \multicolumn{1}{|c}{81.0} \\
 & 100 & 0.1 & 0.1 & 0.2 & 9.1 & 47.5 & 72.6 & 77.4 & 78.7 & \multicolumn{1}{|c}{81.1} \\
 & 300 & 0.1 & 0.1 & 0.1 & 8.6 & 48.8 & 70.9 & 77.2 & 78.5 & \multicolumn{1}{|c}{81.0} \\
\bottomrule
\end{tabular}
\end{table}

\section{Additional Vision Model Results}\label{sec:app_additional_vision_results}

Here, we report the complete uniform compression results for the vision models. Results are shown for ViT-B/16 in \cref{tab:vitb16_imagenet_acc}, ResNet-50 in \cref{tab:rn50_imagenet_acc}, ViT-L/16 in \cref{tab:vitl16_imagenet_acc}, and ResNet-18 in \cref{tab:rn18_scratch_imagenet_acc}.

For each table, the different columns indicate the retention ratio (and Full indicates the uncompressed model). Top results are bolded, and second-top results are underlined. For the ViT-B/16 table, where we swept learning rates, $\diamond$ is used for $10^{-5}$, and $\dagger$ for $5 \times 10^{-5}$.

\paragraph{Notation. } For Q3R, the first argument denotes its regularization strength, and the second its target rank. For LoRITa, its only argument denotes the regularization strength (in the form of weight decay). For all SLORR variants, the only argument denotes regularization strength.

\begin{table}[h]
\centering
\small
\caption{Complete uniform compression results for ViT-L/16 continued training.}
\label{tab:vitl16_imagenet_acc}
\setlength{\tabcolsep}{3pt}
\begin{tabular}{lccccccccccc}
\toprule
\textbf{Method} & \textbf{10\%} & \textbf{20\%} & \textbf{30\%} & \textbf{40\%} & \textbf{50\%} & \textbf{60\%} & \textbf{70\%} & \textbf{80\%} & \textbf{90\%} & \textbf{95\%} & \multicolumn{1}{|c}{\textbf{Full}} \\
\midrule
Baseline & \textbf{0.2} & \underline{0.6} & 7.5 & 38.9 & 63.5 & 75.2 & 80.6 & 83.2 & 84.4 & 84.8 & \multicolumn{1}{|c}{\underline{85.8}} \\
\midrule
LoRITa(0.0001) & \underline{0.1} & \textbf{0.7} & \underline{12.5} & 45.8 & 67.2 & 76.9 & 81.3 & 83.5 & 84.5 & 84.9 & \multicolumn{1}{|c}{85.5} \\
LoRITa(0.001) & \underline{0.1} & \textbf{0.7} & \underline{12.5} & 45.8 & 67.2 & 76.9 & 81.3 & 83.5 & 84.5 & 84.9 & \multicolumn{1}{|c}{85.5} \\
LoRITa(0.01) & \underline{0.1} & \textbf{0.7} & \textbf{12.7} & 44.9 & 66.8 & 76.9 & 81.2 & 83.4 & 84.5 & 84.8 & \multicolumn{1}{|c}{85.5} \\
LoRITa(0.05) & \underline{0.1} & \textbf{0.7} & 11.5 & 44.6 & 66.7 & 76.8 & 81.2 & 83.4 & 84.4 & 84.8 & \multicolumn{1}{|c}{85.5} \\
LoRITa(0.3) & \textbf{0.2} & 0.4 & 7.5 & 38.8 & 63.9 & 75.8 & 80.7 & 83.0 & 84.3 & 84.5 & \multicolumn{1}{|c}{85.5} \\
LoRITa(1) & \textbf{0.2} & 0.2 & 0.3 & 3.6 & 36.4 & 66.5 & 77.0 & 81.1 & 83.1 & 83.7 & \multicolumn{1}{|c}{85.1} \\
\midrule
SLORR-Nuc(0.00002) & \textbf{0.2} & \underline{0.6} & 9.1 & 42.4 & 65.7 & 76.5 & 81.3 & 83.5 & 84.6 & 84.9 & \multicolumn{1}{|c}{85.7} \\
SLORR-Nuc(0.0002) & \textbf{0.2} & 0.4 & 8.5 & 43.8 & 70.0 & 79.1 & 82.5 & 84.1 & 84.8 & 85.0 & \multicolumn{1}{|c}{85.5} \\
SLORR-Nuc(0.001) & \underline{0.1} & \underline{0.6} & 12.2 & 54.0 & 77.0 & 80.3 & 81.1 & 81.5 & 81.5 & 81.5 & \multicolumn{1}{|c}{81.5} \\
SLORR-Nuc(0.00005) & \textbf{0.2} & \underline{0.6} & 7.8 & 41.5 & 66.6 & 77.1 & 81.6 & 83.8 & 84.7 & 84.9 & \multicolumn{1}{|c}{85.7} \\
\midrule
SLORR-Hoyer(0.00005) & \underline{0.1} & 0.4 & 7.0 & 38.2 & 64.1 & 76.2 & 81.4 & 83.8 & 84.8 & \underline{85.2} & \multicolumn{1}{|c}{\underline{85.8}} \\
SLORR-Hoyer(0.0002) & \underline{0.1} & 0.2 & 7.0 & 42.8 & 65.5 & 77.5 & 82.8 & \underline{84.6} & \textbf{85.3} & \textbf{85.5} & \multicolumn{1}{|c}{\textbf{85.9}} \\
SLORR-Hoyer(0.0005) & \underline{0.1} & 0.2 & 7.9 & 52.2 & 73.4 & 82.0 & \textbf{84.1} & \textbf{84.9} & \underline{85.1} & \underline{85.2} & \multicolumn{1}{|c}{85.4} \\
SLORR-Hoyer(0.001) & \underline{0.1} & 0.2 & 3.3 & \underline{54.7} & \underline{78.0} & \underline{82.4} & \underline{83.6} & 84.1 & 84.3 & 84.3 & \multicolumn{1}{|c}{84.4} \\
SLORR-Hoyer(0.0015) & \underline{0.1} & 0.2 & 5.8 & \textbf{62.1} & \textbf{80.1} & \textbf{82.6} & \underline{83.6} & 83.9 & 84.1 & 84.1 & \multicolumn{1}{|c}{84.1} \\
\bottomrule

\end{tabular}
\end{table}

\begin{table}[h]
\centering
\small
\caption{Complete uniform compression results for ResNet-18 pretraining.}
\label{tab:rn18_scratch_imagenet_acc}
\setlength{\tabcolsep}{3pt}
\begin{tabular}{lccccccccccc}
\toprule
\textbf{Method} & \textbf{14\%} & \textbf{23\%} & \textbf{33\%} & \textbf{42\%} & \textbf{52\%} & \textbf{62\%} & \textbf{71\%} & \textbf{81\%} & \textbf{90\%} & \textbf{95\%} & \multicolumn{1}{|c}{\textbf{Full}} \\
\midrule
Baseline & 0.1 & 0.1 & 0.2 & 0.2 & 1.5 & 4.4 & 19.1 & 25.7 & 45.8 & 48.7 & \multicolumn{1}{|c}{66.2} \\
\midrule
Q3R(1, 0.0025) & 0.1 & 0.2 & 0.9 & 4.6 & 12.8 & 26.9 & 39.4 & 48.4 & 52.2 & 55.4 & \multicolumn{1}{|c}{60.0} \\
Q3R(1, 0.01) & 0.1 & 0.3 & 1.6 & 15.7 & 32.6 & 42.4 & 48.6 & 54.0 & 56.0 & 57.8 & \multicolumn{1}{|c}{62.1} \\
Q3R(5, 0.01) & \textbf{0.3} & \underline{5.3} & 18.7 & 34.3 & 33.2 & 47.0 & 50.2 & 52.7 & 53.4 & 54.2 & \multicolumn{1}{|c}{55.2} \\
Q3R(1, 0.025) & \underline{0.2} & 0.6 & \underline{30.8} & 53.2 & 57.0 & 60.0 & 62.8 & 63.5 & 63.6 & 63.9 & \multicolumn{1}{|c}{64.0} \\
Q3R(5, 0.025) & 0.1 & 0.2 & 26.2 & \textbf{62.1} & \textbf{63.3} & \underline{63.3} & 63.3 & 63.3 & 63.3 & 63.3 & \multicolumn{1}{|c}{63.3} \\
Q3R(1, 0.05) & 0.1 & 0.3 & 1.2 & 17.4 & 42.5 & \textbf{64.0} & \textbf{65.6} & \underline{65.9} & 65.8 & 65.9 & \multicolumn{1}{|c}{65.9} \\
Q3R(2, 0.05) & 0.1 & 0.2 & 0.9 & 18.0 & 46.3 & 61.4 & 64.9 & 65.6 & 65.7 & 65.7 & \multicolumn{1}{|c}{65.7} \\
Q3R(5, 0.05) & 0.1 & 0.2 & 0.5 & 3.5 & 19.4 & 50.0 & 63.7 & 65.4 & 65.8 & 65.9 & \multicolumn{1}{|c}{65.9} \\
Q3R(2, 0.1) & 0.1 & 0.1 & 0.4 & 1.7 & 8.4 & 24.3 & 41.4 & 51.2 & 59.8 & 62.8 & \multicolumn{1}{|c}{66.3} \\
Q3R(0.01, 0.15) & 0.1 & 0.1 & 0.3 & 1.2 & 3.3 & 6.0 & 21.6 & 35.6 & 49.4 & 55.6 & \multicolumn{1}{|c}{66.5} \\
Q3R(2, 0.2) & 0.1 & 0.1 & 0.1 & 0.4 & 0.9 & 2.0 & 15.9 & 28.0 & 40.2 & 46.2 & \multicolumn{1}{|c}{59.8} \\
\midrule
LoRITa(0.0001) & 0.1 & 0.2 & 1.0 & 6.6 & 26.2 & 45.6 & 53.9 & 62.9 & \underline{66.2} & 66.5 & \multicolumn{1}{|c}{67.1} \\
LoRITa(0.001) & 0.1 & 0.2 & 0.6 & 7.9 & 27.0 & 46.9 & 59.9 & 64.0 & 66.0 & 66.5 & \multicolumn{1}{|c}{67.0} \\
LoRITa(0.01) & 0.1 & 0.1 & 0.3 & 2.5 & 22.0 & 43.1 & 56.8 & 64.0 & 65.9 & 66.6 & \multicolumn{1}{|c}{67.3} \\
LoRITa(0.05) & 0.1 & 0.2 & 0.3 & 6.9 & 37.2 & 58.7 & \underline{65.1} & \textbf{67.1} & \textbf{67.5} & \textbf{67.7} & \multicolumn{1}{|c}{\textbf{67.9}} \\
LoRITa(0.3) & 0.1 & 0.1 & 1.2 & 23.9 & 55.5 & 63.2 & 65.0 & 65.8 & 66.0 & 66.0 & \multicolumn{1}{|c}{66.0} \\
LoRITa(0.5) & 0.1 & 0.1 & 1.4 & 43.0 & 55.7 & 63.1 & 64.0 & 64.1 & 64.2 & 64.2 & \multicolumn{1}{|c}{64.2} \\
LoRITa(1) & 0.1 & 0.2 & 3.8 & 38.1 & 48.8 & 59.5 & 60.7 & 60.7 & 60.7 & 60.7 & \multicolumn{1}{|c}{60.8} \\
\midrule
SLORR-Nuc(0.00001) & 0.1 & 0.2 & 0.5 & 1.5 & 5.6 & 16.0 & 25.8 & 51.9 & 59.1 & 62.1 & \multicolumn{1}{|c}{66.6} \\
SLORR-Nuc(0.00005) & 0.1 & 0.3 & 0.4 & 1.6 & 17.6 & 34.8 & 43.5 & 62.2 & 65.8 & 66.6 & \multicolumn{1}{|c}{\textbf{67.9}} \\
SLORR-Nuc(0.0001) & 0.1 & 0.1 & 0.3 & 1.9 & 11.1 & 48.3 & 58.7 & 63.9 & 66.0 & \underline{66.8} & \multicolumn{1}{|c}{\underline{67.8}} \\
SLORR-Nuc(0.0005) & 0.1 & 0.1 & 0.2 & 0.2 & 0.7 & 4.9 & 13.6 & 43.9 & 54.5 & 55.4 & \multicolumn{1}{|c}{60.3} \\
SLORR-Nuc(0.001) & 0.1 & 0.1 & 0.2 & 0.1 & 0.4 & 2.2 & 37.6 & 52.2 & 54.4 & 55.5 & \multicolumn{1}{|c}{56.2} \\
\midrule
SLORR-Hoyer(0.0002) & 0.1 & 0.2 & 0.3 & 0.9 & 3.5 & 17.2 & 37.9 & 48.8 & 58.0 & 61.2 & \multicolumn{1}{|c}{66.5} \\
SLORR-Hoyer(0.0005) & 0.1 & 0.2 & 0.6 & 0.6 & 2.1 & 13.6 & 31.1 & 50.1 & 57.7 & 61.3 & \multicolumn{1}{|c}{66.7} \\
SLORR-Hoyer(0.001) & 0.1 & 0.2 & 0.1 & 0.6 & 6.2 & 13.8 & 36.4 & 53.5 & 57.9 & 61.4 & \multicolumn{1}{|c}{66.4} \\
SLORR-Hoyer(0.002) & 0.1 & 0.1 & 0.1 & 0.4 & 0.8 & 11.3 & 30.0 & 43.5 & 54.9 & 59.0 & \multicolumn{1}{|c}{66.2} \\
SLORR-Hoyer(0.005) & 0.1 & 0.1 & 0.1 & 0.3 & 1.5 & 7.4 & 23.6 & 46.8 & 58.6 & 62.2 & \multicolumn{1}{|c}{65.2} \\
SLORR-Hoyer(0.007) & 0.1 & 0.1 & 0.1 & 0.2 & 0.8 & 5.2 & 36.1 & 54.5 & 58.8 & 61.3 & \multicolumn{1}{|c}{64.4} \\
SLORR-Hoyer(0.01) & 0.1 & 0.2 & 0.2 & 0.1 & 0.7 & 8.5 & 27.4 & 40.8 & 57.0 & 60.9 & \multicolumn{1}{|c}{63.3} \\
\midrule
SLORR-Hoyer-D(0.05) & 0.1 & 0.2 & 0.5 & 1.0 & 1.6 & 3.0 & 8.3 & 23.8 & 44.6 & 53.9 & \multicolumn{1}{|c}{66.3} \\
SLORR-Hoyer-D(0.1) & 0.1 & 0.1 & 0.3 & 1.5 & 3.8 & 14.6 & 31.3 & 49.3 & 58.2 & 61.0 & \multicolumn{1}{|c}{66.2} \\
SLORR-Hoyer-D(0.3) & 0.1 & 0.5 & 1.3 & 4.2 & 10.6 & 32.9 & 49.8 & 58.2 & 62.5 & 63.7 & \multicolumn{1}{|c}{65.7} \\
SLORR-Hoyer-D(0.5) & 0.1 & 0.2 & 1.2 & 3.3 & 19.8 & 46.9 & 55.6 & 60.8 & 63.6 & 64.3 & \multicolumn{1}{|c}{65.9} \\
SLORR-Hoyer-D(1) & 0.1 & 0.2 & 3.9 & 26.7 & 45.5 & 54.9 & 60.1 & 62.4 & 63.9 & 64.5 & \multicolumn{1}{|c}{65.1} \\
SLORR-Hoyer-D(5) & 0.1 & 0.6 & 28.3 & 53.0 & \underline{58.9} & 61.1 & 61.6 & 61.9 & 61.9 & 61.9 & \multicolumn{1}{|c}{61.9} \\
SLORR-Hoyer-D(10) & 0.1 & \textbf{8.8} & \textbf{38.8} & \underline{54.9} & 58.7 & 59.3 & 59.4 & 59.5 & 59.6 & 59.6 & \multicolumn{1}{|c}{59.6} \\
\bottomrule

\end{tabular}
\end{table}

\begin{table}[h]
\centering
\small
\caption{Complete uniform compression results for ViT-B/16 continued training.}
\label{tab:vitb16_imagenet_acc}
\setlength{\tabcolsep}{3pt}
\begin{tabular}{lccccccccccc}
\toprule
\textbf{Method} & \textbf{11\%} & \textbf{21\%} & \textbf{31\%} & \textbf{41\%} & \textbf{51\%} & \textbf{60\%} & \textbf{70\%} & \textbf{80\%} & \textbf{90\%} & \textbf{95\%} & \multicolumn{1}{|c}{\textbf{Full}} \\
\midrule
Baseline$^{\diamond}$ & \underline{0.1} & 0.1 & 0.1 & 0.7 & 12.1 & 46.5 & 67.0 & 76.8 & 80.9 & 81.9 & \multicolumn{1}{|c}{\textbf{84.9}} \\
Baseline$^{\dagger}$ & \underline{0.1} & 0.1 & 0.2 & 1.1 & 17.8 & 52.4 & 69.4 & 77.3 & 80.7 & 81.6 & \multicolumn{1}{|c}{84.1} \\
\midrule
Q3R(1, 0.05)$^{\diamond}$ & \underline{0.1} & 0.1 & 0.1 & 0.1 & 0.2 & 5.1 & 36.0 & 65.6 & 75.5 & 78.0 & \multicolumn{1}{|c}{84.4} \\
Q3R(2, 0.05)$^{\diamond}$ & \underline{0.1} & 0.1 & 0.1 & 0.2 & 0.1 & 0.1 & 0.1 & 0.1 & 0.1 & 0.2 & \multicolumn{1}{|c}{83.4} \\
Q3R(3, 0.05)$^{\diamond}$ & \underline{0.1} & 0.1 & 0.2 & 0.1 & 0.1 & 0.1 & 0.1 & 0.1 & 0.1 & 0.1 & \multicolumn{1}{|c}{81.8} \\
Q3R(5, 0.05)$^{\diamond}$ & \underline{0.1} & 0.1 & 0.1 & 0.1 & 0.7 & 3.8 & 17.0 & 45.1 & 53.8 & 58.9 & \multicolumn{1}{|c}{76.6} \\
Q3R(1, 0.1)$^{\diamond}$ & \underline{0.1} & 0.1 & 0.1 & 0.1 & 0.9 & 25.9 & 57.4 & 73.4 & 79.0 & 80.5 & \multicolumn{1}{|c}{84.5} \\
Q3R(2, 0.1)$^{\diamond}$ & \underline{0.1} & 0.1 & 0.1 & 0.2 & 0.4 & 9.2 & 34.2 & 50.8 & 66.1 & 70.9 & \multicolumn{1}{|c}{83.9} \\
Q3R(3, 0.1)$^{\diamond}$ & \underline{0.1} & 0.1 & 0.1 & 0.2 & 0.4 & 2.2 & 1.5 & 10.2 & 37.1 & 58.2 & \multicolumn{1}{|c}{83.1} \\
Q3R(5, 0.1)$^{\diamond}$ & \underline{0.1} & 0.1 & 0.1 & 0.1 & 9.3 & 53.4 & 73.3 & 77.3 & 78.7 & 79.1 & \multicolumn{1}{|c}{81.1} \\
Q3R(0.01, 0.15)$^{\diamond}$ & \underline{0.1} & 0.1 & 0.1 & 0.6 & 12.3 & 47.0 & 67.3 & 76.7 & 80.8 & 81.9 & \multicolumn{1}{|c}{\textbf{84.9}} \\
Q3R(1, 0.15)$^{\diamond}$ & \underline{0.1} & 0.1 & 0.1 & 0.1 & 3.6 & 41.1 & 65.9 & 76.5 & 80.7 & 81.8 & \multicolumn{1}{|c}{84.7} \\
Q3R(3, 0.15)$^{\diamond}$ & \underline{0.1} & 0.1 & 0.1 & 0.2 & 1.7 & 41.8 & 51.5 & 70.6 & 77.2 & 79.8 & \multicolumn{1}{|c}{83.8} \\
Q3R(2, 0.2)$^{\diamond}$ & \underline{0.1} & 0.1 & 0.1 & 0.1 & 3.4 & 44.8 & 69.0 & 78.4 & 81.7 & 82.5 & \multicolumn{1}{|c}{84.5} \\
Q3R(3, 0.2)$^{\diamond}$ & \underline{0.1} & 0.1 & 0.1 & 0.2 & 3.1 & 45.4 & 70.7 & 78.8 & 81.9 & 82.6 & \multicolumn{1}{|c}{84.2} \\
Q3R(3, 0.05)$^{\dagger}$ & \underline{0.1} & 0.1 & 70.8 & \underline{76.0} & 77.4 & 77.9 & 78.2 & 78.4 & 78.5 & 78.5 & \multicolumn{1}{|c}{78.6} \\
Q3R(5, 0.05)$^{\dagger}$ & \underline{0.1} & 2.3 & \underline{75.1} & \textbf{76.3} & 76.9 & 77.2 & 77.3 & 77.3 & 77.3 & 77.4 & \multicolumn{1}{|c}{77.4} \\
Q3R(8, 0.05)$^{\dagger}$ & \underline{0.1} & \underline{29.6} & \textbf{75.2} & \underline{76.0} & 76.1 & 76.2 & 76.3 & 76.3 & 76.3 & 76.3 & \multicolumn{1}{|c}{76.3} \\
Q3R(10, 0.05)$^{\dagger}$ & \underline{0.1} & \textbf{44.5} & 75.0 & 75.5 & 75.7 & 75.8 & 75.9 & 75.9 & 75.9 & 75.9 & \multicolumn{1}{|c}{75.9} \\
Q3R(3, 0.1)$^{\dagger}$ & \underline{0.1} & 0.1 & 0.3 & 65.4 & \underline{80.9} & \textbf{81.3} & 81.5 & 81.6 & 81.7 & 81.8 & \multicolumn{1}{|c}{81.8} \\
Q3R(5, 0.1)$^{\dagger}$ & \underline{0.1} & 0.1 & 3.3 & 71.9 & \textbf{81.0} & \underline{81.2} & 81.3 & 81.4 & 81.4 & 81.4 & \multicolumn{1}{|c}{81.4} \\
Q3R(3, 0.15)$^{\dagger}$ & \underline{0.1} & 0.1 & 0.2 & 21.4 & 71.9 & 80.6 & \textbf{82.5} & \underline{82.7} & 82.7 & 82.8 & \multicolumn{1}{|c}{82.8} \\
Q3R(5, 0.15)$^{\dagger}$ & \underline{0.1} & 0.1 & 0.3 & 34.5 & 74.6 & 81.1 & \textbf{82.5} & \underline{82.7} & 82.7 & 82.8 & \multicolumn{1}{|c}{82.8} \\
Q3R(3, 0.2)$^{\dagger}$ & \underline{0.1} & 0.1 & 0.2 & 4.9 & 57.8 & 75.6 & 80.7 & 82.3 & 83.1 & 83.2 & \multicolumn{1}{|c}{83.3} \\
Q3R(5, 0.2)$^{\dagger}$ & \underline{0.1} & 0.1 & 0.3 & 8.7 & 61.7 & 76.7 & 81.1 & 82.6 & 83.2 & 83.3 & \multicolumn{1}{|c}{83.3} \\
\midrule
LoRITa(0.0001)$^{\diamond}$ & \underline{0.1} & 0.1 & 0.2 & 1.0 & 17.8 & 52.2 & 70.0 & 77.7 & 81.2 & 82.2 & \multicolumn{1}{|c}{84.7} \\
LoRITa(0.001)$^{\diamond}$ & \underline{0.1} & 0.1 & 0.2 & 1.0 & 17.8 & 52.2 & 70.0 & 77.7 & 81.2 & 82.2 & \multicolumn{1}{|c}{84.7} \\
LoRITa(0.01)$^{\diamond}$ & \underline{0.1} & 0.1 & 0.1 & 1.0 & 17.2 & 51.9 & 70.1 & 77.9 & 81.3 & 82.1 & \multicolumn{1}{|c}{\underline{84.8}} \\
LoRITa(0.05)$^{\diamond}$ & \underline{0.1} & 0.1 & 0.1 & 1.1 & 17.7 & 52.6 & 70.3 & 77.8 & 81.2 & 82.2 & \multicolumn{1}{|c}{\underline{84.8}} \\
LoRITa(0.3)$^{\diamond}$ & \underline{0.1} & 0.1 & 0.1 & 0.6 & 12.3 & 48.4 & 68.3 & 77.0 & 80.8 & 81.9 & \multicolumn{1}{|c}{84.7} \\
LoRITa(0.5)$^{\diamond}$ & \underline{0.1} & 0.1 & 0.2 & 0.4 & 5.9 & 40.6 & 65.1 & 75.9 & 80.2 & 81.4 & \multicolumn{1}{|c}{84.6} \\
LoRITa(1)$^{\diamond}$ & \underline{0.1} & 0.1 & 0.1 & 0.1 & 0.2 & 8.3 & 42.7 & 68.1 & 76.2 & 78.7 & \multicolumn{1}{|c}{84.4} \\
LoRITa(0.0001)$^{\dagger}$ & \underline{0.1} & 0.2 & 1.0 & 15.4 & 47.1 & 67.7 & 76.4 & 80.2 & 81.9 & 82.6 & \multicolumn{1}{|c}{84.0} \\
LoRITa(0.001)$^{\dagger}$ & \underline{0.1} & 0.1 & 0.9 & 13.2 & 47.2 & 67.6 & 76.2 & 80.1 & 81.9 & 82.6 & \multicolumn{1}{|c}{84.0} \\
LoRITa(0.01)$^{\dagger}$ & \underline{0.1} & 0.1 & 1.1 & 14.0 & 47.4 & 68.0 & 76.5 & 80.0 & 81.8 & 82.5 & \multicolumn{1}{|c}{84.1} \\
LoRITa(0.05)$^{\dagger}$ & \underline{0.1} & 0.1 & 0.7 & 15.4 & 48.3 & 68.5 & 76.6 & 80.2 & 82.0 & 82.6 & \multicolumn{1}{|c}{84.0} \\
LoRITa(0.3)$^{\dagger}$ & \underline{0.1} & 0.1 & 0.2 & 3.3 & 35.8 & 64.9 & 75.0 & 79.5 & 81.5 & 82.2 & \multicolumn{1}{|c}{83.7} \\
LoRITa(0.5)$^{\dagger}$ & \underline{0.1} & 0.1 & 0.2 & 0.6 & 10.7 & 47.7 & 70.0 & 77.2 & 80.2 & 81.4 & \multicolumn{1}{|c}{83.5} \\
LoRITa(1)$^{\dagger}$ & \underline{0.1} & 0.1 & 0.1 & 0.1 & 0.2 & 0.2 & 1.4 & 4.5 & 11.6 & 15.4 & \multicolumn{1}{|c}{81.8} \\
\midrule
SLORR-Nuc(0.001)$^{\diamond}$ & \underline{0.1} & 0.1 & 0.2 & 1.2 & 30.6 & 64.9 & 75.5 & 79.9 & 81.7 & 82.1 & \multicolumn{1}{|c}{82.8} \\
SLORR-Nuc(0.002)$^{\diamond}$ & \underline{0.1} & 0.1 & 0.2 & 1.1 & 36.9 & 68.8 & 75.9 & 78.0 & 78.9 & 79.1 & \multicolumn{1}{|c}{79.2} \\
SLORR-Nuc(0.003)$^{\diamond}$ & \underline{0.1} & 0.1 & 0.1 & 6.7 & 47.8 & 69.2 & 73.4 & 74.6 & 75.0 & 75.1 & \multicolumn{1}{|c}{75.1} \\
SLORR-Nuc(0.001)$^{\dagger}$ & \underline{0.1} & 0.5 & 66.4 & 73.7 & 74.9 & 75.3 & 75.3 & 75.3 & 75.4 & 75.3 & \multicolumn{1}{|c}{75.3} \\
SLORR-Nuc(0.0002)$^{\dagger}$ & \underline{0.1} & 0.2 & 0.3 & 16.7 & 55.4 & 74.8 & 79.4 & 81.8 & 82.5 & 82.8 & \multicolumn{1}{|c}{83.2} \\
SLORR-Nuc(0.0001)$^{\dagger}$ & \underline{0.1} & 0.1 & 0.2 & 7.6 & 52.1 & 72.1 & 78.5 & 81.3 & 82.6 & 83.0 & \multicolumn{1}{|c}{83.7} \\
SLORR-Nuc(0.0005)$^{\dagger}$ & \underline{0.1} & 0.2 & 4.0 & 63.7 & 77.2 & 79.7 & 80.4 & 80.6 & 80.7 & 80.7 & \multicolumn{1}{|c}{80.7} \\
\midrule
SLORR-Hoyer(0.0002)$^{\diamond}$ & \underline{0.1} & 0.1 & 0.1 & 1.1 & 21.3 & 58.2 & 74.1 & 80.4 & 82.4 & 83.2 & \multicolumn{1}{|c}{\underline{84.8}} \\
SLORR-Hoyer(0.0005)$^{\diamond}$ & \textbf{0.2} & 0.1 & 0.1 & 1.4 & 24.9 & 67.1 & 78.9 & 82.0 & \textbf{83.4} & \textbf{83.8} & \multicolumn{1}{|c}{84.4} \\
SLORR-Hoyer(0.001)$^{\diamond}$ & \underline{0.1} & 0.1 & 0.1 & 1.3 & 32.3 & 75.3 & 81.3 & \textbf{82.8} & \textbf{83.4} & \underline{83.5} & \multicolumn{1}{|c}{83.5} \\
SLORR-Hoyer(0.002)$^{\diamond}$ & \underline{0.1} & 0.1 & 0.1 & 6.6 & 66.2 & 79.5 & 81.1 & 81.5 & 81.7 & 81.6 & \multicolumn{1}{|c}{81.7} \\
SLORR-Hoyer(0.003)$^{\diamond}$ & \underline{0.1} & 0.1 & 0.1 & 34.7 & 74.7 & 78.3 & 79.2 & 79.6 & 79.6 & 79.6 & \multicolumn{1}{|c}{79.8} \\
SLORR-Hoyer(0.005)$^{\diamond}$ & \underline{0.1} & 0.2 & 1.8 & 61.5 & 73.0 & 75.4 & 75.9 & 75.9 & 76.0 & 76.0 & \multicolumn{1}{|c}{76.1} \\
SLORR-Hoyer(0.0002)$^{\dagger}$ & \underline{0.1} & 0.1 & 0.1 & 0.7 & 10.1 & 67.9 & 80.8 & \textbf{82.8} & \underline{83.3} & \underline{83.5} & \multicolumn{1}{|c}{83.8} \\
SLORR-Hoyer(0.0005)$^{\dagger}$ & \underline{0.1} & 0.1 & 0.1 & 6.7 & 70.6 & 80.6 & \underline{82.1} & \underline{82.7} & 82.8 & 82.9 & \multicolumn{1}{|c}{83.0} \\
SLORR-Hoyer(0.0007)$^{\dagger}$ & \underline{0.1} & 0.1 & 0.2 & 15.9 & 75.7 & 81.0 & 82.0 & 82.3 & 82.6 & 82.6 & \multicolumn{1}{|c}{82.6} \\
SLORR-Hoyer(0.001)$^{\dagger}$ & \underline{0.1} & 0.1 & 0.1 & 62.3 & 78.7 & 80.8 & 81.5 & 81.7 & 81.9 & 81.9 & \multicolumn{1}{|c}{82.0} \\
SLORR-Hoyer(0.00125)$^{\dagger}$ & \underline{0.1} & 0.1 & 0.3 & 65.8 & 78.8 & 80.4 & 81.1 & 81.3 & 81.3 & 81.4 & \multicolumn{1}{|c}{81.4} \\
\bottomrule

\end{tabular}
\end{table}

\begin{table}[h]
\centering
\small
\caption{Complete uniform compression results for ResNet-50 continued training.}
\label{tab:rn50_imagenet_acc}
\setlength{\tabcolsep}{3pt}
\begin{tabular}{lccccccccccc}
\toprule
\textbf{Method} & \textbf{17\%} & \textbf{26\%} & \textbf{35\%} & \textbf{45\%} & \textbf{54\%} & \textbf{63\%} & \textbf{72\%} & \textbf{81\%} & \textbf{90\%} & \textbf{95\%} & \multicolumn{1}{|c}{\textbf{Full}} \\
\midrule
Baseline & \underline{0.1} & 0.2 & 0.4 & 9.4 & 40.3 & 62.5 & 67.9 & 71.9 & 74.0 & 74.5 & \multicolumn{1}{|c}{\underline{76.1}} \\
\midrule
Q3R(8, 0.0025) & \underline{0.1} & 0.3 & 1.4 & 14.3 & 53.8 & 64.6 & 68.6 & 69.5 & 70.2 & 70.5 & \multicolumn{1}{|c}{71.4} \\
Q3R(20, 0.0025) & \textbf{0.2} & 0.6 & 11.3 & 46.7 & 60.2 & 63.1 & 64.0 & 64.6 & 64.8 & 65.0 & \multicolumn{1}{|c}{65.6} \\
Q3R(8, 0.01) & \underline{0.1} & 1.0 & 3.2 & 36.3 & 60.7 & 68.0 & 69.4 & 70.1 & 70.4 & 70.5 & \multicolumn{1}{|c}{71.1} \\
Q3R(20, 0.01) & \underline{0.1} & \textbf{5.3} & \underline{41.9} & 59.1 & 63.2 & 64.6 & 65.1 & 65.4 & 65.6 & 65.6 & \multicolumn{1}{|c}{65.9} \\
Q3R(5, 0.025) & \underline{0.1} & 0.4 & 7.5 & 49.7 & 65.6 & 70.2 & 71.3 & 71.9 & 72.2 & 72.5 & \multicolumn{1}{|c}{73.1} \\
Q3R(8, 0.025) & \textbf{0.2} & 0.8 & 29.0 & 63.1 & 67.4 & 70.0 & 70.9 & 71.2 & 71.4 & 71.5 & \multicolumn{1}{|c}{71.9} \\
Q3R(1, 0.05) & \underline{0.1} & 0.3 & 0.7 & 25.8 & 58.0 & 69.4 & 72.5 & 73.7 & 74.4 & 74.6 & \multicolumn{1}{|c}{75.3} \\
Q3R(2, 0.05) & \textbf{0.2} & 0.3 & 4.4 & 38.4 & 67.2 & 71.8 & 73.3 & 73.5 & 74.1 & 74.3 & \multicolumn{1}{|c}{74.6} \\
Q3R(3, 0.05) & \textbf{0.2} & 0.4 & 8.1 & 49.9 & 69.9 & 73.0 & 73.5 & 73.7 & 74.2 & 74.2 & \multicolumn{1}{|c}{74.5} \\
Q3R(5, 0.05) & \underline{0.1} & 0.3 & 11.7 & 60.8 & 69.5 & 73.0 & 73.4 & 73.6 & 73.6 & 73.7 & \multicolumn{1}{|c}{74.1} \\
Q3R(8, 0.05) & \textbf{0.2} & 0.5 & 8.5 & 63.7 & 72.5 & 73.2 & 73.4 & 73.5 & 73.5 & 73.6 & \multicolumn{1}{|c}{73.7} \\
Q3R(12, 0.05) & \underline{0.1} & 0.4 & 11.3 & 65.6 & \underline{72.7} & 72.9 & 73.1 & 73.2 & 73.3 & 73.3 & \multicolumn{1}{|c}{73.4} \\
Q3R(20, 0.05) & \underline{0.1} & 0.4 & 9.1 & 63.7 & \underline{72.7} & 72.8 & 72.8 & 72.9 & 72.9 & 72.9 & \multicolumn{1}{|c}{73.0} \\
Q3R(1, 0.1) & \underline{0.1} & 0.2 & 1.0 & 32.5 & 61.9 & 70.8 & 73.4 & 74.6 & 75.0 & \underline{75.3} & \multicolumn{1}{|c}{75.5} \\
Q3R(2, 0.1) & \underline{0.1} & 0.2 & 3.1 & 40.5 & 65.1 & 72.4 & 74.0 & \textbf{74.9} & \textbf{75.2} & \textbf{75.4} & \multicolumn{1}{|c}{75.4} \\
Q3R(0.01, 0.15) & \textbf{0.2} & 0.2 & 0.3 & 10.4 & 43.2 & 62.2 & 67.7 & 71.8 & 73.9 & 74.5 & \multicolumn{1}{|c}{76.0} \\
Q3R(1, 0.15) & \underline{0.1} & 0.2 & 1.2 & 18.6 & 54.2 & 68.5 & 72.4 & 74.5 & 75.0 & 75.2 & \multicolumn{1}{|c}{75.6} \\
Q3R(2, 0.2) & \underline{0.1} & 0.1 & 0.5 & 19.2 & 50.6 & 68.0 & 71.8 & 74.1 & \underline{75.1} & \underline{75.3} & \multicolumn{1}{|c}{75.6} \\
\midrule
LoRITa(0.00001) & \textbf{0.2} & 0.2 & 0.4 & 7.0 & 37.6 & 61.9 & 67.5 & 71.7 & 73.8 & 74.3 & \multicolumn{1}{|c}{76.0} \\
LoRITa(0.0001) & \textbf{0.2} & 0.2 & 0.4 & 7.0 & 37.6 & 61.9 & 67.5 & 71.7 & 73.8 & 74.3 & \multicolumn{1}{|c}{76.0} \\
LoRITa(0.001) & \underline{0.1} & 0.2 & 0.4 & 6.6 & 39.8 & 61.5 & 67.5 & 71.8 & 73.9 & 74.5 & \multicolumn{1}{|c}{\underline{76.1}} \\
LoRITa(0.01) & \textbf{0.2} & 0.2 & 0.4 & 5.9 & 39.5 & 61.8 & 67.3 & 71.5 & 73.8 & 74.5 & \multicolumn{1}{|c}{76.0} \\
LoRITa(0.05) & \textbf{0.2} & 0.2 & 0.3 & 5.9 & 40.3 & 61.7 & 67.5 & 71.7 & 74.0 & 74.7 & \multicolumn{1}{|c}{\textbf{76.2}} \\
LoRITa(0.3) & \underline{0.1} & 0.2 & 0.4 & 6.0 & 38.8 & 61.5 & 67.5 & 71.5 & 73.8 & 74.4 & \multicolumn{1}{|c}{\underline{76.1}} \\
LoRITa(0.5) & \textbf{0.2} & 0.1 & 0.4 & 6.6 & 38.5 & 61.2 & 67.8 & 71.7 & 73.9 & 74.3 & \multicolumn{1}{|c}{\underline{76.1}} \\
\midrule
SLORR-Nuc(0.0002) & \underline{0.1} & 0.2 & 0.5 & 13.0 & 52.8 & 67.0 & 71.5 & 74.0 & 75.0 & 75.2 & \multicolumn{1}{|c}{75.7} \\
SLORR-Nuc(0.0005) & \textbf{0.2} & 0.1 & 0.2 & 20.5 & 64.0 & 72.1 & \underline{74.3} & \textbf{74.9} & \underline{75.1} & 75.2 & \multicolumn{1}{|c}{75.3} \\
SLORR-Nuc(0.001) & \underline{0.1} & 0.1 & 0.2 & 37.1 & 71.6 & \textbf{74.0} & \textbf{74.5} & 74.6 & 74.6 & 74.6 & \multicolumn{1}{|c}{74.7} \\
SLORR-Nuc(0.002) & \underline{0.1} & 0.1 & 0.1 & 58.1 & \textbf{72.9} & 73.2 & 73.3 & 73.3 & 73.3 & 73.3 & \multicolumn{1}{|c}{73.3} \\
SLORR-Nuc(0.003) & \underline{0.1} & 0.2 & 1.0 & 62.0 & 71.8 & 72.0 & 72.0 & 72.0 & 72.0 & 72.0 & \multicolumn{1}{|c}{72.0} \\
SLORR-Nuc(0.005) & \underline{0.1} & 0.2 & 1.9 & 66.2 & 69.4 & 69.6 & 69.6 & 69.6 & 69.6 & 69.6 & \multicolumn{1}{|c}{69.6} \\
\midrule
SLORR-Hoyer(0.0002) & \underline{0.1} & 0.2 & 0.4 & 23.4 & 56.8 & 69.2 & 73.4 & \underline{74.7} & \textbf{75.2} & \textbf{75.4} & \multicolumn{1}{|c}{75.5} \\
SLORR-Hoyer(0.0005) & \underline{0.1} & 0.2 & 0.6 & 44.1 & 68.2 & 72.5 & 74.1 & 74.5 & 74.6 & 74.6 & \multicolumn{1}{|c}{74.6} \\
SLORR-Hoyer(0.001) & \underline{0.1} & 0.2 & 1.6 & 46.8 & 69.9 & 72.8 & 73.5 & 73.6 & 73.7 & 73.7 & \multicolumn{1}{|c}{73.6} \\
SLORR-Hoyer(0.002) & \underline{0.1} & 0.1 & 0.9 & 55.5 & 70.7 & 72.0 & 72.2 & 72.4 & 72.3 & 72.3 & \multicolumn{1}{|c}{72.4} \\
SLORR-Hoyer(0.003) & \underline{0.1} & 0.1 & 1.5 & 61.1 & 70.6 & 71.2 & 71.2 & 71.3 & 71.3 & 71.3 & \multicolumn{1}{|c}{71.3} \\
SLORR-Hoyer(0.005) & \underline{0.1} & 0.1 & 8.2 & 68.1 & 69.7 & 69.8 & 69.9 & 69.9 & 69.9 & 69.9 & \multicolumn{1}{|c}{69.9} \\
\midrule
SLORR-Hoyer-D(0.1) & \textbf{0.2} & 0.2 & 0.6 & 17.5 & 49.8 & 67.7 & 71.9 & 74.1 & \underline{75.1} & \textbf{75.4} & \multicolumn{1}{|c}{75.9} \\
SLORR-Hoyer-D(0.25) & \underline{0.1} & 0.3 & 0.6 & 29.0 & 59.3 & 71.7 & 74.1 & \textbf{74.9} & \textbf{75.2} & \textbf{75.4} & \multicolumn{1}{|c}{75.5} \\
SLORR-Hoyer-D(0.5) & \underline{0.1} & 0.4 & 0.7 & 45.0 & 69.5 & \underline{73.7} & \underline{74.3} & 74.4 & 74.5 & 74.4 & \multicolumn{1}{|c}{74.5} \\
SLORR-Hoyer-D(1) & \textbf{0.2} & 0.6 & 16.6 & 68.1 & 72.4 & 73.0 & 73.0 & 73.0 & 73.0 & 73.0 & \multicolumn{1}{|c}{73.0} \\
SLORR-Hoyer-D(1.5) & \underline{0.1} & 0.7 & 32.4 & \textbf{69.2} & 70.9 & 70.9 & 70.9 & 70.9 & 70.9 & 70.9 & \multicolumn{1}{|c}{70.9} \\
SLORR-Hoyer-D(2) & \underline{0.1} & \underline{1.1} & \textbf{46.5} & \underline{68.2} & 68.5 & 68.5 & 68.5 & 68.6 & 68.6 & 68.5 & \multicolumn{1}{|c}{68.6} \\
\bottomrule

\end{tabular}
\end{table}

\clearpage

\section{Overhead on Vision Models}\label{sec:app_overhead_details}
For our main experiments on ResNet-50 and ViT-B/16, we used NVIDIA A100-80GB GPUs. For the ViT-L/16 experiments, we used B200 GPUs. For the ResNet-18 experiments, we used NVIDIA A100-40GB GPUs. To obtain controlled overhead measurements, we do not rely on end-to-end times from full training runs. Instead, our procedure is as follows. For each model, we conduct 10 short benchmark runs on the same hardware as in the main experiments. For each run, we first perform 10 warmup training steps and then measure time and peak memory over 300 training steps under the same setup as in our training runs. The measurement takes every part of training into account: data loading and augmentation, transfer to the device, training steps, regularization, etc. In the case of Q3R, this includes periodic SVDs. This does not include, for example, evaluation or checkpointing. Each benchmark run repeats this for all methods, including the unregularized baseline. We also tried using autotuning during the Polar Express compilation step, but this always yielded worse performance, so we do not use it.

For SLORR, we also vary the number of Polar Express steps; for Q3R, we vary the refresh interval and the target rank.

For each run, we compute the normalized time as the time taken with a given method or regularizer divided by the time taken with no regularizer in the same run. More precisely, if $T_{m,i}$ is the time taken by method $m$ in run $i$, measured over the 300 timed training steps, and $T_{\mathrm{base},i}$ is the corresponding time for the unregularized baseline, then
\[
\mathrm{NormTime}_{m,i} = \frac{T_{m,i}}{T_{\mathrm{base},i}}.
\]
Thus, a normalized time of $\times 1.037$ means that the method takes $3.7\%$ more time than unregularized training in that setting. Normalized memory is measured similarly, using peak memory instead of time. We show the means and standard deviations of both the normalized and raw measurements in \cref{tab:efficiency_measures_per_method}.

\begin{table}[h]
\centering
\scriptsize
\setlength{\tabcolsep}{6pt}
\caption{\textbf{Training overhead.} Average wall-clock time and peak memory per 300 optimization steps. Means and standard deviations over 10 jobs for each method are reported.}
\label{tab:efficiency_measures_per_method}
\begin{tabular}{llrrrr}
\toprule
Model & Method & Time/300 steps (s) & Peak Mem (GB) & Norm. Time & Norm. Mem \\
\midrule
\multirow{16}{*}{ViT-B/16} & Baseline & 258.975 $\pm$ 6.500 & 71.584 $\pm$ 0.001 & $\times1.000\pm0.000$ & $\times1.000\pm0.000$ \\
 & SLORR-Hoyer (steps=6) & 268.447 $\pm$ 6.470 & 71.923 $\pm$ 0.000 & $\times1.037\pm0.003$ & $\times1.005\pm0.000$ \\
 & SLORR-Hoyer (steps=8) & 271.520 $\pm$ 6.112 & 71.926 $\pm$ 0.000 & $\times1.049\pm0.004$ & $\times1.005\pm0.000$ \\
 & SLORR-Hoyer (steps=10) & 274.345 $\pm$ 6.506 & 71.924 $\pm$ 0.001 & $\times1.059\pm0.003$ & $\times1.005\pm0.000$ \\
 & SLORR-Hoyer-D (steps=6) & 274.687 $\pm$ 6.255 & 71.582 $\pm$ 0.002 & $\times1.061\pm0.003$ & $\times1.000\pm0.000$ \\
 & SLORR-Hoyer-D (steps=8) & 277.429 $\pm$ 6.678 & 71.582 $\pm$ 0.002 & $\times1.071\pm0.003$ & $\times1.000\pm0.000$ \\
 & SLORR-Hoyer-D (steps=10) & 280.014 $\pm$ 6.140 & 71.582 $\pm$ 0.002 & $\times1.081\pm0.005$ & $\times1.000\pm0.000$ \\
 & LoRITa & 260.670 $\pm$ 6.372 & 72.153 $\pm$ 0.001 & $\times1.007\pm0.002$ & $\times1.008\pm0.000$ \\
 & Q3R (interval=5, rank=0.01) & 426.005 $\pm$ 6.311 & 72.729 $\pm$ 0.001 & $\times1.645\pm0.016$ & $\times1.016\pm0.000$ \\
 & Q3R (interval=5, rank=0.05) & 426.984 $\pm$ 6.009 & 72.728 $\pm$ 0.002 & $\times1.649\pm0.019$ & $\times1.016\pm0.000$ \\
 & Q3R (interval=5, rank=0.1) & 428.857 $\pm$ 6.109 & 72.729 $\pm$ 0.001 & $\times1.656\pm0.018$ & $\times1.016\pm0.000$ \\
 & Q3R (interval=5, rank=0.15) & 431.551 $\pm$ 6.355 & 72.727 $\pm$ 0.002 & $\times1.667\pm0.018$ & $\times1.016\pm0.000$ \\
 & Q3R (interval=20, rank=0.1) & 310.177 $\pm$ 6.220 & 72.728 $\pm$ 0.002 & $\times1.198\pm0.006$ & $\times1.016\pm0.000$ \\
 & Q3R (interval=50, rank=0.1) & 286.478 $\pm$ 6.248 & 72.726 $\pm$ 0.001 & $\times1.106\pm0.004$ & $\times1.016\pm0.000$ \\
 & Q3R (interval=100, rank=0.1) & 278.859 $\pm$ 6.291 & 72.725 $\pm$ 0.001 & $\times1.077\pm0.004$ & $\times1.016\pm0.000$ \\
 & Q3R (interval=300, rank=0.1) & 273.238 $\pm$ 6.271 & 72.724 $\pm$ 0.001 & $\times1.055\pm0.003$ & $\times1.016\pm0.000$ \\
\midrule
\multirow{16}{*}{ResNet-50} & Baseline & 238.173 $\pm$ 5.457 & 46.221 $\pm$ 0.001 & $\times1.000\pm0.000$ & $\times1.000\pm0.000$ \\
 & SLORR-Hoyer (steps=6) & 249.511 $\pm$ 5.716 & 46.316 $\pm$ 0.002 & $\times1.048\pm0.003$ & $\times1.002\pm0.000$ \\
 & SLORR-Hoyer (steps=8) & 252.558 $\pm$ 5.616 & 46.315 $\pm$ 0.001 & $\times1.060\pm0.002$ & $\times1.002\pm0.000$ \\
 & SLORR-Hoyer (steps=10) & 256.021 $\pm$ 5.647 & 46.315 $\pm$ 0.001 & $\times1.075\pm0.004$ & $\times1.002\pm0.000$ \\
 & SLORR-Hoyer-D (steps=6) & 256.658 $\pm$ 5.717 & 46.221 $\pm$ 0.001 & $\times1.078\pm0.004$ & $\times1.000\pm0.000$ \\
 & SLORR-Hoyer-D (steps=8) & 260.724 $\pm$ 5.618 & 46.221 $\pm$ 0.001 & $\times1.095\pm0.006$ & $\times1.000\pm0.000$ \\
 & SLORR-Hoyer-D (steps=10) & 264.580 $\pm$ 6.477 & 46.222 $\pm$ 0.001 & $\times1.111\pm0.016$ & $\times1.000\pm0.000$ \\
 & LoRITa & 239.302 $\pm$ 5.506 & 46.350 $\pm$ 0.001 & $\times1.005\pm0.006$ & $\times1.003\pm0.000$ \\
 & Q3R (interval=5, rank=0.01) & 291.274 $\pm$ 5.665 & 46.532 $\pm$ 0.003 & $\times1.223\pm0.007$ & $\times1.007\pm0.000$ \\
 & Q3R (interval=5, rank=0.05) & 291.162 $\pm$ 5.735 & 46.533 $\pm$ 0.002 & $\times1.223\pm0.008$ & $\times1.007\pm0.000$ \\
 & Q3R (interval=5, rank=0.1) & 290.995 $\pm$ 5.594 & 46.531 $\pm$ 0.002 & $\times1.222\pm0.006$ & $\times1.007\pm0.000$ \\
 & Q3R (interval=5, rank=0.15) & 290.084 $\pm$ 5.973 & 46.531 $\pm$ 0.001 & $\times1.218\pm0.012$ & $\times1.007\pm0.000$ \\
 & Q3R (interval=20, rank=0.1) & 258.970 $\pm$ 5.592 & 46.530 $\pm$ 0.002 & $\times1.087\pm0.002$ & $\times1.007\pm0.000$ \\
 & Q3R (interval=50, rank=0.1) & 253.010 $\pm$ 5.494 & 46.530 $\pm$ 0.002 & $\times1.062\pm0.002$ & $\times1.007\pm0.000$ \\
 & Q3R (interval=100, rank=0.1) & 251.038 $\pm$ 5.510 & 46.529 $\pm$ 0.001 & $\times1.054\pm0.006$ & $\times1.007\pm0.000$ \\
 & Q3R (interval=300, rank=0.1) & 249.735 $\pm$ 5.463 & 46.527 $\pm$ 0.001 & $\times1.049\pm0.005$ & $\times1.007\pm0.000$ \\
\midrule
\multirow{16}{*}{ViT-L/16} & Baseline & 192.797 $\pm$ 1.644 & 164.376 $\pm$ 0.000 & $\times1.000\pm0.000$ & $\times1.000\pm0.000$ \\
 & SLORR-Hoyer (steps=6) & 203.345 $\pm$ 1.227 & 165.582 $\pm$ 0.001 & $\times1.055\pm0.003$ & $\times1.007\pm0.000$ \\
 & SLORR-Hoyer (steps=8) & 205.978 $\pm$ 1.196 & 165.583 $\pm$ 0.001 & $\times1.068\pm0.004$ & $\times1.007\pm0.000$ \\
 & SLORR-Hoyer (steps=10) & 208.406 $\pm$ 1.064 & 165.582 $\pm$ 0.001 & $\times1.081\pm0.005$ & $\times1.007\pm0.000$ \\
 & SLORR-Hoyer-D (steps=6) & 207.288 $\pm$ 1.125 & 164.375 $\pm$ 0.001 & $\times1.075\pm0.004$ & $\times1.000\pm0.000$ \\
 & SLORR-Hoyer-D (steps=8) & 210.050 $\pm$ 1.210 & 164.375 $\pm$ 0.001 & $\times1.090\pm0.004$ & $\times1.000\pm0.000$ \\
 & SLORR-Hoyer-D (steps=10) & 212.514 $\pm$ 1.081 & 164.375 $\pm$ 0.001 & $\times1.102\pm0.004$ & $\times1.000\pm0.000$ \\
 & LoRITa & 194.632 $\pm$ 1.502 & 166.391 $\pm$ 0.000 & $\times1.010\pm0.001$ & $\times1.012\pm0.000$ \\
 & Q3R (interval=5, rank=0.01) & 580.708 $\pm$ 1.323 & 168.420 $\pm$ 0.001 & $\times3.012\pm0.023$ & $\times1.025\pm0.000$ \\
 & Q3R (interval=5, rank=0.05) & 585.406 $\pm$ 1.353 & 168.422 $\pm$ 0.001 & $\times3.037\pm0.022$ & $\times1.025\pm0.000$ \\
 & Q3R (interval=5, rank=0.1) & 589.815 $\pm$ 1.282 & 168.421 $\pm$ 0.002 & $\times3.059\pm0.023$ & $\times1.025\pm0.000$ \\
 & Q3R (interval=5, rank=0.15) & 594.638 $\pm$ 1.470 & 168.421 $\pm$ 0.001 & $\times3.084\pm0.022$ & $\times1.025\pm0.000$ \\
 & Q3R (interval=20, rank=0.1) & 307.126 $\pm$ 1.312 & 168.419 $\pm$ 0.001 & $\times1.593\pm0.007$ & $\times1.025\pm0.000$ \\
 & Q3R (interval=50, rank=0.1) & 250.827 $\pm$ 1.454 & 168.418 $\pm$ 0.002 & $\times1.301\pm0.004$ & $\times1.025\pm0.000$ \\
 & Q3R (interval=100, rank=0.1) & 232.288 $\pm$ 1.213 & 168.416 $\pm$ 0.000 & $\times1.205\pm0.006$ & $\times1.025\pm0.000$ \\
 & Q3R (interval=300, rank=0.1) & 219.497 $\pm$ 1.362 & 168.415 $\pm$ 0.002 & $\times1.139\pm0.003$ & $\times1.025\pm0.000$ \\
\midrule
\multirow{16}{*}{ResNet-18} & Baseline & 169.288 $\pm$ 7.488 & 13.146 $\pm$ 0.000 & $\times1.000\pm0.000$ & $\times1.000\pm0.000$ \\
 & SLORR-Hoyer (steps=6) & 169.822 $\pm$ 6.864 & 13.192 $\pm$ 0.000 & $\times1.003\pm0.022$ & $\times1.003\pm0.000$ \\
 & SLORR-Hoyer (steps=8) & 171.127 $\pm$ 8.464 & 13.191 $\pm$ 0.000 & $\times1.011\pm0.025$ & $\times1.003\pm0.000$ \\
 & SLORR-Hoyer (steps=10) & 169.576 $\pm$ 6.405 & 13.191 $\pm$ 0.000 & $\times1.002\pm0.027$ & $\times1.003\pm0.000$ \\
 & SLORR-Hoyer-D (steps=6) & 168.660 $\pm$ 5.460 & 13.144 $\pm$ 0.000 & $\times0.997\pm0.028$ & $\times1.000\pm0.000$ \\
 & SLORR-Hoyer-D (steps=8) & 170.404 $\pm$ 9.644 & 13.144 $\pm$ 0.000 & $\times1.006\pm0.022$ & $\times1.000\pm0.000$ \\
 & SLORR-Hoyer-D (steps=10) & 170.196 $\pm$ 8.941 & 13.144 $\pm$ 0.000 & $\times1.005\pm0.031$ & $\times1.000\pm0.000$ \\
 & LoRITa & 167.907 $\pm$ 4.230 & 13.191 $\pm$ 0.000 & $\times0.993\pm0.032$ & $\times1.003\pm0.000$ \\
 & Q3R (interval=5, rank=0.01) & 174.902 $\pm$ 16.476 & 13.288 $\pm$ 0.001 & $\times1.031\pm0.051$ & $\times1.011\pm0.000$ \\
 & Q3R (interval=5, rank=0.05) & 170.511 $\pm$ 9.189 & 13.289 $\pm$ 0.001 & $\times1.008\pm0.047$ & $\times1.011\pm0.000$ \\
 & Q3R (interval=5, rank=0.1) & 169.531 $\pm$ 5.638 & 13.289 $\pm$ 0.000 & $\times1.002\pm0.022$ & $\times1.011\pm0.000$ \\
 & Q3R (interval=5, rank=0.15) & 170.471 $\pm$ 6.208 & 13.288 $\pm$ 0.000 & $\times1.008\pm0.032$ & $\times1.011\pm0.000$ \\
 & Q3R (interval=20, rank=0.1) & 171.572 $\pm$ 9.018 & 13.289 $\pm$ 0.000 & $\times1.014\pm0.038$ & $\times1.011\pm0.000$ \\
 & Q3R (interval=50, rank=0.1) & 170.862 $\pm$ 7.709 & 13.288 $\pm$ 0.001 & $\times1.010\pm0.026$ & $\times1.011\pm0.000$ \\
 & Q3R (interval=100, rank=0.1) & 171.271 $\pm$ 8.197 & 13.288 $\pm$ 0.000 & $\times1.012\pm0.034$ & $\times1.011\pm0.000$ \\
 & Q3R (interval=300, rank=0.1) & 170.911 $\pm$ 7.846 & 13.288 $\pm$ 0.000 & $\times1.010\pm0.038$ & $\times1.011\pm0.000$ \\
\bottomrule
\end{tabular}
\end{table}

\subsection{Scaling Up with Vision Transformers}\label{sec:app_scaling_vit}

To support our scalability claims, we also measure overhead at different ViT scales. In addition to standard configurations (ViT-T/B/L/H), we include interpolated variants to obtain a smooth overhead curve across model sizes. We use the same settings as in our ViT overhead measurements, but vary the batch size so that all regularizers fit into a B200 GPU, with a maximum of 1024. These models span from the ViT-T scale (on the order of 6M parameters) to ViT-H (over 0.6B parameters). We detail the configurations used in \cref{tab:vit_configs}. We then measure overhead following the protocol described in \cref{sec:app_overhead_details}: normalized time is computed by dividing the measured time of each method by the measured time of the corresponding unregularized run under the same model and benchmark setup, and normalized memory is computed analogously using peak memory. Results are shown in \cref{fig:vit_gradual_overhead}. For Q3R, the target rank is fixed to $0.1$, and different refresh intervals are shown for completeness.

\begin{table}[H]
\centering
\caption{Vision Transformer (ViT) model configurations and parameter counts used in our overhead measurements.}
\label{tab:vit_configs}
\setlength{\tabcolsep}{4pt}
\begin{tabular}{lccccccc}
\toprule
Model &
Patch Size &
Embed Dim &
Depth &
\# Heads &
MLP Ratio &
Batch Size &
Params (M) \\
\midrule
\texttt{vit\_ti}        & 16 & 192  & 12 & 3  & 4.0 & 1024 & 5.72  \\
\texttt{vit\_256}       & 16 & 256  & 12 & 4  & 4.0 & 1024 & 9.98  \\
\texttt{vit\_s}         & 16 & 384  & 12 & 6  & 4.0 & 1024 & 22.05 \\
\texttt{vit\_512}       & 16 & 512  & 12 & 8  & 4.0 & 1024 & 38.84 \\
\texttt{vit\_b}         & 16 & 768  & 12 & 12 & 4.0 & 1024 & 86.57 \\
\texttt{vit\_b24}       & 16 & 768  & 24 & 12 & 4.0 & 1024 & 171.62 \\
\texttt{vit\_960\_24}   & 16 & 960  & 24 & 12 & 4.0 & 1024 & 267.61 \\
\texttt{vit\_l}         & 16 & 1024 & 24 & 16 & 4.0 & 896  & 304.33 \\
\texttt{vit\_1152\_28}  & 16 & 1152 & 28 & 16 & 4.0 & 720  & 448.60 \\
\texttt{vit\_h}         & 16 & 1280 & 32 & 16 & 4.0 & 512  & 632.20 \\
\bottomrule
\end{tabular}
\end{table}

\clearpage

\section{LLM Pretraining Experimental and Implementation Details}\label{sec:app_llm_training_details}

\paragraph{Implementation.} For our LLM pretraining experiments, we use the OLMo \citep{olmo20242olmo2furious} codebase\footnote{\url{https://github.com/allenai/olmo}}. For all training experiments, which use multiple GPUs, we use distributed data parallelism (DDP). Aside from adding the SLORR regularizer and benchmarking utilities, the main codebase remains mostly unchanged. SLORR is implemented as an additional loss term, similar to \cref{lst:lowrank_reg}. To distribute its computational cost, we partition the regularized layers across DDP ranks, so that each rank computes SLORR only for its assigned subset.

This does not introduce any extra communication. Instead, SLORR reuses the gradient synchronization already performed by DDP. Since DDP averages gradients across ranks during its standard all-reduce, each rank scales its locally computed SLORR gradients by the world size before synchronization. After the DDP all-reduce, every regularized parameter therefore receives exactly the gradient it would have obtained if the full SLORR regularizer had been computed redundantly on every rank. In this way, the computational overhead of SLORR is divided across GPUs without any additional communication cost.

\paragraph{Experimental details.} We use a Llama-like architecture \citep{touvron2023llamaopenefficientfoundation} and the T5 tokenizer \citep{JMLR:v21:20-074}, available at \url{https://huggingface.co/google-t5/t5-base}. The results reported use the original \texttt{OLMoLLamaBlock}, which computes attention manually. We note, however, that for the overhead measurements in \cref{sec:app_llm_training_overhead}, we also include results for fused attention implementations to provide a practical overhead estimate. We use the AdamW optimizer with learning rate $10^{-3}$, $\beta_1 = 0.9$, $\beta_2 = 0.95$, and weight decay set to $0.1$. Gradients are clipped using a maximum global norm of 1.0. We use a cosine scheduler with 10\% of the training steps dedicated to a linear warmup from 0 to the base learning rate, which is then decayed to 10\% of the base learning rate. We chose these settings based on common practice for standard training and did not try other settings.

Training is performed on a subsample of the FineWeb-Edu \citep{penedo2024finewebdatasetsdecantingweb} 100BT sample\footnote{\url{https://huggingface.co/datasets/HuggingFaceFW/fineweb-edu/viewer/sample-100BT}}; evaluation and calibration sets (for SVD-LLM) are also selected from this sample but are not included in the training set.

We summarize the size configurations used in \cref{tab:llm_model_configs}. All models use a global batch size of 512 and a sequence length of 1024. The 135M and 560M rows correspond to compute-optimal training budgets. The multipliers indicate the amount of data on which the model was trained relative to compute optimality \citep{hoffmann2022trainingcomputeoptimallargelanguage}. The 135M models and their overtrained versions were trained on 4 H100 GPUs, while the 560M models were trained on 8 H100 GPUs.

We regularize (and factorize after training) every linear layer inside transformer blocks (we exclude the embedding matrix and the final projection layer).

\begin{table}[H]
\centering
\caption{LLM model configurations and training budgets. All models use a global batch size of 512 and sequence length 1024.}
\label{tab:llm_model_configs}
\begin{tabular}{lccccc}
\toprule
Model & Hidden & Intermediate & Heads & Layers & Training steps \\
\midrule
135M & 768 & 2048 & 12 & 12 & 5{,}200 \\
135M$\times$4 & 768 & 2048 & 12 & 12 & 20{,}600 \\
135M$\times$8 & 768 & 2048 & 12 & 12 & 41{,}200 \\
560M & 1280 & 3456 & 20 & 24 & 21{,}400 \\
\bottomrule
\end{tabular}
\end{table}

As in our vision experiments, all runs use a default seed. For transparency, we note that the 135M$\times$8 Llama run with regularization strength $10^{-5}$ encountered a numerical explosion and terminated before completion. We monitored different metrics and did not observe preceding indicators of instability. We believe the crash may have been due to the particular trajectory taken by the model under this seed and hyperparameter setting, rather than due to the regularizer itself. We therefore reran this configuration with another seed and report the completed rerun.

\paragraph{Method hyperparameters. } The range of hyperparameters was not exhaustively explored and was mainly based on early experimentation. With the chosen hyperparameters, training was mostly stable. Among the completed runs, the $\times8$ run with the highest regularization strength reported (see \cref{sec:llm_overtraining}) had a loss spike early in training. This can also destabilize compression, as can be seen in \cref{tab:llama135mx8-plain-ppl}.

\subsection{Compression}

For compression, we adapt the SVD-LLM \citep{wang2025svdllm} codebase\footnote{\url{https://github.com/AIoT-MLSys-Lab/SVD-LLM}} to support OLMo checkpoints and the FineWeb dataset. We substantially simplified the codebase to specialize it for our setting, and also added support for plain SVD compression. For SVD-LLM, we use 256 samples of sequence length 1024 (same as in the training procedure) of held-out FineWeb-Edu data.

For activation-aware (whitened) truncation as used in SVD-LLM, we observe that regularized models often yield rank-deficient activations, meaning that their second-moment matrices are singular. SVD-LLM uses Cholesky-based whitening, which does not support rank-deficient matrices; in practice, this is addressed by adding diagonal shift. We instead modify the codebase to use BALF-like whitening \citep{gonzalezmartinez2025balfbudgetedactivationawarelowrank}, which admits rank-deficient activations in closed form. We did not observe notable end-to-end gains from this modification, though.

\subsection{Evaluation}

After training and compressing each model, we evaluate them according to standard practice. For perplexity, we report results on a held-out set of approximately 2M tokens. For downstream evaluations, we use the \texttt{lm-evaluation-harness} library \citep{eval-harness} throughout and report zero-shot metrics. When available, we use \texttt{acc\_norm}; otherwise, we use the \texttt{acc} field.

\paragraph{Tasks.}
We evaluate on ARC-Easy and ARC-Challenge \citep{clark2018thinksolvedquestionanswering}, HellaSwag \citep{zellers-etal-2019-hellaswag}, LAMBADA \citep{paperno-etal-2016-lambada}, OpenBookQA \citep{mihaylov-etal-2018-suit}, and PIQA \citep{bisk2019piqareasoningphysicalcommonsense}. We also initially explored WinoGrande \citep{sakaguchi2019winograndeadversarialwinogradschema}, but accuracy appeared to be mostly random in all cases, so we did not end up using it.

\section{Beyond Compute-Optimal LLM Pretraining}\label{sec:llm_overtraining}
Here, we report results for 135M models trained for $4\times$ and $8\times $ the compute-optimal token budget \citep{hoffmann2022trainingcomputeoptimallargelanguage}. As noted in the main text, we find that the compressibility of all models, including regularized and unregularized ones, appears to be reduced as they ingest more data. However, in all cases, regularized models are still more compressible.

\begin{figure}[H]
    \centering

    \begin{subfigure}[t]{\linewidth}
        \centering
        \includegraphics[width=\linewidth]{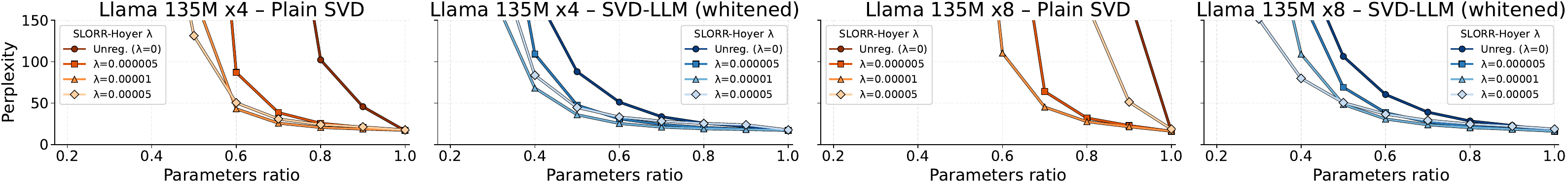}
    \end{subfigure}

    \caption{Llama models trained for $4\times $ and $8\times $ the compute-optimal token budget. Perplexity values are clipped to 150, as in our main text figures. For complete tabular results, see \cref{sec:app_complete_llm_results}.}
    \label{fig:compressibility_llm_longer_training}
\end{figure}

\section{Complete LLM Results}\label{sec:app_complete_llm_results}
In this section, we include tabular results for the LLM figures in the main text. We include one table for perplexity and, for the 560M variants, one table for each downstream task. Each row corresponds to a regularization parameter $\lambda$, including the unregularized baseline $\lambda=0$. The columns denote the retained parameter ratio in the transformer blocks.

\subsection{Llama 135M Results}

{
\setlength{\tabcolsep}{3pt}

\begin{table}[H]
    \centering
    \caption{Llama 135M -- Plain SVD -- FineWeb-Edu perplexity}
    \label{tab:llama135m-plain-ppl}
    \footnotesize
  \begingroup
  \renewenvironment{tabular}[1]{%
    \begin{tabular*}{0.95\textwidth}{@{\extracolsep{\fill}}#table/llama_135m/plain_ppl.tex@{}}%
  }{%
    \end{tabular*}%
  }%
  \begin{tabular}{lrrrrrrrrrr}
\toprule
$\lambda$ & 0.100 & 0.200 & 0.300 & 0.400 & 0.500 & 0.600 & 0.700 & 0.800 & 0.900 & 1.000 \\
\midrule
Unreg. ($\lambda=0$) & 148924.22 & 96542.02 & 94756.84 & 48167.34 & 6204.80 & 745.65 & 120.20 & 43.56 & 30.31 & \textbf{20.03} \\
$\lambda=0.000005$ & 82613.87 & 89934.46 & 121806.95 & 10169.25 & 620.10 & 62.31 & 30.94 & 24.71 & 22.56 & 20.29 \\
$\lambda=0.00001$ & 244959.88 & 126294.20 & 17418.12 & 2114.17 & 101.18 & 31.82 & 24.55 & \textbf{22.72} & \textbf{21.73} & 20.79 \\
$\lambda=0.00005$ & \textbf{19597.55} & \textbf{4195.52} & \textbf{144.97} & \textbf{30.38} & \textbf{25.98} & \textbf{24.73} & \textbf{24.17} & 23.85 & 23.68 & 23.51 \\
\bottomrule
\end{tabular}
  \endgroup

\end{table}

\begin{table}[H]
    \centering
    \caption{Llama 135M -- SVD-LLM (whitened) -- FineWeb-Edu perplexity}
    \label{tab:llama135m-whitening-ppl}
    \footnotesize
  \begingroup
  \renewenvironment{tabular}[1]{%
    \begin{tabular*}{0.95\textwidth}{@{\extracolsep{\fill}}#table/llama_135m/whitening_ppl.tex@{}}%
  }{%
    \end{tabular*}%
  }%
  \begin{tabular}{lrrrrrrrrrr}
\toprule
$\lambda$ & 0.100 & 0.200 & 0.300 & 0.400 & 0.500 & 0.600 & 0.700 & 0.800 & 0.900 & 1.000 \\
\midrule
Unreg. ($\lambda=0$) & 1860.53 & 508.60 & 185.68 & 90.66 & 52.75 & 36.50 & 29.00 & 25.01 & 22.86 & \textbf{20.03} \\
$\lambda=0.000005$ & 4931.33 & 659.68 & 158.93 & 66.69 & 38.35 & 28.79 & 24.71 & 22.65 & 21.52 & 20.29 \\
$\lambda=0.00001$ & 4820.77 & 586.07 & 125.66 & 51.09 & 31.66 & 25.63 & \textbf{23.13} & \textbf{21.93} & \textbf{21.33} & 20.79 \\
$\lambda=0.00005$ & \textbf{1855.93} & \textbf{123.38} & \textbf{38.59} & \textbf{27.87} & \textbf{25.40} & \textbf{24.45} & 24.00 & 23.77 & 23.64 & 23.51 \\
\bottomrule
\end{tabular}
  \endgroup

\end{table}

\begin{table}[H]
    \centering
    \caption{Llama 135M$\times4$ -- Plain SVD -- FineWeb-Edu perplexity}
    \label{tab:llama135mx4-plain-ppl}
    \footnotesize
  \begingroup
  \renewenvironment{tabular}[1]{%
    \begin{tabular*}{0.95\textwidth}{@{\extracolsep{\fill}}#table/llama_135m_x4/plain_ppl.tex@{}}%
  }{%
    \end{tabular*}%
  }%
  \begin{tabular}{lrrrrrrrrrr}
\toprule
$\lambda$ & 0.100 & 0.200 & 0.300 & 0.400 & 0.500 & 0.600 & 0.700 & 0.800 & 0.900 & 1.000 \\
\midrule
Unreg. ($\lambda=0$) & 256446.09 & 163660.19 & 38166.82 & 15480.69 & 9899.26 & 2142.90 & 410.64 & 102.25 & 45.71 & \textbf{16.69} \\
$\lambda=0.000005$ & 120145.84 & 70002.43 & 173048.14 & 26394.15 & 664.70 & 86.99 & 38.58 & 25.66 & 20.70 & 16.98 \\
$\lambda=0.00001$ & 481274.56 & 260275.45 & 547987.81 & 12598.17 & 181.30 & \textbf{43.16} & \textbf{25.81} & \textbf{20.62} & \textbf{18.68} & 17.29 \\
$\lambda=0.00005$ & \textbf{31923.13} & \textbf{9960.79} & \textbf{1046.29} & \textbf{379.80} & \textbf{131.44} & 50.42 & 30.91 & 24.03 & 21.29 & 17.63 \\
\bottomrule
\end{tabular}
  \endgroup

\end{table}

\begin{table}[H]
    \centering
    \caption{Llama 135M$\times4$ -- SVD-LLM (whitened) -- FineWeb-Edu perplexity}
    \label{tab:llama135mx4-whitening-ppl}
    \footnotesize
  \begingroup
  \renewenvironment{tabular}[1]{%
    \begin{tabular*}{0.95\textwidth}{@{\extracolsep{\fill}}#table/llama_135m_x4/whitening_ppl.tex@{}}%
  }{%
    \end{tabular*}%
  }%
  \begin{tabular}{lrrrrrrrrrr}
\toprule
$\lambda$ & 0.100 & 0.200 & 0.300 & 0.400 & 0.500 & 0.600 & 0.700 & 0.800 & 0.900 & 1.000 \\
\midrule
Unreg. ($\lambda=0$) & 33472.52 & 4275.64 & 704.20 & 183.61 & 87.98 & 51.24 & 33.59 & 25.79 & 21.79 & \textbf{16.69} \\
$\lambda=0.000005$ & 10243.44 & 1826.38 & 385.01 & 109.28 & 47.59 & 31.06 & 24.29 & 20.58 & 18.64 & 16.98 \\
$\lambda=0.00001$ & 6497.80 & \textbf{1064.54} & \textbf{181.37} & \textbf{68.22} & \textbf{35.97} & \textbf{25.45} & \textbf{21.04} & \textbf{19.05} & \textbf{18.13} & 17.29 \\
$\lambda=0.00005$ & \textbf{5202.02} & 1210.30 & 241.95 & 83.69 & 44.83 & 33.21 & 28.31 & 25.47 & 23.85 & 17.63 \\
\bottomrule
\end{tabular}
  \endgroup

\end{table}

\begin{table}[H]
    \centering
    \caption{Llama 135M$\times8$ -- Plain SVD -- FineWeb-Edu perplexity}
    \label{tab:llama135mx8-plain-ppl}
    \footnotesize
    \begingroup
    \setlength{\tabcolsep}{0.5pt}
  \begingroup
  \renewenvironment{tabular}[1]{%
    \begin{tabular*}{0.95\textwidth}{@{\extracolsep{\fill}}#table/llama_135m_x8/plain_ppl.tex@{}}%
  }{%
    \end{tabular*}%
  }%
  \begin{tabular}{lrrrrrrrrrr}
\toprule
$\lambda$ & 0.100 & 0.200 & 0.300 & 0.400 & 0.500 & 0.600 & 0.700 & 0.800 & 0.900 & 1.000 \\
\midrule
Unreg. ($\lambda=0$) & 458317.91 & 145537.27 & 46363.39 & 91373.03 & 175450.86 & 96216.00 & 15849.05 & 1432.41 & 176.94 & \textbf{15.85} \\
$\lambda=0.000005$ & 61717.04 & \textbf{24633.72} & 83167.53 & 59771.51 & 21030.25 & 391.34 & 64.09 & 32.02 & 22.94 & 16.06 \\
$\lambda=0.00001$ & 146102.70 & 133060.38 & 37594.02 & \textbf{2437.90} & \textbf{573.78} & \textbf{110.50} & \textbf{45.29} & \textbf{27.67} & \textbf{21.48} & 16.17 \\
$\lambda=0.00005$ & \textbf{30970.36} & 35753.34 & \textbf{21145.99} & 10263.97 & 6197.59 & 1838.47 & 590.69 & 169.88 & 51.43 & 18.67 \\
\bottomrule
\end{tabular}
  \endgroup

    \endgroup
\end{table}

\begin{table}[H]
    \centering
    \caption{Llama 135M$\times8$ -- SVD-LLM (whitened) -- FineWeb-Edu perplexity}
    \label{tab:llama135mx8-whitening-ppl}
    \footnotesize
  \begingroup
  \renewenvironment{tabular}[1]{%
    \begin{tabular*}{0.95\textwidth}{@{\extracolsep{\fill}}#table/llama_135m_x8/whitening_ppl.tex@{}}%
  }{%
    \end{tabular*}%
  }%
  \begin{tabular}{lrrrrrrrrrr}
\toprule
$\lambda$ & 0.100 & 0.200 & 0.300 & 0.400 & 0.500 & 0.600 & 0.700 & 0.800 & 0.900 & 1.000 \\
\midrule
Unreg. ($\lambda=0$) & 9130.10 & 5885.35 & 781.80 & 223.97 & 106.42 & 60.41 & 39.07 & 28.43 & 22.96 & \textbf{15.85} \\
$\lambda=0.000005$ & 11983.68 & 2819.26 & 587.96 & 178.03 & 69.14 & 38.42 & 27.14 & 21.92 & 19.04 & 16.06 \\
$\lambda=0.00001$ & 5442.37 & 1502.64 & 328.19 & 109.19 & \textbf{48.44} & \textbf{30.83} & \textbf{23.87} & \textbf{20.34} & \textbf{18.45} & 16.17 \\
$\lambda=0.00005$ & \textbf{3653.86} & \textbf{457.59} & \textbf{151.60} & \textbf{79.93} & 50.61 & 36.68 & 29.26 & 24.77 & 22.12 & 18.67 \\
\bottomrule
\end{tabular}
  \endgroup

\end{table}

}

\subsection{Llama 560M Results}

{
\setlength{\tabcolsep}{3pt}

\begin{table}[H]
    \centering
    \caption{Llama 560M -- Plain SVD -- ARC-Challenge}
    \label{tab:llama560m-plain-arc-challenge}
    \footnotesize
  \begingroup
  \renewenvironment{tabular}[1]{%
    \begin{tabular*}{0.95\textwidth}{@{\extracolsep{\fill}}#table/llama_560m/plain_downstream_arc_challenge.tex@{}}%
  }{%
    \end{tabular*}%
  }%
  \begin{tabular}{lrrrrrrrrrr}
\toprule
$\lambda$ & 0.100 & 0.200 & 0.300 & 0.400 & 0.500 & 0.600 & 0.700 & 0.800 & 0.900 & 1.000 \\
\midrule
Unreg. ($\lambda=0$) & \textbf{27.39} & 25.68 & 23.72 & 25.34 & 24.49 & 28.58 & \textbf{29.52} & \textbf{31.14} & \textbf{30.63} & \textbf{30.55} \\
$\lambda=0.000005$ & 26.28 & \textbf{26.19} & 23.55 & \textbf{27.90} & \textbf{28.33} & \textbf{29.44} & 29.27 & 29.35 & 30.12 & 29.69 \\
$\lambda=0.00001$ & 26.37 & 24.23 & \textbf{26.02} & 27.30 & 27.39 & 28.24 & 28.50 & 28.07 & 28.67 & 28.33 \\
\bottomrule
\end{tabular}
  \endgroup

\end{table}

\begin{table}[H]
    \centering
    \caption{Llama 560M -- Plain SVD -- ARC-Easy}
    \label{tab:llama560m-plain-arc-easy}
    \footnotesize
  \begingroup
  \renewenvironment{tabular}[1]{%
    \begin{tabular*}{0.95\textwidth}{@{\extracolsep{\fill}}#table/llama_560m/plain_downstream_arc_easy.tex@{}}%
  }{%
    \end{tabular*}%
  }%
  \begin{tabular}{lrrrrrrrrrr}
\toprule
$\lambda$ & 0.100 & 0.200 & 0.300 & 0.400 & 0.500 & 0.600 & 0.700 & 0.800 & 0.900 & 1.000 \\
\midrule
Unreg. ($\lambda=0$) & 26.09 & 25.72 & 26.14 & 28.79 & 34.64 & 42.21 & 47.52 & 52.78 & 54.08 & \textbf{56.82} \\
$\lambda=0.000005$ & \textbf{26.43} & 26.22 & 39.02 & \textbf{50.08} & \textbf{52.40} & \textbf{53.41} & \textbf{54.04} & \textbf{54.12} & \textbf{54.38} & 54.55 \\
$\lambda=0.00001$ & 24.92 & \textbf{27.61} & \textbf{45.50} & 49.45 & 52.02 & 52.78 & 52.95 & 53.32 & 53.49 & 53.54 \\
\bottomrule
\end{tabular}
  \endgroup

\end{table}

\begin{table}[H]
    \centering
    \caption{Llama 560M -- Plain SVD -- HellaSwag}
    \label{tab:llama560m-plain-hellaswag}
    \footnotesize
  \begingroup
  \renewenvironment{tabular}[1]{%
    \begin{tabular*}{0.95\textwidth}{@{\extracolsep{\fill}}#table/llama_560m/plain_downstream_hellaswag.tex@{}}%
  }{%
    \end{tabular*}%
  }%
  \begin{tabular}{lrrrrrrrrrr}
\toprule
$\lambda$ & 0.100 & 0.200 & 0.300 & 0.400 & 0.500 & 0.600 & 0.700 & 0.800 & 0.900 & 1.000 \\
\midrule
Unreg. ($\lambda=0$) & 26.16 & 26.26 & 26.39 & 26.80 & 30.80 & 35.87 & 40.84 & \textbf{42.88} & \textbf{43.93} & \textbf{44.32} \\
$\lambda=0.000005$ & 25.70 & 26.44 & 31.60 & \textbf{37.18} & \textbf{39.74} & \textbf{40.63} & \textbf{41.31} & 41.85 & 42.02 & 42.45 \\
$\lambda=0.00001$ & \textbf{26.21} & \textbf{26.50} & \textbf{34.22} & 37.00 & 38.74 & 39.69 & 40.15 & 40.32 & 40.47 & 40.71 \\
\bottomrule
\end{tabular}
  \endgroup

\end{table}

\begin{table}[H]
    \centering
    \caption{Llama 560M -- Plain SVD -- LAMBADA}
    \label{tab:llama560m-plain-lambada-openai}
    \footnotesize
  \begingroup
  \renewenvironment{tabular}[1]{%
    \begin{tabular*}{0.95\textwidth}{@{\extracolsep{\fill}}#table/llama_560m/plain_downstream_lambada.tex@{}}%
  }{%
    \end{tabular*}%
  }%
  \begin{tabular}{lrrrrrrrrrr}
\toprule
$\lambda$ & 0.100 & 0.200 & 0.300 & 0.400 & 0.500 & 0.600 & 0.700 & 0.800 & 0.900 & 1.000 \\
\midrule
Unreg. ($\lambda=0$) & \textbf{0.00} & 0.00 & 0.00 & 0.37 & 3.28 & 16.51 & 26.14 & 31.19 & 33.67 & \textbf{36.43} \\
$\lambda=0.000005$ & \textbf{0.00} & 0.00 & 9.37 & 23.91 & 30.39 & \textbf{34.14} & \textbf{35.09} & \textbf{35.07} & \textbf{34.80} & 35.94 \\
$\lambda=0.00001$ & \textbf{0.00} & \textbf{0.02} & \textbf{22.69} & \textbf{29.40} & \textbf{31.96} & 33.13 & 33.48 & 32.58 & 32.78 & 33.40 \\
\bottomrule
\end{tabular}
  \endgroup

\end{table}

\begin{table}[H]
    \centering
    \caption{Llama 560M -- Plain SVD -- OpenBookQA}
    \label{tab:llama560m-plain-openbookqa}
    \footnotesize
  \begingroup
  \renewenvironment{tabular}[1]{%
    \begin{tabular*}{0.95\textwidth}{@{\extracolsep{\fill}}#table/llama_560m/plain_downstream_openbookqa.tex@{}}%
  }{%
    \end{tabular*}%
  }%
  \begin{tabular}{lrrrrrrrrrr}
\toprule
$\lambda$ & 0.100 & 0.200 & 0.300 & 0.400 & 0.500 & 0.600 & 0.700 & 0.800 & 0.900 & 1.000 \\
\midrule
Unreg. ($\lambda=0$) & 28.20 & 25.40 & 23.00 & 25.60 & 27.60 & 33.60 & \textbf{35.80} & \textbf{37.20} & \textbf{35.20} & 34.80 \\
$\lambda=0.000005$ & \textbf{29.00} & \textbf{26.60} & 27.80 & 31.00 & \textbf{33.80} & \textbf{34.00} & 35.00 & 34.40 & 34.60 & 34.40 \\
$\lambda=0.00001$ & 26.20 & 24.60 & \textbf{31.20} & \textbf{31.20} & 32.20 & 33.20 & 34.20 & 34.80 & 34.60 & \textbf{35.20} \\
\bottomrule
\end{tabular}
  \endgroup

\end{table}

\begin{table}[H]
    \centering
    \caption{Llama 560M -- Plain SVD -- PIQA}
    \label{tab:llama560m-plain-piqa}
    \footnotesize
  \begingroup
  \renewenvironment{tabular}[1]{%
    \begin{tabular*}{0.95\textwidth}{@{\extracolsep{\fill}}#table/llama_560m/plain_downstream_piqa.tex@{}}%
  }{%
    \end{tabular*}%
  }%
  \begin{tabular}{lrrrrrrrrrr}
\toprule
$\lambda$ & 0.100 & 0.200 & 0.300 & 0.400 & 0.500 & 0.600 & 0.700 & 0.800 & 0.900 & 1.000 \\
\midrule
Unreg. ($\lambda=0$) & \textbf{50.54} & 50.16 & 50.54 & 53.54 & 56.64 & 60.07 & 64.15 & 66.43 & \textbf{68.12} & \textbf{69.04} \\
$\lambda=0.000005$ & 50.05 & 48.20 & 58.16 & 62.79 & 65.18 & \textbf{66.49} & \textbf{67.52} & \textbf{67.85} & 67.79 & 68.23 \\
$\lambda=0.00001$ & 49.40 & \textbf{51.74} & \textbf{61.32} & \textbf{64.09} & \textbf{65.72} & 65.89 & 66.54 & 67.03 & 66.97 & 66.92 \\
\bottomrule
\end{tabular}
  \endgroup

\end{table}

\begin{table}[H]
    \centering
    \caption{Llama 560M -- Plain SVD -- FineWeb-Edu perplexity}
    \label{tab:llama560m-plain-ppl}
    \footnotesize
  \begingroup
  \renewenvironment{tabular}[1]{%
    \begin{tabular*}{0.95\textwidth}{@{\extracolsep{\fill}}#table/llama_560m/plain_ppl.tex@{}}%
  }{%
    \end{tabular*}%
  }%
  \begin{tabular}{lrrrrrrrrrr}
\toprule
$\lambda$ & 0.100 & 0.200 & 0.300 & 0.400 & 0.500 & 0.600 & 0.700 & 0.800 & 0.900 & 1.000 \\
\midrule
Unreg. ($\lambda=0$) & \textbf{21391.07} & 40587.82 & 21837.16 & 30551.34 & 581.20 & 55.90 & 28.62 & 20.61 & 17.10 & \textbf{12.61} \\
$\lambda=0.000005$ & 42353.88 & 14071.56 & 93.47 & 24.59 & 17.41 & 15.06 & \textbf{14.11} & \textbf{13.71} & \textbf{13.54} & 13.38 \\
$\lambda=0.00001$ & 72488.98 & \textbf{2113.69} & \textbf{31.16} & \textbf{19.55} & \textbf{16.39} & \textbf{15.05} & 14.43 & 14.15 & 13.97 & 13.76 \\
\bottomrule
\end{tabular}
  \endgroup

\end{table}

\begin{table}[H]
    \centering
    \caption{Llama 560M -- SVD-LLM (whitened) -- ARC-Challenge}
    \label{tab:llama560m-whitening-arc-challenge}
    \footnotesize
  \begingroup
  \renewenvironment{tabular}[1]{%
    \begin{tabular*}{0.95\textwidth}{@{\extracolsep{\fill}}#table/llama_560m/whitening_downstream_arc_challenge.tex@{}}%
  }{%
    \end{tabular*}%
  }%
  \begin{tabular}{lrrrrrrrrrr}
\toprule
$\lambda$ & 0.100 & 0.200 & 0.300 & 0.400 & 0.500 & 0.600 & 0.700 & 0.800 & 0.900 & 1.000 \\
\midrule
Unreg. ($\lambda=0$) & 25.17 & 23.38 & 23.21 & 23.72 & 24.91 & 25.34 & 28.07 & 27.56 & 28.33 & \textbf{30.55} \\
$\lambda=0.000005$ & \textbf{25.68} & \textbf{23.55} & 24.74 & \textbf{27.47} & \textbf{27.82} & \textbf{28.33} & \textbf{29.01} & \textbf{29.61} & \textbf{29.78} & 29.69 \\
$\lambda=0.00001$ & 23.46 & 22.27 & \textbf{27.39} & 27.39 & 27.22 & 28.07 & 28.33 & 28.50 & 28.50 & 28.33 \\
\bottomrule
\end{tabular}
  \endgroup

\end{table}

\begin{table}[H]
    \centering
    \caption{Llama 560M -- SVD-LLM (whitened) -- ARC-Easy}
    \label{tab:llama560m-whitening-arc-easy}
    \footnotesize
  \begingroup
  \renewenvironment{tabular}[1]{%
    \begin{tabular*}{0.95\textwidth}{@{\extracolsep{\fill}}#table/llama_560m/whitening_downstream_arc_easy.tex@{}}%
  }{%
    \end{tabular*}%
  }%
  \begin{tabular}{lrrrrrrrrrr}
\toprule
$\lambda$ & 0.100 & 0.200 & 0.300 & 0.400 & 0.500 & 0.600 & 0.700 & 0.800 & 0.900 & 1.000 \\
\midrule
Unreg. ($\lambda=0$) & 25.25 & 27.36 & 31.65 & 35.52 & 42.34 & 47.39 & 51.35 & 54.08 & \textbf{55.60} & \textbf{56.82} \\
$\lambda=0.000005$ & 25.88 & 31.99 & 41.25 & \textbf{50.08} & 50.97 & \textbf{53.16} & \textbf{53.83} & \textbf{54.17} & 54.59 & 54.55 \\
$\lambda=0.00001$ & \textbf{27.69} & \textbf{37.67} & \textbf{46.42} & 49.96 & \textbf{52.10} & 52.82 & 53.62 & 53.24 & 53.28 & 53.54 \\
\bottomrule
\end{tabular}
  \endgroup

\end{table}

\begin{table}[H]
    \centering
    \caption{Llama 560M -- SVD-LLM (whitened) -- HellaSwag}
    \label{tab:llama560m-whitening-hellaswag}
    \footnotesize
  \begingroup
  \renewenvironment{tabular}[1]{%
    \begin{tabular*}{0.95\textwidth}{@{\extracolsep{\fill}}#table/llama_560m/whitening_downstream_hellaswag.tex@{}}%
  }{%
    \end{tabular*}%
  }%
  \begin{tabular}{lrrrrrrrrrr}
\toprule
$\lambda$ & 0.100 & 0.200 & 0.300 & 0.400 & 0.500 & 0.600 & 0.700 & 0.800 & 0.900 & 1.000 \\
\midrule
Unreg. ($\lambda=0$) & 25.65 & 26.16 & 27.06 & 29.72 & 33.40 & 36.82 & 39.28 & 41.12 & \textbf{42.32} & \textbf{44.32} \\
$\lambda=0.000005$ & 25.35 & 27.64 & 32.36 & 36.63 & \textbf{38.54} & \textbf{40.49} & \textbf{41.30} & \textbf{41.59} & 41.92 & 42.45 \\
$\lambda=0.00001$ & \textbf{25.85} & \textbf{28.91} & \textbf{33.84} & \textbf{36.73} & 38.17 & 39.11 & 39.75 & 40.19 & 40.43 & 40.71 \\
\bottomrule
\end{tabular}
  \endgroup

\end{table}

\begin{table}[H]
    \centering
    \caption{Llama 560M -- SVD-LLM (whitened) -- LAMBADA}
    \label{tab:llama560m-whitening-lambada-openai}
    \footnotesize
  \begingroup
  \renewenvironment{tabular}[1]{%
    \begin{tabular*}{0.95\textwidth}{@{\extracolsep{\fill}}#table/llama_560m/whitening_downstream_lambada.tex@{}}%
  }{%
    \end{tabular*}%
  }%
  \begin{tabular}{lrrrrrrrrrr}
\toprule
$\lambda$ & 0.100 & 0.200 & 0.300 & 0.400 & 0.500 & 0.600 & 0.700 & 0.800 & 0.900 & 1.000 \\
\midrule
Unreg. ($\lambda=0$) & 0.00 & 0.00 & 0.62 & 4.95 & 13.86 & 20.98 & 25.67 & 28.59 & 31.85 & \textbf{36.43} \\
$\lambda=0.000005$ & \textbf{0.06} & 3.36 & 21.42 & \textbf{29.05} & \textbf{32.84} & \textbf{35.16} & \textbf{35.44} & \textbf{35.30} & \textbf{35.30} & 35.94 \\
$\lambda=0.00001$ & 0.00 & \textbf{4.99} & \textbf{22.34} & 28.76 & 29.58 & 31.79 & 32.56 & 32.85 & 33.28 & 33.40 \\
\bottomrule
\end{tabular}
  \endgroup

\end{table}

\begin{table}[H]
    \centering
    \caption{Llama 560M -- SVD-LLM (whitened) -- OpenBookQA}
    \label{tab:llama560m-whitening-openbookqa}
    \footnotesize
  \begingroup
  \renewenvironment{tabular}[1]{%
    \begin{tabular*}{0.95\textwidth}{@{\extracolsep{\fill}}#table/llama_560m/whitening_downstream_openbookqa.tex@{}}%
  }{%
    \end{tabular*}%
  }%
  \begin{tabular}{lrrrrrrrrrr}
\toprule
$\lambda$ & 0.100 & 0.200 & 0.300 & 0.400 & 0.500 & 0.600 & 0.700 & 0.800 & 0.900 & 1.000 \\
\midrule
Unreg. ($\lambda=0$) & \textbf{29.60} & \textbf{28.60} & 25.60 & 27.20 & 27.40 & 31.60 & 33.00 & \textbf{34.60} & 34.60 & 34.80 \\
$\lambda=0.000005$ & 25.80 & 26.80 & 28.20 & 30.60 & \textbf{34.40} & \textbf{34.20} & \textbf{33.80} & \textbf{34.60} & \textbf{34.80} & 34.40 \\
$\lambda=0.00001$ & 27.20 & 24.40 & \textbf{29.60} & \textbf{31.80} & 31.00 & 33.00 & \textbf{33.80} & 34.40 & 34.00 & \textbf{35.20} \\
\bottomrule
\end{tabular}
  \endgroup

\end{table}

\begin{table}[H]
    \centering
    \caption{Llama 560M -- SVD-LLM (whitened) -- PIQA}
    \label{tab:llama560m-whitening-piqa}
    \footnotesize
  \begingroup
  \renewenvironment{tabular}[1]{%
    \begin{tabular*}{0.95\textwidth}{@{\extracolsep{\fill}}#table/llama_560m/whitening_downstream_piqa.tex@{}}%
  }{%
    \end{tabular*}%
  }%
  \begin{tabular}{lrrrrrrrrrr}
\toprule
$\lambda$ & 0.100 & 0.200 & 0.300 & 0.400 & 0.500 & 0.600 & 0.700 & 0.800 & 0.900 & 1.000 \\
\midrule
Unreg. ($\lambda=0$) & 50.49 & 50.82 & 52.83 & 54.79 & 59.96 & 63.28 & 65.40 & 66.38 & 67.46 & \textbf{69.04} \\
$\lambda=0.000005$ & \textbf{50.54} & 54.30 & 60.01 & 63.87 & \textbf{65.61} & \textbf{67.03} & \textbf{66.81} & \textbf{67.63} & \textbf{68.01} & 68.23 \\
$\lambda=0.00001$ & 50.38 & \textbf{55.77} & \textbf{61.64} & \textbf{63.93} & 64.64 & 66.27 & 66.59 & 66.59 & 66.92 & 66.92 \\
\bottomrule
\end{tabular}
  \endgroup

\end{table}

\begin{table}[H]
    \centering
    \caption{Llama 560M -- SVD-LLM (whitened) -- FineWeb-Edu perplexity}
    \label{tab:llama560m-whitening-ppl}
    \footnotesize
  \begingroup
  \renewenvironment{tabular}[1]{%
    \begin{tabular*}{0.95\textwidth}{@{\extracolsep{\fill}}#table/llama_560m/whitening_ppl.tex@{}}%
  }{%
    \end{tabular*}%
  }%
  \begin{tabular}{lrrrrrrrrrr}
\toprule
$\lambda$ & 0.100 & 0.200 & 0.300 & 0.400 & 0.500 & 0.600 & 0.700 & 0.800 & 0.900 & 1.000 \\
\midrule
Unreg. ($\lambda=0$) & 89505.28 & 7759.18 & 475.48 & 82.92 & 36.62 & 23.30 & 18.17 & 15.68 & 14.33 & \textbf{12.61} \\
$\lambda=0.000005$ & 9702.74 & 187.39 & 35.58 & 20.72 & 16.35 & \textbf{14.67} & \textbf{13.97} & \textbf{13.65} & \textbf{13.52} & 13.38 \\
$\lambda=0.00001$ & \textbf{1080.51} & \textbf{76.44} & \textbf{25.46} & \textbf{18.30} & \textbf{15.93} & 14.87 & 14.37 & 14.10 & 13.93 & 13.76 \\
\bottomrule
\end{tabular}
  \endgroup

\end{table}

}

\section{Overhead on LLM training}\label{sec:app_llm_training_overhead}

As noted before, our main results use the original OLMo Llama-like implementation, which uses \texttt{OLMoLlamaBlock} and computes attention manually. To provide a more realistic benchmark, we also measure overhead under several implementation choices. Specifically, we benchmark \texttt{OLMoLlamaBlock}, which uses separate Q, K, and V projections, with both manual attention and fused attention; we add the fused attention variant by modifying the block. We also benchmark \texttt{OLMoSequentialBlock}, which uses a single shared QKV projection and only uses fused attention. Here, fused attention refers to \texttt{torch.nn.functional.scaled\_dot\_product\_attention}.

For each model/configuration, we run four matched pairs of jobs. Each pair consists of two independent training runs with identical model size, attention implementation, block type, microbatch size, data, and benchmark window. One baseline is run without regularization and one run with regularization. Within a pair, the two runs are performed on the exact same node, as we observed slight overall performance variations per node. For these, we use autotuning in the Polar Express compilation step, which we observed to yield slight gains in terms of overhead (unlike in the vision benchmarks).

Each individual run measures the wall-clock time for steps 101--700. The 135M experiments use 4 H100 GPUs, and the 560M experiments use 8 H100 GPUs, following our main experiments.

We aggregate results by first computing the overhead for each matched pair,
\[
100 \times \left(\frac{T_{\mathrm{reg},i}}{T_{\mathrm{base},i}} - 1\right),
\]
where $T_{\mathrm{reg},i}$ and $T_{\mathrm{base},i}$ are the 600-step wall times for pair $i$. We report the mean and standard deviation of these paired overheads. Seconds and tokens/sec are reported as arithmetic means over the corresponding baseline or regularized runs. Peak memory is computed similarly, and was stable across runs. The reported quantities are per device, and the peak memory is taken as the maximum over all GPUs.

Results are shown in \cref{tab:llm_overhead}. MBS denotes the microbatch size per GPU. Note that, for the sake of fairness, it is adjusted for fused attention runs because each sequence consumes less memory. On average, overhead remains below 1\% across the different settings, although individual runs can slightly exceed 1\%. Peak memory overhead is negligible in these benchmarks.

\begin{table}[H]
\footnotesize
\centering
\setlength{\tabcolsep}{4pt}
\caption{LLM overhead benchmark summary.}
\label{tab:llm_overhead}
\begin{tabular}{lllcccccc}
\toprule
Size & Attention & Block & MBS & Reg. & Seconds & Tokens/sec & Overhead & Peak mem. (GB) \\
\midrule
135M & Manual & Llama & 32 & no & 735.031 & 106,992.99 & 0.000\% & 58.890 \\
135M & Manual & Llama & 32 & yes & 738.614 & 106,474.08 & 0.487\% $\pm$ 0.042\% & 59.010 \\
\midrule
135M & SDPA & Llama & 64 & no & 355.179 & 221,418.77 & 0.000\% & 56.810 \\
135M & SDPA & Llama & 64 & yes & 358.113 & 219,606.97 & 0.826\% $\pm$ 0.275\% & 56.930 \\
\midrule
135M & SDPA & Sequential & 64 & no & 345.173 & 227,837.61 & 0.000\% & 56.780 \\
135M & SDPA & Sequential & 64 & yes & 347.225 & 226,491.57 & 0.594\% $\pm$ 0.122\% & 56.930 \\
\midrule
560M & Manual & Llama & 12 & no & 1,267.318 & 31,027.41 & 0.000\% & 72.880 \\
560M & Manual & Llama & 12 & yes & 1,276.090 & 30,814.13 & 0.692\% $\pm$ 0.003\% & 72.890 \\
\midrule
560M & SDPA & Llama & 24 & no & 616.515 & 63,780.88 & 0.000\% & 61.010 \\
560M & SDPA & Llama & 24 & yes & 622.627 & 63,154.74 & 0.991\% $\pm$ 0.014\% & 61.020 \\
\midrule
560M & SDPA & Sequential & 24 & no & 596.743 & 65,894.00 & 0.000\% & 61.030 \\
560M & SDPA & Sequential & 24 & yes & 602.191 & 65,297.77 & 0.913\% $\pm$ 0.079\% & 61.040 \\
\bottomrule
\end{tabular}
\end{table}

\clearpage
\section{All Llama 135M Singular Values}
\begin{figure}[!htbp]
\centering
\scriptsize
\includegraphics[width=1.00\textwidth]{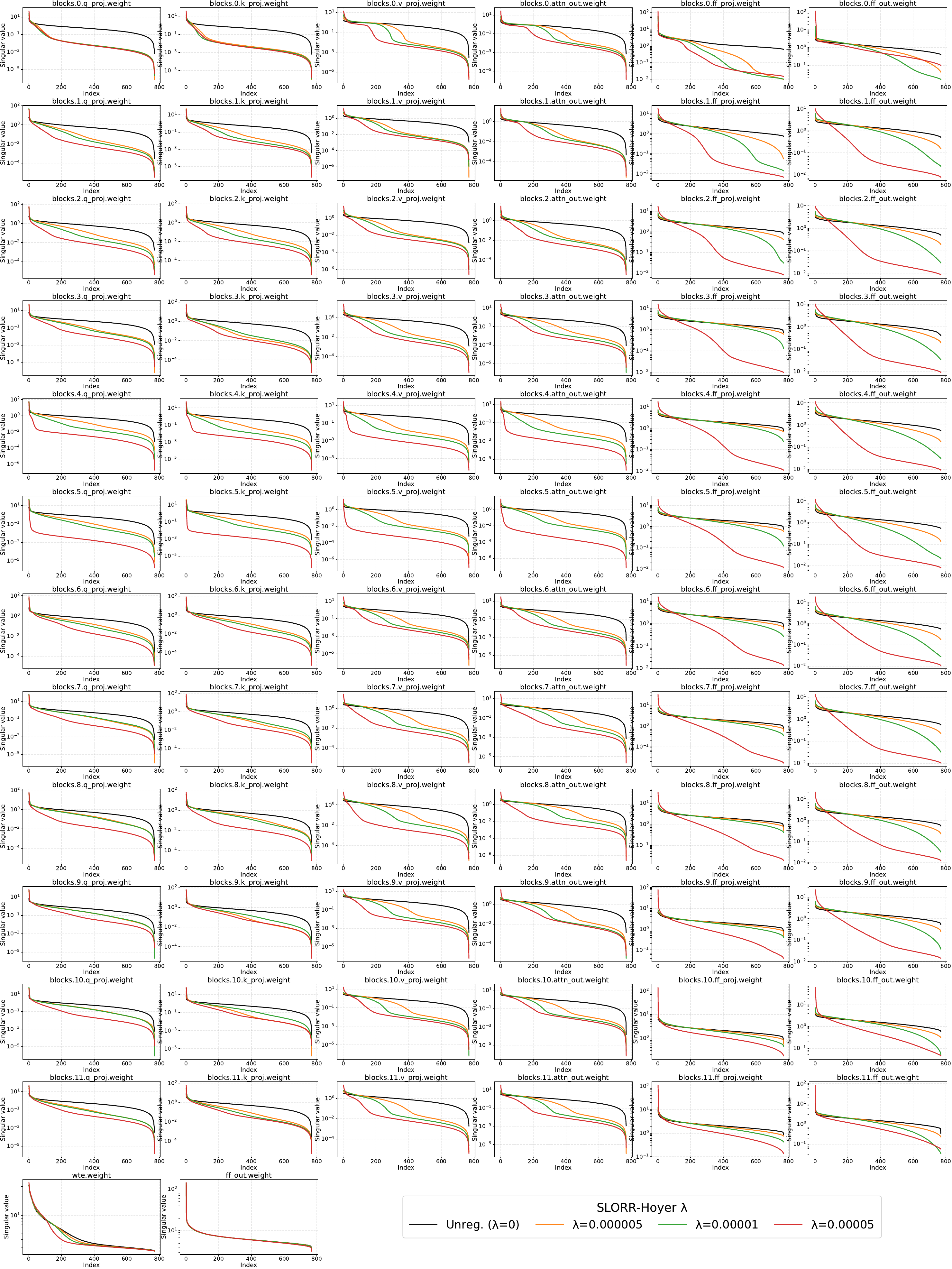}%
\caption{Singular value spectra for all layers of Llama 135M. The different curves correspond to different values of $\lambda$. As noted in the main text, only layers inside transformer blocks are regularized, but we also include weight embeddings and the final projection layer for completeness.}
\label{fig:llama135m_singular_values}
\end{figure}

\section*{LLM Usage Disclosure}
We used LLM tools in limited supporting roles during this work. In particular, they helped identify potentially relevant related work, in addition to manual search. All cited works were selected and read by the authors; LLMs primarily assisted with search. We also used LLMs to discuss our research ideas, suggest possible proof steps, check the presentation of existing technical arguments, support code development, and proofread or help polish parts of the writing. All results presented were checked and validated by the authors.

\end{document}